\documentclass[final,onefignum,onetabnum]{siamart190516}

\usepackage{amsfonts,amsopn}
\usepackage{graphicx}
\usepackage{epstopdf}
\usepackage{algpseudocode}
\ifpdf
  \DeclareGraphicsExtensions{.eps,.pdf,.png,.jpg}
\else
  \DeclareGraphicsExtensions{.eps}
\fi
\usepackage{bm}
\usepackage{subfig,color}

\headers{Targeted Adaptive Design}{C. Graziani and M. Ngom}

\title{Targeted Adaptive Design\thanks{Submitted to the editors DATE.
\funding{The U.S. Government retains a nonexclusive, royalty-free license to
publish or reproduce the published form of this contribution, or allow others to
do so, for U.S. Government purposes. Copyright is owned by SIAM to the extent
not limited by these rights.}}}

\author{
Carlo Graziani\thanks{Argonne National Laboratory, Lemont, IL 60439 USA
(\email{cgraziani@anl.gov})}
\and 
Marieme Ngom\thanks{Argonne National Laboratory, Lemont, IL 60439, USA
(\email{mngom@anl.gov})}
}

\newsiamremark{remk}{Remark}

\newcommand{\etattr}{\eta_{\mathrm{TTR}}}
\newcommand{\bx}{\bm{x}}
\newcommand{\bX}{\bm{X}}
\newcommand{\by}{\bm{y}}
\newcommand{\bz}{\bm{z}}
\newcommand{\br}{\bm{r}}
\newcommand{\bg}{\bm{g}}
\newcommand{\bC}{\bm{C}}
\newcommand{\bK}{\bm{K}}
\newcommand{\bJ}{\bm{J}}
\newcommand{\bL}{\bm{L}}
\newcommand{\bmm}{\bm{m}}
\newcommand{\bp}{\bm{p}}
\newcommand{\bP}{\bm{P}}
\newcommand{\bq}{\bm{q}}
\newcommand{\bQ}{\bm{Q}}
\newcommand{\bS}{\bm{S}}
\newcommand{\bI}{\bm{I}}

\newcommand{\bSigma}{\bm{\Sigma}}
\newcommand{\bsigma}{\bm{\sigma}}
\newcommand{\bepsilon}{\bm{\epsilon}}
\newcommand{\bmu}{\bm{\mu}}

\DeclareMathOperator*{\argmax}{arg\,max}

\begin{document}

\maketitle

\begin{abstract}
Modern advanced manufacturing and advanced materials design often require searches of relatively high-dimensional process control parameter spaces for settings that result in optimal structure, property, and performance parameters.  The mapping from the former to the latter must be determined from noisy experiments or from expensive simulations. We abstract this problem to a mathematical framework in which an unknown function from a control space to a design space must be ascertained by means of expensive noisy measurements, which locate optimal control settings generating desired design features within specified tolerances, with quantified uncertainty. We describe \textit{targeted adaptive design} (TAD), a new algorithm that performs this sampling task efficiently. TAD creates a Gaussian process surrogate model of the unknown mapping at each iterative stage, proposing a new batch of control settings to sample experimentally and optimizing the updated expected log-predictive probability density of the target design. TAD either stops upon locating a solution with uncertainties that fit inside the tolerance box or uses a measure of expected future information to determine that the search space has been exhausted with no solution. TAD thus embodies the exploration-exploitation tension in a manner that recalls, but is essentially different from, Bayesian optimization and optimal experimental design.

\end{abstract}

  

\section{Introduction}\label{sec:Intro}

Modern manufacturing methods such as  3-D printing, a type of \textit{additive manufacturing} \cite{10.1115/DETC2019-98415}, flame-spray pyrolysis \cite{PAULSON2020108972}, melt-spray deposition \cite{DelshadKhatibi2017}, and atomic layer deposition \cite{cheng2021recent}, have created a new landscape
of materials and manufactures, accelerating technological development in
engineering fields such as printed electronics, solid-state batteries, and
materials for advanced PPE. These new manufacturing techniques have brought a
new set of challenges, however. One of the principal challenges is the
difficulty and computational cost of modeling the complex detailed physical
processes that typically occur during manufacture. In \cite{10.1115/DETC2019-98415}, for example, the authors note that for additive manufacturing (AM), ``...achieving consistency with part quality and process reliability...remains a challenge...In a sense, AM has become a manufacturing domain that is data-rich but knowledge-sparse. Extracting knowledge from the vast amounts of available AM data can be a tedious process. Despite the advances in measurement science and increasing number of datasets from the AM lifecycle, there is limited scientific understanding to characterize AM materials-geometry-process-structure-property-performance relationships.'' 

In other words, scientists and engineers working with these types of processes frequently have numerous input variables to set and control, and a target product with a list of features to be attained, each within its own tolerance, but no principled and affordable way of predicting the features produced by a list of settings---the \textit{experimental response function}. There can be, however, abundant (usually noisy) data constraining the experimental response. There is therefore increasing interest in the development of machine learning methods that can take advantage of such data to assist manufacturing design by constructing inexpensive data-driven surrogates for the experimental response, and using those surrogates to improve the efficiency of the design process \cite{jenks2020basic}.

In the abstract, this problem concerns the task of locating a target design $f_T$---an $E$-dimensional vector of desired feature values---whilst simultaneously determining the structure of the experimental response---a function $f(\cdot):\mathbb{R}^{D}\rightarrow\mathbb{R}^{E}$ from the $D$-dimensional control space to the $E$-dimensional feature space. As such, it combines the aspects of target optimization \cite{609493} with black-box function optimization \cite{jones1998efficient} into a relatively novel problem, that of \textit{multi-target black-box optimization} \cite{Peremezhney_et_al_2014}. Versions of a similar problem exist for Bayesian parameter inference (as opposed to predictive design): \cite{ranjan2016inverse,zhang2019sequential} give methods for estimating the control parameters of time-series output data, where the underlying model is a potentially expensive simulation. \cite{Vernon_etal_2010} furnish a method that uses history matching to find galaxy formation model parameters that yield plausible matches to observed galaxy luminosity functions, where the galaxy formation model is again represented by an expensive simulation. These works exploit multi-output (that is, vector-valued) Gaussian process models to construct a surrogate for the unknown experimental response. Such multi-output GP modeling was explored in the context of parameter inference by \cite{higdon2008computer}.

There are some common threads in this previous research that do not perfectly suit the problem at hand. We would like a method that acknowledges and incorporates the existence of independent tolerance ranges for each of the output parameters, and that can gracefully compare those tolerances to the changing uncertainties ascribed by the inexpensive surrogate model to those individual parameters. In this way, convergence can be defined in terms of ``uncertainty within tolerance''.  In addition, we have set for ourselves the goal of obtaining a fully Bayesian method that improves on previous methodology by incorporating such uncertainty information into the scalar optimization objective, so that model-target deviations along different output parameters are properly weighted by their attending uncertainties. In other words, the concern of this work is \textit{multi-target black-box optimization with solution uncertainties within target tolerances.} 

The work of \cite{Peremezhney_et_al_2014} does in fact treat the outputs separately through a Pareto-front approach. The proposed method is frequentist, however, and in addition makes no use of tolerance information. In addition, the method must make use of a discretized input space, so as to exploit a genetic algorithm optimization strategy, which limits its application to low-dimensional input. 

Parameter inference methods have some qualitative differences that distinguish them from the manufacturing problem described above: the role of the target is played by the data itself, rather than by some desired design, whereas the ``ground truth'' information furnished by data in the manufacturing problem is instead supplied by an expensive simulation.  The Bayesian methods described in \cite{ranjan2016inverse,zhang2019sequential} operate by optimization of an Expected Improvement (EI) \cite{jones1998efficient} function that monitors improvement of an unweighted $L_2$-norm scalarization of the deviation of simulation output from the target, and hence treat deviations along better-constrained directions in feature space on an equal footing with those along poorly-constrained directions. There is also no use of parameter tolerances, which is not unusual in inference for model parameter calibration.  

In the Bayesian work of \cite{Vernon_etal_2010}, parameter tolerances are considered, but only as devices to exclude large regions of ``implausible'' parameter space, rather than as part of a convergence criterion, which instead relies on a comparison between emulator (surrogate model) uncertainty and other types of uncertainty. Per-parameter total uncertainties are used as weights in ``implausibility measures'', and the input space is iteratively reduced until arriving at an irreducible ``plausible'' region. The approach is well-suited to the study of galaxy formation, where models are characterized by large uncertainties representing poorly-controlled or unknown physics, so that coarse history matching is more informative than detailed parameter calibration. It does not provide good guidance for the manufacturing control problem, however, since the expensive simulation is then replaced by the data acquisition process, and there are of course no modeling issues that arise in this case.

The multi-output GP modeling discussed in \cite{higdon2008computer} uses a static experimental design and static data, which permits a drastic data reduction step through a Principal Component Analysis (PCA) of simulator output. The multi-output modeling is thereby simplified, since independent Gaussian process models may be used for a small number of most active PCA components. This clever approach is not available within the iterative framework for manufacturing design that we consider here, since there is no static design to furnish PCA components, so that one cannot use this data-driven technique for construction of the experimental response surrogate. 

Thus, the abstract problem raised by manufacturing design as we have stated it
does not  appear to have been addressed to date. It is nonetheless recognizably related to a broad statistical literature under the connected headings of Bayesian optimization, adaptive experimental design, and design of computer experiments.

Bayesian optimization (BO) is reviewed in \cite{7352306}. The subject addresses the general problem of finding a global optimum of an unknown acquisition function that may be noisily sampled by sequentially updating a surrogate model of the function, prescribing prior beliefs and using sampled data to revise the model iteratively through posterior updating. At each iteration, the updated model results in an updated acquisition function that is used to choose the next data sample. 

The main ingredients of BO, then, are a surrogate model, often a Gaussian process (GP), and a loss functional that processes the updated model to furnish the updated acquisition function. Expected improvement (EI) of the surrogate is possibly the most common choice of loss functional.  Important considerations include model validation at each stage of iteration and stopping rules \cite{jones1998efficient,bingham2014design}.

Within the statistical theory of design of physical experiments, the subfield of Bayesian optimal experiment design (OED) is also closely connected to the problem statement that concerns us here. Stochastic process surrogate models for the response surface were introduced by \cite{sacks1984some}, and Gaussian processes have been in common use since then.  Lindley \cite{lindley1956measure} explored the role of information theory in iterative updates of decisions concerning future experiments based on currently available data.  Information theoretic gain functions have played a prominent role in the subject (see, e.g., \cite{ryan2003estimating,zhang2016adaptive}). In the context of normal models, optimizing expected information gain is equivalent to the criterion of ``D-optimality,'' that is, to minimization of the determinant of the predictive covariance \cite{10.2307/2246015}. 

A closely connected topic is the design and analysis of computer experiments  \cite{sacks1989design,kennedy2001bayesian}. Here the problem is that expensive computer simulations must be run at a continuous range of input parameter settings (and, potentially, at a range of physical fidelities) in order to predict physical observations, so that the model itself may be constrained and possibly refined. The expense of computer simulations has incentivized research on optimal adaptive sequential strategies for sampling simulation input settings. In \cite{huang2006sequential}, an expected improvement maximization criterion was demonstrated in the context of multiple-fidelity simulations of metal formation. The work described in \cite{sinsbeck2017sequential}
gives applications of sequential design with minimization of Bayesian risk. 

We will describe an approach to target optimization adapted to the manufacturing problem that we call \textit{targeted adaptive design} (TAD). TAD exploits valuable connections to BO and OED but  differs from them in decisive ways. We will introduce a GP surrogate, albeit a vector-valued GP surrogate (since the response is vector-valued). One difference from usual practice is that rather than introducing new sample points one at a time, we do so in batches of $N_2$ points at a time (the reason for the notation will be clear below). This is not customary in the statistical sequential OED literature, although recent BO approaches have featured batch sampling (e.g., \cite{NIPS2016_18d10dc6,contal2013parallel,gonzalez2016batch}). 

In addition, instead of extremizing expected improvement or some variant thereof, our objective will be obtained by extremizing the \textit{expected log-predictive probability density} (ELPPD) that the latent response $f(x)$ attain the target value $f_T$. The resulting TAD acquisition function has a simple closed form. We exhibit this acquisition function in \cref{subsec:objectivefn}. It has some distinct advantages over existing acquisition functions: it is both easy to compute (unlike, for example, the knowledge gradient \cite{NIPS2016_18d10dc6}) and continuously differentiable (unlike the expected improvement \cite{jones1998efficient}), making it an interesting candidate for use with gradient optimization methods in high-dimensional problems. In addition, this acquisition function has the property that unlike acquisition functions that are based on $L_2$-norm measures of the discrepancy between the model and the target \cite{ranjan2016inverse,zhang2019sequential}, \textit{it properly incorporates model uncertainty in assessing target discrepancy}. The ELPPD is more intolerant of model deviations from the target along output features that are well-constrained according to their predictive uncertainty than it is of deviations along directions that are poorly constrained. This is a desirable property in an acquisition function, since it prevents the optimization from pursuing targets that are seemingly nearby according to their $L_2$-norm distance, but are actually quite far in probability.

An actual algorithm based on this approach must search the space of control settings without advance knowledge of whether a solution leading to a design that satisfies the specified tolerances in fact exists. Therefore  the algorithm must be equipped with two stopping rules: \textit{convergence/success} if such a control point $x_T$ has in fact been located and \textit{convergence/failure} if one can conclude that no such point exists. The second stopping rule is not straightforward to formulate without resorting to brute-force search. We show that the form of the TAD acquisition function is intimately connected to a second diagnostic: the \textit{expected information gain} (EIG)---in the information-theoretic sense---to be obtained from an upcoming (but not yet acquired) sample. The interplay between the TAD acquisition function and the EIG allows us to let the algorithm run, alternating tentative exploitation phases with exploration phases until either a solution is located or the surrogate function uncertainty has ``squeezed'' the target out of plausible solution, as diagnosed by a plummeting value of the EIG. In \cref{subsec:EI} we discuss EIG and its use in TAD.

The form of vector-valued Gaussian process models that we adopt is specified in terms of weighted sums of Kronecker products of scalar covariance functions over the control space with covariance matrices in the output feature space. This form permits us to simply add complexity to the surrogate model when validation measures seem to demand it. Kernel choice is discussed in \cref{subsec:kernels}.

One important lesson that we (re)learned is that an algorithm can be based on satisfying and valid theory, but a bad model can nonetheless can cause it to malfunction. For the correct functioning of the resulting algorithm,  model validity should be constantly monitored, and  model flexibility should be dynamically augmentable when model invalidity is detected. We have introduced tests of model validity that extend the usual ``leave-one-out'' (LOO) validation tests, based on out-of-sample GP prediction of future acquired samples, using frequentist $\chi^2$ tests and the appropriate GP covariances. These tests give us the required ability to dynamically extend model flexibility. We describe them in \cref{subsec:validation}.

We gather the theory into a discussion of the features of the TAD algorithm and formally describe the algorithm in \cref{sec:algorithm}. We exhibit its behavior in numerical experiments in \cref{sec:results}.

\section{Theory of Targeted Adaptive Design}

In this section we set out the theoretical development required for assembling the TAD algorithm. After the notational definitions, there are two main subsections: derivation of the \textit{TAD acquisition function}, which embodies the optimization problem to be solved, and derivation of the \textit{expected information gain}, a quantity required for the algorithm to decide whether it should declare that a search for the target design will not succeed because such a design likely does not exist in the search space. Following these are a subsection on covariance model choice and another on validation methodology.

Consider the problem of designing a product according to a list of desired
features, each within its own tolerance, using advanced manufacturing equipment,
when the experimental response to the equipment controls is not known a priori
and may be constrained only by noisy experiments.

More formally, assume a control space $X\subset\mathbb{R}^{D}$,  a design space $Y\subset\mathbb{R}^{E}$, and an unknown
experimental response function $f:X\rightarrow Y$. We may take noisy measurements of $f(\cdot)$, $g_{k}=f(x_{k})+\epsilon_{k}$, $x_{k}\in X$, $k=1,\ldots,N$,
where $\epsilon_{k}$ is a zero-mean Gaussian noise term. These measurements give
us some idea of the shape of $f(\cdot)$ in the neighborhood of the $x_{k}$. We
want to find a point $x_{T}\in X$ such that the function $f(\cdot)$ approaches
some desired target value $f_{T}$ (the design), that is, $f(x_{T})\approx f_{T}$.
The degree of approximation is set by an $E$-dimensional tolerance vector $\tau$
with components $\tau_{i}$, $i=1,\ldots,E$ representing the tolerance within
which we must hit each component of the target.  We define a ``targeted
tolerance region'' (TTR) in terms of intervals
$\eta_{i}\equiv\left[f_{Ti}-\tau_{i},f_{Ti}+\tau_{i}\right]$, $i=1,\ldots,E$,
where $f_{Ti}$ is the $i$th component of the target vector $f_{T}$. The TTR
is the region $\etattr\equiv\{y\in Y: y_i\in\eta_i, i=1,\ldots,E\}$. 

We would like an efficient procedure that samples batches of points in $X$, successively constraining the shape of $f$ while searching for regions of $X$ where $f(\cdot)$ has values closest to $f_T$, and stopping either when a value of $x_T\in X$ is located such that, with good confidence, $f(x_T)\in\etattr$ or when we are sure that no such $x$ can be found in $X$. 

The approach adopted here is to assume a Gaussian process  model on
$f(\cdot)$. We will define a candidate acquisition function in terms of the
\textit{log-predictive probability density
of }$f_{T}$ and show that this function has
desirable properties with respect to targeting $f_{T}$. We will then show that
given a proposed new experiment, one can  compute the expectation value
of this acquisition function under the distribution of predicted new data given
current data. The acquisition function can thus be optimized simultaneously with respect to
the new experimental setting and the choice of $x_{T}$.

As we will see, this acquisition function has the desirable property of embodying
the ``exploration-exploitation'' tension---the competition between
exploring regions of $X$ where the function $f(\cdot)$ is 
uncertain and investigating regions of $X$ that seem promising for
values of $x$ satisfying $f(x)=f_{T}$. 

\subsection{SetUp and Notation}\label{subsec:setup}

The standard procedure in these types of black-box problems is to model the
unknown function by some kind of surrogate function, usually a GP. We will
assume a GP model on $f(\cdot)$. Technically, since $f$ is vector-valued
($f:\mathbb{R}^{D}\rightarrow\mathbb{R}^{E})$, we require a
vector-valued GP (a ``VVGP''), also known as a ``multitask'' or
``multioutput'' GP \cite{williams2007multi}. This is a straightforward
generalization of a GP. A VVGP is characterized by a vector-valued mean function
$\mu_{i}(x)$ and by a matrix-valued covariance function
$C_{ii^{\prime}}(x,x^{\prime})$, $i,i^\prime=1,\ldots,E$. $C$ must be positive
definite, so that for any vector field
$h:\mathbb{R}^{D}\rightarrow\mathbb{R}^{E}$ we have that $\sum_{ii^{\prime}}\int
dx\,dx^{\prime}\,h_{i}(x)\,C_{ii^{\prime}}(x,x^{\prime})h_{i^{\prime}}(x^{\prime})>0$.
For a discussion of possible VVGP kernel models, see \cite{alvarez2012kernels}.
We discuss our specific model choices in \cref{subsec:kernels}. In what
follows we assume that $f\sim\textnormal{VVGP}(\mu,C)$ is a valid approximation
for the black-box function, for some vector-valued function $\mu(x)$, and for some
valid matrix-valued covariance $C(x,x^\prime)$. 

Suppose that we have three sets of
points in $X$: a set $\bx_1\equiv\{x_{1,k}\in X:k=1,\ldots,N_1\}$; a set
$\bx_2\equiv\{x_{2,l}\in X:l=1,\ldots,N_2\}$; and a  set consisting of the
singleton $\{x\in X\}$. The $\bx_1$ will represent points at which noisy data
from $f(\cdot)$ has already been acquired; $\bx_2$ will represent ``proposed''
new points at which data has yet to be obtained; and $x$ is a candidate target
solution. Below we will have occasion to use a useful shorthand, where a function $h:\mathbb{R}^D\rightarrow\mathbb{R}^E$, applied to an argument consisting of an indexed set of elements of $\mathbb{R}^E$ such as $\bx_1$, is to be interpreted as a vector $h(\bx_1)=\begin{pmatrix}h(x_{1,1})^T,\ldots,h(x_{1,N_1})^T\end{pmatrix}^T$.

Suppose noisy data from the function $f(\cdot)$ is sampled at the points in $\bx_1$. We denote this data by a vector 
$\bg_1\equiv f(\bx_1) + \bepsilon_1$, where by our shorthand $f(\bx_1)$ means the vector
$\begin{bmatrix}
f(x_{1,1})^T&f(x_{1,2})^T&\ldots&f(x_{1,N_1})^T
\end{bmatrix}^T$ and where $\bepsilon$ is an $(N_1 E)$-dimensional Gaussian zero-mean noise term with covariance $\bSigma_1$.  A similar sample $\bg_2\equiv f(\bx_2) + \bepsilon_2$ is realized at the points $\bx_2$, where the noise $\bepsilon_2$ is an $(N_2 E)$-dimensional random normal zero-mean vector with covariance $\bSigma_2$.  We assume that the observation noise is uncorrelated with the GP model of $f(\cdot)$ and  that $<\bepsilon_1\bepsilon_2>=0$ (although this latter condition could easily be relaxed). For later convenience we define the $(N_1 E)$-dimensional vector space $W_1$ such that $\bg_1\in W_1$, the $(N_2 E)$-dimensional vector space  $W_2$ such that $\bg_2\in W_2$, and the vector space $W_{1+2}\equiv W_1\oplus W_2$.

With these assumptions, the random vector $\begin{bmatrix}
f(x),\bg_1^T,\bg_2^T
\end{bmatrix}^T$ is governed by the following normal distribution:
\begin{equation}
\begin{bmatrix}
f(x)\\
\bg_1\\
\bg_2
\end{bmatrix} \sim \mathcal{N}\left\{ 
\begin{bmatrix}
\mu(x)\\
\bmu_1\\
\bmu_2
\end{bmatrix},
\Gamma
\right\},\label{eq:joint1}
\end{equation}
where $\Gamma$ is the matrix
\begin{equation}
\Gamma\equiv
\begin{bmatrix}
K_{xx} & \bK_{x1} & \bK_{x2}\\
\bK_{1x} & \bK_{11}+\bSigma_1 & \bK_{12}\\
\bK_{2x} & \bK_{21} & \bK_{22}+\bSigma_{2}
\end{bmatrix}
\label{eq:joint1_cov}
\end{equation}
and where we have defined the quantities $\bmu_1\equiv\mu(\bx_1)$, $\bmu_2\equiv\mu(\bx_2)$, $K_{xx}\equiv C(x,x)$, $\bK_{x1}=\bK_{1x}^T\equiv C(x,\bx_1)$, $\bK_{x2}=\bK_{2x}^T\equiv C(x,\bx_2)$, $K_{11}\equiv C(\bx_1,\bx_1)$, $K_{22}\equiv C(\bx_2,\bx_2)$, and $K_{12}=K_{21}^T\equiv C(\bx_1,\bx_2)$.

This distribution gives rise to the well-known \textit{predictive distribution} for $f(x)$ \cite[Chapter 2]{RandW}:

\begin{align}
f(x)|(\bg_{1},\bg_{2}) &\sim \mathcal{N}\left\{ p^{(f(x)|1+2)},Q^{(f(x)|1+2)}\right\} \label{eq:pred_1+2_t}\\
p^{(f(x)|1+2)} &\equiv \mu(x)+
\begin{bmatrix}
\bK_{x1} & \bK_{x2}\end{bmatrix}
\begin{bmatrix}
\bK_{11}+\bSigma_1 & \bK_{12}\\
\bK_{21} & \bK_{22}+\bSigma_2
\end{bmatrix}^{-1}
\begin{bmatrix}
\bg_1-\bmu_1\\
\bg_2-\bmu_2
\end{bmatrix}\label{eq:pred_1+2_mean_t}\\
Q^{(f(x)|1+2)} &\equiv K_{xx}-
\begin{bmatrix}
\bK_{x1} & \bK_{x2}
\end{bmatrix}
\begin{bmatrix}
\bK_{11}+\bSigma_1 & \bK_{12}\\
\bK_{21} & \bK_{22}+\bSigma_2
\end{bmatrix}^{-1}
\begin{bmatrix}
\bK_{1x}\\
\bK_{2x}
\end{bmatrix}.\label{eq:pred_1+2_cov_t}
\end{align}

Recall our interpretation of the data $(\bx_1,\bg_1)$ as already acquired and processed, while the data $(\bx_2,\bg_2)$ is latent, in that it might be acquired, but only after the current iteration is concluded. It follows that training of hyperparameters of the GP model in use in the current iteration produces a dependence of the GP covariance matrix $Q^{(f(x)|1+2)}$ and mean vector $p^{(f(x)|1+2)}$ on $\bg_1$ through the model hyperparameters. Since $\bg_2$ was not present during training, there is no comparable dependence through the hyperparameters on $\bg_2$. Conditioning does produce a dependence of $p^{(f(x)|1+2)}$ on $\bg_2$ in Equation \cref{eq:pred_1+2_mean_t}. However, note that according to Equation \cref{eq:pred_1+2_cov_t}, the covariance matrix $Q^{(f(x)|1+2)}$ is independent of $\bg_2$.

\subsection{Acquisition Function: Expected Log-Probability Density}\label{subsec:objectivefn}

We may begin constructing the TAD acquisition function by considering the log-predictive probability density of the target design $f_T$ at the point $x$, conditioned on observations $(\bg_1,\bg_2)$. By Equation \cref{eq:pred_1+2_t} and the standard formula for multivariate normal probability density, this is
\begin{multline}
\mathcal{L}_P(x,\bx_1,\bg_1,\bx_2,\bg_2)=-\frac{1}{2}\log\det Q^{(f(x)|1+2)}\\
-\frac{1}{2}\left(f_T-p^{(f(x)|1+2)}\right)^{T}\left(Q^{(f(x)|1+2)}\right)^{-1}\left(f_T-p^{(f(x)|1+2)}\right),\label{eq:loglik_g2}
\end{multline}
up to an inessential additive constant.

The motivation for choosing this function as the starting point in the construction of the TAD acquisition function is that it measures departure of the current model prediction for $f(x)$ from the target $f_T$ in a way that fully incorporates the current uncertainties concerning the output features $f_i(x)$, $i=1,\ldots,E$. This is a desirable property in an acquisition function, because the unweighted residual measures $r_i\equiv|p^{(f(x)|1+2)}_i-f_{T,i}|$ do not reveal by their relative magnitudes whether the current solution is in fact closer to the target along feature $i$ than along feature $j$, where $i\neq j$.  We may have $r_i>r_j$ so that the $j$-th feature appears closer to the target, but if the $j$-th feature should happen to be much better constrained by the data than the $i$-th feature (so that the corresponding variances satisfy $\sigma^2_i\gg\sigma^2_j$) then the actual discrepancy would be much larger along feature $j$ than along feature $i$, so that (for example) it would be wasteful to acquire data directed at further reducing the $i$-discrepancy at the expense of reducing the $j$-discrepancy. In fact the weighted $j$-discrepancy may be so large as to reveal that the current target point $x$ is in an unpromising region of the control space, a fact not discernible from $r_j$ alone.  A lack of such discrimination between the current relative importances of discrepancies along different features necessarily affects acquisition functions based on the $L_2$-norm discrepancy (Euclidean ``distance'' $d\equiv\left(\sum_{i=1}^Er_i^2\right)^{1/2}$), such as the expected improvement (EI)-based functions of \cite{ranjan2016inverse,zhang2019sequential}. By contrast, the log predictive probability density of $f(x)=f_T$, given in Equation \cref{eq:loglik_g2} naturally weights feature discrepancies by the corresponding uncertainties.

Broadly speaking, we intend to proceed by maximizing a quantity derived from
$\mathcal{L}_P()$ with respect to the target point $x$ and with respect to the
proposed sample points $\bx_2$ in order to locate a region where the target
design $f_T$ ``has greatest probability'' according to the GP model. The
quantity $\mathcal{L}_P()$ is not directly suitable for this task because TAD
is to be an iterative procedure in which at each iteration $\bg_1$ represents
``data in the can,'' whereas $\bg_2$ represents latent data---that is, data not
yet acquired, from proposed points $\bx_2$ whose optimality is still under
evaluation. Since $\mathcal{L}_P()$ depends on $p^{(f|1+2)}$, which in turn
depends on $\bg_2$, we have no way to evaluate $\mathcal{L}_P()$ directly. As noted previously, the dependence of $\mathcal{L}_P()$ on $Q^{(f|1+2)}$ presents
no problem with respect to latent data, since $Q^{(f|1+2)}$ depends on $\bx_2$
but not on $\bg_2$.

We may draw inspiration from the procedure adopted in similar circumstances in BO \cite{7352306,jones1998efficient,bingham2014design} or OED \cite{lindley1956measure,10.2307/2246015}, where some latent measure of improvement or optimality that depends on not-yet-acquired data is averaged over the values that the future data may attain, weighted by the probability ascribed to those values by the current probabilistic model. In the context of TAD, the analog of this procedure is easily stated: The TAD acquisition function is the ELPPD,
\begin{equation}
\mathcal{L}_{TAD}(x,\bx_1,\bg_1,\bx_2)\equiv E_{\bg_2|\bg_1}\left\{\mathcal{L_P}(x,\bx_1,\bg_1,\bx_2,\bg_2)\right\},\label{eq:TAD_Obj_1}
\end{equation}
where the \textit{data predictive distribution} $\bg_2|\bg_1$ may be written down directly from Equation \cref{eq:joint1},
\begin{align}
\bg_{2}|\bg_{1} &\sim \mathcal{N}\left(\bp^{(2|1)},\bQ^{(2|1)}\right)\label{eq:datapred_t}\\
\bp^{(2|1)} &\equiv \bmu_2+\bK_{21}\left(\bK_{11}+\bSigma_1\right)^{-1}(\bg_{1}-\bmu_1)\label{eq:datapred_mean_t}\\
\bQ^{(2|1)} &\equiv \bK_{22}+\bSigma_2-\bK_{21}\left(\bK_{11}+\bSigma_1\right)^{-1}\bK_{12}.\label{eq:datapred_cov_t}
\end{align}

We may now state the main result of this section.

\begin{theorem}[TAD acquisition function]\label{thm:TAD_of}
The TAD acquisition function $\mathcal{L}_{TAD}(x,\bx_1,\bg_1,\bx_2)$ of Equation \cref{eq:TAD_Obj_1} has the following explicit functional form:
\begin{multline}
\mathcal{L}_{TAD}(x,\bx_1,\bg_1,\bx_2)=
-\frac{1}{2}\log\det Q^{(f(x)|1+2)}-\\
\frac{1}{2}\left(f_{T}-p^{(f(x)|1)}\right)^{T}\left(Q^{(f(x)|1+2)}\right)^{-1}\left(f_{T}-p^{(f(x)|1)}\right) \\
-\frac{1}{2}
\mathrm{Trace}\biggl\{
\left(\bK_{x2}-\bK_{x1}\left(\bK_{11}+\bSigma_1\right)^{-1}\bK_{12}\right)
\left(\bQ^{(2|1)}\right)^{-1}\\
\times
\left(\bK_{2x}-\bK_{21}\left(\bK_{11}+\bSigma_1\right)^{-1}\bK_{1x}\right)\left(Q^{(f(x)|1+2)}\right)^{-1}
\biggr\}. \label{eq:TAD_of}
\end{multline}
\end{theorem}

\begin{proof}
Comparing Equations \cref{eq:loglik_g2,eq:TAD_Obj_1},  we can see that  we require moments up to order 2 of $p^{(f(x)|1+2)}$ under $\bg_2|\bg_1$.
We begin by expanding Equation \cref{eq:pred_1+2_mean_t} using (i) of \cref{thm:gpred_update}:
\begin{multline}
p^{(f(x)|1+2)} = p^{(f(x)|1)} \\
+\left[\bK_{x2}-\bK_{x1}\left(\bK_{11}+\bSigma_1\right)^{-1}\bK_{12}\right]
\left(\bQ^{(2|1)}\right)^{-1}
\left(\bg_2-\bp^{(2|1)}\right).
\label{eq:gpred_update_t}
\end{multline}

We therefore have
\begin{equation}
E_{\bg_2|\bg_1}\left\{p^{(f(x)|1+2)}\right\}=p^{(f(x)|1)}
\label{eq:moment_1}
\end{equation}
and
\begin{multline}
E_{\bg_2|\bg_1}
\left\{
\left(p^{(f(x)|1+2)}-p^{(f(x)|1)}\right)
\left(p^{(f(x)|1+2)}-p^{(f(x)|1)}\right)^T
\right\} \\
= 
\left[
\bK_{x2}
-\bK_{x1}\left(\bK_{11}+\bSigma_1\right)^{-1}\bK_{12}
\right]\\
\times\left(\bQ^{(2|1)}\right)^{-1}
\left[
\bK_{2x}
-\bK_{21}\left(\bK_{11}+\bSigma_1\right)^{-1}\bK_{1x}
\right].
\label{eq:moment_2}
\end{multline}

Inserting Equations \cref{eq:moment_1,eq:moment_2} into Equations \cref{eq:loglik_g2,eq:TAD_Obj_1}, we get Equation \cref{eq:TAD_of}.
\end{proof}

A further development that is numerically beneficial consists of the transformation of the term $Q^{(f(x)|1+2)}$ in Equation \cref{eq:TAD_of}.

\begin{theorem}\label{thm:Qthm}
The term $Q^{(f(x)|1+2)}$ in Equation \cref{eq:pred_1+2_cov_t} may be written in update form as follows:
\begin{equation}
Q^{(f(x)|1+2)}=Q^{(f(x)|1)}-T,\label{eq:Qthm}
\end{equation}
where
\begin{equation}
T\equiv\left[
\bK_{x2}-\bK_{x1}\left(\bK_{11}+\bSigma_1\right)^{-1}\bK_{12}
\right]
\left(\bQ^{(2|1)}\right)^{-1}
\left[
\bK_{2x}-\bK_{21}\left(\bK_{11}+\bSigma_1\right)^{-1}\bK_{1x}
\right].
\label{eq:correction_term_t}
\end{equation}
\end{theorem}

The proof is a straightforward application of part (ii) of \cref{thm:gpred_update} to Equation \cref{eq:pred_1+2_cov_t}.

The numerical benefit of writing $Q^{(f(x)|1+2)}$ in update form stems from the advantage of solving the linear problem associated with the term $\left(Q^{(f(x)|1+2)}\right)^{-1}$ in Equation \cref{eq:TAD_of} in the vector space $W_2$ only (i.e., dealing only with the term $\left(\bQ^{(2|1)}\right)^{-1}$) instead of in the (typically much larger) space $W_{1+2}$. 

The quantity $T$ carries all the dependence on the $\bx_2$, so that it is convenient to combine Equations \cref{eq:TAD_of,eq:Qthm,eq:correction_term_t} to write the TAD acquisition function as follows:
\begin{multline}
\mathcal{L}_{TAD}(x,\bx_1,\bg_1,\bx_2)=
-\frac{1}{2}\log\det \left[Q^{(f(x)|1)}-T\right]\\
-\frac{1}{2}\left(f_{T}-p^{(f(x)|1)}\right)^{T}
\left(Q^{(f(x)|1)}-T\right)^{-1}\left(f_{T}-p^{(f(x)|1)}\right) \\
-\frac{1}{2}
\mathrm{Trace}\biggl\{
T\left(Q^{(f(x)|1)}-T\right)^{-1}
\biggr\}. \label{eq:TAD_of2}
\end{multline}

We will use Equation \cref{eq:TAD_of2} as the  acquisition for our optimization
problem(s), as discussed in further detail in the next section. Thus the \textit{TAD optimization problem} is
\begin{equation}
\argmax_{x,\bx_2}\mathcal{L}_{TAD}(x,\bx_1,\bg_1,\bx_2)+\mathcal{L}_X(x,\bx_2),\quad
\mathrm{subject\ to}\quad x\in X, \bx_2\in X^{N_2}.
\label{eq:TAD_opt}
\end{equation}
In \cref{eq:TAD_opt} we have defined a penalty function $\mathcal{L}_X(x,\bx_2)$ whose purpose is to keep the points $x$, $\bx_2$ from straying too far from the domain $X$. This term is identically zero when all its arguments are strictly inside $X$, is strictly negative, and grows in magnitude as any of its arguments outside of $X$ increases its distance from the boundary of $X$.

One property of expressions such as Equation \cref{eq:TAD_of2} that is not obvious at first is that in principle they remain finite even in the noise-free ($\bSigma_1=0$, $\bSigma_2=0$) case when some of the $\bx_2$ points redundandtly approach $\bx_1$ points, despite the fact that the expression for $T$ contains the term $\left(\bQ^{(2|1)}\right)^{-1}$, which is not even defined in this limit, because the predictive variances of the redundant $\bg_2$ data vanish, making the inverse predictive covariance singular. The fact that $\mathcal{L}_{TAD}()$ remains finite nonetheless in this limit is understandable from an information-theory perspective, since all the ostensibly singular terms in $\mathcal{L}_{TAD}()$ arise in connection with computations of the predictive covariance $Q^{(f(x)|1+2)}$ (including the correction term $T$), and this covariance corresponds to a predictive distribution that should ignore the arrival of redundant information, rather than as representative of the kind of crisis requiring a descent into singularity. So we should expect points in $\bx_2$ that approach redundancy by encroaching on points in $\bx_1$ to drop out of the computation in the noise-free limit. This is in fact the case, as we establish in \cref{thm:noise-free}. It follows that singularities that may be encountered in the TAD algorithm in these sorts of redundant-approach circumstances are purely numerical and must be dealt with as numerical issues.

The function $\mathcal{L}_{TAD}()$ has some other interesting features. The
proposed future acquisition points $\bx_2$ have unknown $\bg_2$ data but a known predictive
covariance $Q^{(f|1+2)}$, and this covariance naturally takes up station in the
log-determinant term (first) and in the data-fit term (second). In default of
the $\bg_2$ data, the role of the mean in the data-fit term is played by the
predictive mean $p^{(f|1)}$. The third, trace, term is manifestly
negative-definite and represents the uncertainty-broadening effect of the
unknown latent $\bg_2$-values. 

Assuming that the GP model adequately represents the experimental response function $f()$, then, roughly speaking, the form of $\mathcal{L}_{TAD}()$ leads to two available paths for its maximization, which we will refer to as the ``Hopeful Path'' and the ``Hopeless Path'' below. For the purpose of this discussion, it is helpful to think of the log-determinant term and the data-fit term as grouped together, since they constitute the logarithm of a normalized Gaussian. We will refer to the sum of those two terms as the ``log-Gaussian'' term and contrast them to the third, ``trace,'' term.

\textit{Hopeful Path:} If there exists a point $x_0\in X$ such that $f(x_0)\approx f_T$ (as measured by the scale set by eigenvalues of the covariance $Q^{(f(x_0)|1+2)}$), then $\mathcal{L}_{TAD}()$ can be efficiently maximized by maximizing the log-Gaussian term. This can happen if the target point $x$ approaches $x_0$, so that $p^{(f(x)|1)}\approx f_T$, and the data-fit term in Equation \cref{eq:TAD_of2} is as close to zero as possible. Simultaneously there is some advantage in arranging the settings $\bx_2$ so as to reduce the eigenvalues of $Q^{(f(x_0)|1+2)}=Q^{(f(x_0)|1)}-T$ somewhat with respect to those of $Q^{(f(x_0)|1)}$, in order to grow the log-determinant term without paying too much of a data-fit penalty. Viewed in terms of the $E$-dimensional  multivariate normal distribution in $\mathbb{R}^E$, this amounts to having enough freedom to move the  mean vector $p^{(f(x)|1)}$ close to the point  $f_T$ and adjusting the covariance  $Q^{(f(x_0)|1+2)}$ (through the variation of the points $\bx_2$), so as to maximize the density of the distribution at $f_T$.

\textit{Hopeless Path} On the other hand, if there is no such point $x_0$ in $X$---or there is no longer reason to believe that such a point exists---then the log-normal term can no longer be maximized in this manner.  In this case there is no benefit in reducing eigenvalues of $Q^{(f(x_0)|1+2)}$ with respect to those of $Q^{(f(x_0)|1)}$, since this will only increase the data-fit term. The only possibility is to reduce the matrix elements of $T$ as much as possible, in order to shrink the magnitude of the trace term. Doing so also reduces the covariance $Q^{(f(x_0)|1+2)}$ to $Q^{(f(x_0)|1)}$, so that the variances of the log-normal term are as large as possible, making that term as small as possible.

As obscure as these observations may appear at present, they are in fact key to understanding the exploration/exploitation balance embodied by the TAD algorithm, and they are also essential to the functioning of the TAD convergence rules. We will return to this point after explaining the role played by expected information gain in TAD.

\subsection{Expected Information Gain}\label{subsec:EI}

As we remarked above, there is a clear parental lineage connecting TAD to, among other algorithms, optimal experiment design. Since the work of Lindley \cite{lindley1956measure}, information theoretic gain functions have played a prominent role in the subject (see, e.g.,  \cite{ryan2003estimating,zhang2016adaptive}). In the context of normal models, optimizing expected information gain is equivalent to the criterion of ``D-optimality,'' that is, to minimization of the determinant of the predictive covariance \cite{10.2307/2246015}.

It turns out that the TAD acquisition function bears a close relation to expected information gain (EIG) and that relationship is exploitable to establish a convergence criterion. As we will now show, a quantity closely related to $\mathcal{L}_{TAD}()$ can tell us how much information we may expect to gain---in bits, or nats---about the experimental response function from the samples to be acquired at the next iteration. We will use this in the TAD algorithm to decide whether the algorithm should give up the search: if the expected information gain from the next iteration is less than some settable threshold and a solution has not yet been located, the algorithm will conclude that the uncertainty in the function has been reduced to a sufficient extent that no solution is in prospect, and it will terminate the search.

We define the expected information gain $\mathcal{I}(x,\bx_1,\bg_1,\bx_2)$ as follows:
\begin{equation}
\mathcal{I}(x,\bx_1,\bg_1,\bx_2)=
E_{\bg_2|\bg_1}\left\{
KL\left[
\mathcal{N}\left( p^{(f(x)|1+2)},Q^{(f(x)|1+2)} \right)
\,||\,
\mathcal{N}\left( p^{(f(x)|1)},Q^{(f(x)|1)} \right)
\right]
\right\}.
\label{eq:EI_1}
\end{equation}
Here, $KL(P||Q)$ denotes the ``Kullback--Leibler divergence'' or relative entropy between the two distributions $P$, $Q$ that constitute its arguments \cite{kullback1951information}:
\begin{equation}
KL(P||Q)\equiv\int dP \log\left(\frac{dP}{dQ}\right)
=\int d\mu(z)\,\pi_P(z)\log\left(\frac{\pi_P(z)}{\pi_Q(z)}\right),
\label{eq:KL}
\end{equation}
where the distributions $P$ and $Q$ are assumed mutually absolutely continuous, and hence possessed of a Radon--Nikodym derivative $dP/dQ$, and also absolutely continuous with respect to a third measure $\mu(z)$, with respect to which they have densities $\pi_P(z)$, $\pi_Q(z)$, yielding the more conventional expression for entropy in the second equality of Equation \cref{eq:KL}

This quantity has well-understood and desirable information-theoretic
properties, among which the fact that it is positive-semidefinite and zero only
if $P$ and $Q$ agree almost everywhere as probability measures. As applied in Equation \cref{eq:EI_1}, it furnishes a measure of the information payload in the points $(\bx_2,\bg_2)$ responsible for transforming $\mathcal{N}\left( p^{(f(x)|1)},Q^{(f(x)|1)} \right)$ into $\mathcal{N}\left( p^{(f(x)|1+2)},Q^{(f(x)|1+2)} \right)$. This is measured in bits if the logarithm is base-2 or in ``nats'' if it is a natural log.The expression for the KL divergence between two multivariate normal distributions is given by \S A.5 of \cite{RandW}:
\begin{multline}
KL\left[
\mathcal{N}\left( p^{(f(x)|1+2)},Q^{(f(x)|1+2)} \right)
\,||\,
\mathcal{N}\left( p^{(f(x)|1)},Q^{(f(x)|1)} \right)
\right]
=\\
\frac{1}{2}\log\det\left[
Q^{(f(x)|1+2)\,-1}Q^{(f(x)|1)}
\right]\\
+\frac{1}{2}\mathrm{Trace}
\Bigg[
Q^{(f(x)|1)\,-1}
\bigg(
\left[
p^{(f(x)|1+2)}-p^{(f(x)|1)}
\right]
\left[
p^{(f(x)|1+2)}-p^{(f(x)|1)}
\right]^T\\
+Q^{(f(x)|1+2)} - Q^{(f(x)|1)}
\bigg)
\Bigg].
\label{eq:KL_2}
\end{multline}

If we insert Equation \cref{eq:KL_2} into Equation \cref{eq:EI_1}, we may immediately perform the expectation value using Equation \cref{eq:moment_2}. At the same time we may exploit the update formula Equation \cref{eq:Qthm} and the definition of the quantity $T$ in Equation \cref{eq:correction_term_t} to obtain the more compact form
\begin{align}
\mathcal{I}(x,\bx_1,\bg_1,\bx_2)&=
-\frac{1}{2}\log\det\left[
\left(Q^{(f(x)|1)}-T\right)Q^{(f(x)|1)\,-1}
\right]
+\frac{1}{2}\mathrm{Trace}
\left[
Q^{(f(x)|1)\,-1}(T-T)
\right]\nonumber\\
&=-\frac{1}{2}\log\det\left[
1-TQ^{(f(x)|1)\,-1}
\right].
\label{eq:EI_2}
\end{align}

This form is computationally convenient, but it obscures the intuitive meaning of $\mathcal{I}(x,\bx_1,\bg_1,\bx_2)$, which may be rewritten 
\begin{equation}
\mathcal{I}(x,\bx_1,\bg_1,\bx_2)=\frac{1}{2}\log\frac{\det\left( Q^{(f(x)|1)}\right)}{\det\left(Q^{(f(x)|1+2)}\right)}.
\label{eq:EI_3}
\end{equation}
%

%
In this form it is immediately clear that the expected information represents the log of the reduction in the volume of the ``1-$\sigma$'' ellipsoid of the covariance matrix $Q^{(f(x)|1+2)}$ relative to that of the covariance matrix $Q^{(f(x)|1)}$, which is an entirely satisfactory interpretation of the information content of the $\bx_2$ samples with respect to the response function at the control point $x$.

The forms of Equations \cref{eq:EI_2,eq:EI_3} allow us to connect $\mathcal{I}()$ to the discussion of minimization of $\mathcal{L}_{TAD}()$ at the end of \cref{subsec:objectivefn}. Recall that we observed that in the event that there are still plausible regions of $X$ where the response function might hold a solution point $x_0$ satisfying $f(x_0)\approx f_T$ to within the uncertainties defined by $Q^{(f(x)|1+2)}$---that is, the ``Hopeful Path''---maximization of $\mathcal{L}_{TAD}()$ proceeds by moving the target point $x$ toward $x_0$ and reducing the eigenvalues of $Q^{(f(x)|1+2)}$ with respect to those of $Q^{(f(x_0)|1)}$ somewhat.  This is tantamount to increasing the peakedness of the posterior predictive distribution at its center $p^{(f(x)|1)}$, which is to say near $f_T$, and can be accomplished by increasing the eigenvalues of T (compare Equation \cref{eq:TAD_of2}). And since the eigenvalues of $\bm{Q}^{(f(x)|1+2)}$ can continue to shrink, the EIG can continue to be appreciable, by Equation \cref{eq:EI_3}. 

On the other hand, if the exploration of $X$ has reached the stage where no such $x_0$ can reasonably be held to exist---the ``Hopeless Path''---this fact is naturally reflected in the minimization behavior of $\mathcal{L}_{TAD}()$ by a transition to \textit{decreasing} the matrix elements and eigenvalues of $T$. In this regime, it follows from Equation \cref{eq:EI_2} that the EIG is $\frac{1}{2}\mathrm{Trace}\left[TQ^{(f(x)|1)-1}\right]+\mathcal{O}(T^2)$. So each region of control space more or less stabilizes its local posterior uncertainty $\bm{Q}^{(f(x)|1)}$, since this can only change by the small differences $T$ tolerated by the acquisition function optimization, according to Equation \cref{eq:Qthm}; and the EIG declines in each region as more points are found that depress $T$.
It follows that \textit{the signature of the advent of this phase should be a corresponding decline in expected information gain}.

We can now state the sense in which TAD embodies the exploration/exploitation tension. The log-normal term in $\mathcal{L}_{TAD}()$ is sensitive to similarity between the GP surrogate $p^{(f(x)|1)}$ and $f_T$. As long as there continue to exist regions of $X$ where the shape of $f()$ is substantially uncertain or where $f(x)\approx f_T$ is plausible, that term will continue to dominate the maximization.  Simultaneous monitoring of the expected information gain will show that each successive corresponding probe by $\bx_2$ samples yields finite information.  Eventually, one of two outcomes should occur, however. Either a solution region is located and exploited---a point $x$ satisfies $f(x)\approx f_T$ to within tolerances (in a sense to be made explicit below)---or else a transition occurs to a second optimization regime, in which the search is recognized as ``hopeless.''  This regime occurs because $X$ has been explored to the satisfaction of the GP model, and it is recognizable from the decline of the rate of expected information gain.

The TAD algorithm, described in \cref{sec:algorithm}, reflects these observations in an essential way.

\subsection{Gaussian Process Model Choice}\label{subsec:kernels}

We now turn to a few necessary  details concerning the VVGP models adopted within TAD. As outlined in \cref{subsec:setup}, the GP model is established by specifying a prior vector-valued mean function $\mu:X\rightarrow\mathbb{R}^E$ and a prior matrix-valued covariance function $C:X\times X\rightarrow\mathbb{R}^E\times\mathbb{R}^E$ that is positive-definite as an integral kernel. As is usual in GP modeling, both $\mu(\cdot)$ and $C(\cdot,\cdot)$ have hyperparameters that permit the model to adjust itself flexibly to the structure of the experimental response function inferred from the data, and we avail ourselves of that flexibility within the TAD algorithm at every iteration.

We choose a constant mean function $\mu(\cdot)$, with different constants $\mu_i$, $i=1,\ldots,E$ for each task component. These are hyperparameters optimized at each stage of the algorithm.

For our VVGP covariances we choose  ``sums of separable kernels''  \cite{alvarez2012kernels}. These are constructed from sums of Kronecker products of the form
\begin{equation}
C_{ii^\prime}(x,x^\prime)=\sum_{l=1}^P K^{(l)}\left(x,x^\prime;\theta\right)\times\kappa^{(l)}_{ii^\prime}(\theta).
\label{eq:kronecker}
\end{equation}
In Equation \cref{eq:kronecker}, the $K^{(l)}\left(x,x^\prime;\theta\right)$ are
conventional scalar GP kernels parameterized by some parameters $\theta$, while
the $\kappa^{(l)}_{ii^\prime}(\theta)$ are symmetric positive-definite matrices,
also parameterized. The choice of form of $K()$ has the usual flexibility that is
customary in GP modeling \cite[Chapters 4 and 5]{RandW}. The selection of
$\kappa()$ also leaves much freedom, discussed at length in
\cite{alvarez2012kernels}. 

The number of components $P$ is a modeling choice. In our view a choice of $P\ge
2$ seems advisable. The reason is that in multitask inference settings it is
well known that in the case $P=1$, where the covariance is a simple Kronecker
product, the only interdependence between different task components is through
the noise; that is, in the limit of vanishing noise the different vector
components of the model become statistically independent
\cite{williams2007multi,alvarez2012kernels}. We feel that one may be needlessly
offering a hostage to fortune by relying on noise terms rather than on model
terms to ensure that learned vector components exhibit the correlations required
of them by the data. By choosing $P\ge 2$ this possibility is avoided. 

No principled reason exists why the kernel components $K^{(l)}\left(x,x^\prime;\theta\right)$ corresponding to different values of $l$ should not correspond to different functional forms. However, we have found that dynamically increasing $P$ is an effective way to improve model quality when validation indicators suggest a model refinement is required, as we describe in \cref{subsec:validation2}. We therefore use the same functional form for all the $K^{(l)}()$, for simplicity. Each Kronecker component in the sum \cref{eq:kronecker} has its own parameters $\alpha\in\mathbb{R}^r$, so that the full hyperparameter set is $\theta\in\mathbb{R}^{rP}$.

The parameters $\mu\in\mathbb{R}^E$ and $\theta\in\mathbb{R}^{rP}$ are to be determined at each iteration of the algorithm by the customary GP procedure of maximizing the marginal log-likelihood \cite[Chapter 2]{RandW}:
\begin{multline}
\mathcal{L}_{GP}(\mu,\theta)=-\frac{EN_1}{2}\log(2\pi) -\frac{1}{2}\log\det\left[\bK_{11}(\theta)+\bSigma_1\right]\\
-\frac{1}{2}\left[\bg_1-\bmu(\mu)\right]^T\left[\bK_{11}(\theta)+\bSigma_1\right]^{-1}\left[\bg_1-\bmu(\mu)\right],
\label{eq:gploglik}
\end{multline}
which is to say they are given by the \textit{GP optimization problem}
\begin{equation}
\argmax_{\mu,\theta}{\mathcal{L}_{GP}(\mu,\theta)},\quad\mathrm{subject\ to}\quad
\mu\in\mathbb{R}^E,\,\theta\in\mathbb{R}^{rP}.
\label{eq:gpopt}
\end{equation}

\subsection{Model Validation}\label{subsec:validation}

If the experimental response function has complicated structure, early attempts at GP model selection likely could be unsuccessful at capturing that structure, especially if the control and design space are relatively high dimensional.   Therefore one must have at hand tools that enable the assessment of model adequacy.

The black-box function optimization method of \cite{jones1998efficient} features a GP model together with a LOO cross-validation approach, in which a ``test set'' is effectively created by leaving out individual points, the model is trained on the remaining points, and model validity is assessed on the basis of whether the residuals of the test points are plausible in their predictive normal distributions.

We modify this approach here, taking advantage of the fact that TAD has natural test sets in the samples $\bg_2$ acquired at every iteration: those values are first predicted (probabilistically) and then sampled. 

We have a multivariate normal model for the predicted  $\bg_2$ sample (Equations \cref{eq:datapred_t,eq:datapred_mean_t,eq:datapred_cov_t}).  
In virtue of this distribution, we know that if the model is approximately valid, the squared \textit{Mahalanobis distance}
\begin{equation}
Q\equiv\left[\bg_2-\bp^{(2|1)}\right]^T
\left[\bQ^{(2|1)}\right]^{-1}\left[\bg_2-\bp^{(2|1)}\right]
\label{eq:validation_Q}
\end{equation}
must be a random variable governed by a $\chi^2$ distribution with $N_2\times E$ degrees of freedom  \cite[\S4.2.3]{Eadie:100342}.  We use this observation to monitor model adequacy in a quantitative fashion, looking up the right-tail $P$-values of $Q$ and in its $\chi^2$ distribution at each iteration. The $P$-values are expected to be uniformly distributed, $P\sim\mathcal{U}([0,1])$. If the $P$-values span a plausible range in the unit interval,  one can be reassured about model adequacy (or at least fail to be alarmed). If $P$-values are tiny, then  a problem likely  exists.  In general, the meaning of ``tiny'' depends on the number of $P$-value samples acquired---if $P_1\le P_2\ldots\le P_N$ are $N$ i.i.d. draws from $\mathcal{U}([0,1])$, then $E(P_1)=1/(N+1)$ \cite{david2004order}, so that a ``tiny'' $P$-value would be $P\ll 1/(N+1)$.

The trends in $Q$ from iteration to iteration may be informative about the necessity to tinker with the model, as new data reveals more detailed structure in the response function. In fact, one should expect that when fewer $\bg_1$ sample points are available, so that the response function model is noisy and poorly constrained, even a simple model with (say) few Kronecker components $P$ in Equation \cref{eq:kronecker} could be capable of producing reassuring values of $Q$, whereas in later iterations, as sample points accumulate and the response surface becomes better ascertained by the data, imperfections in the surrogate model can be exposed and detected in the tendency of $Q$ to ``peg'' at values near zero.

An additional effect that we noticed in numerical experiments described in \cref{sec:results} is that $Q$ can indicate a  poor model fit in  early iterations of the algorithm to be described, when  few points have been acquired from a limited region of $X$ and the TAD optimization problem \cref{eq:TAD_opt} has led the $\bx_2$ points to a  different part of $X$. This surprised us at first but should have been expected, since the GP optimization solution of \cref{eq:gpopt} based on information from a small region of $X$ cannot be expected to furnish a model that adequately represents the behavior of the experimental response over the whole of $X$. We will describe a strategy  that restores the benefit of the validation measure $Q$ in this case.

In addition to $Q$ we have found that while carrying out the GP optimization problem of Equation \cref{eq:gpopt}, it is  helpful to monitor the quadratic form part of the log-likelihood in Equation \cref{eq:gploglik}, that is,
\begin{equation} 
S\equiv \left[\bg_1-\bmu(\mu)\right]^T\left[\bK_{11}(\theta)+\bSigma_1\right]^{-1}\left[\bg_1-\bmu(\mu)\right].
\label{eq:verification_S}
\end{equation}
This quantity should not, strictly speaking, be viewed as a random variable in the same sense as $Q$. As a heuristic matter, however, one does expect that if the GP model makes sense and the optimization has converged to a reasonable range of parameters, then $S$ should take values consistent with those sampled from a $\chi^2$ distribution with $N_1\times E-E$ degrees of freedom. This is by no means a ``validation,'' but it is a useful sanity check, as well as a valuable convergence criterion for the optimization routine, which is probably wasting computing cycles if it does not halt when the $P$-value from the $\chi^2$ distribution is already at some plausible value in the unit interval.

This concludes our discussion of the theoretical elements. We now assemble them into the TAD algorithm.

\section{The TAD Algorithm}\label{sec:algorithm}

\begin{figure}[t]
\includegraphics[width=0.325\textwidth]{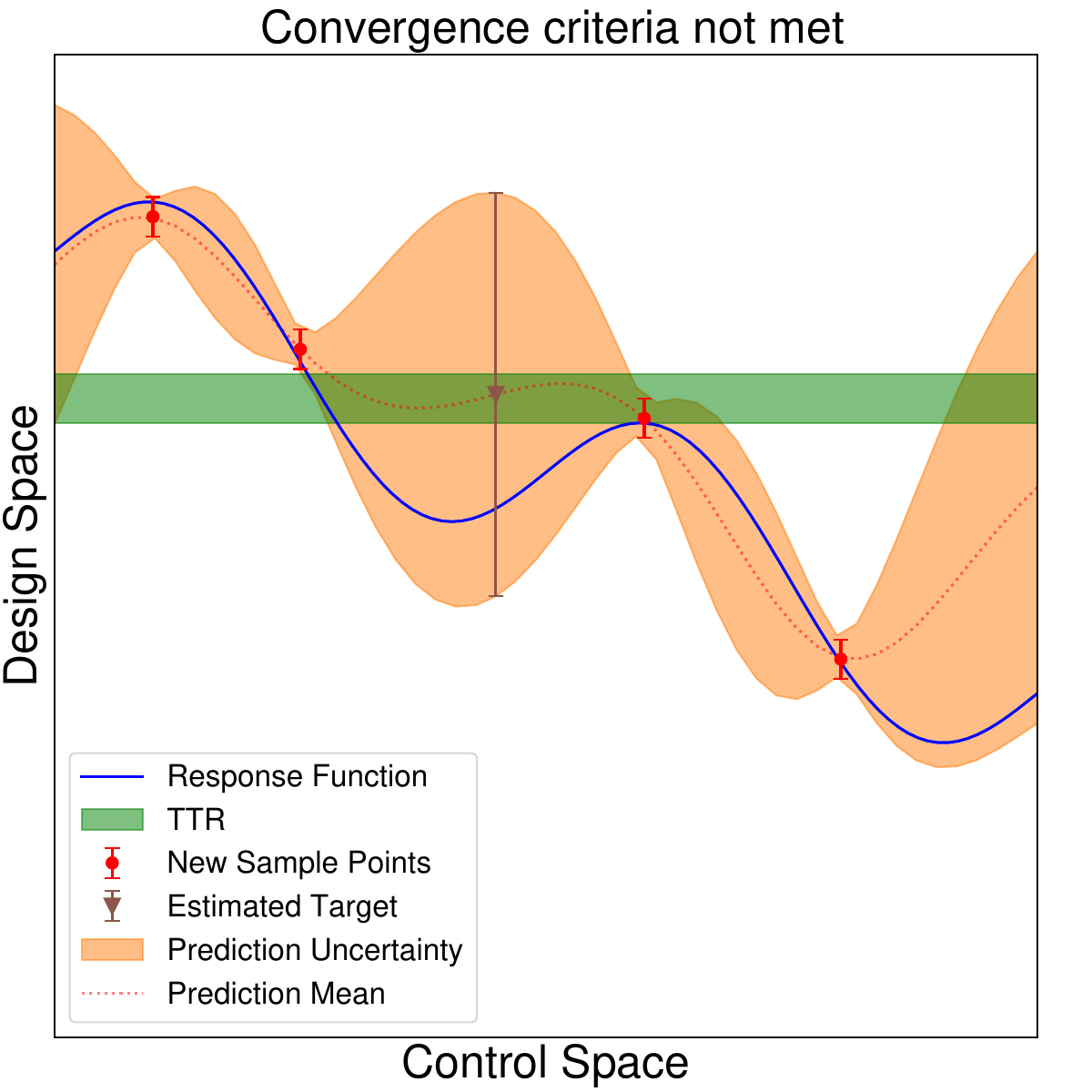}
\includegraphics[width=0.325\textwidth]{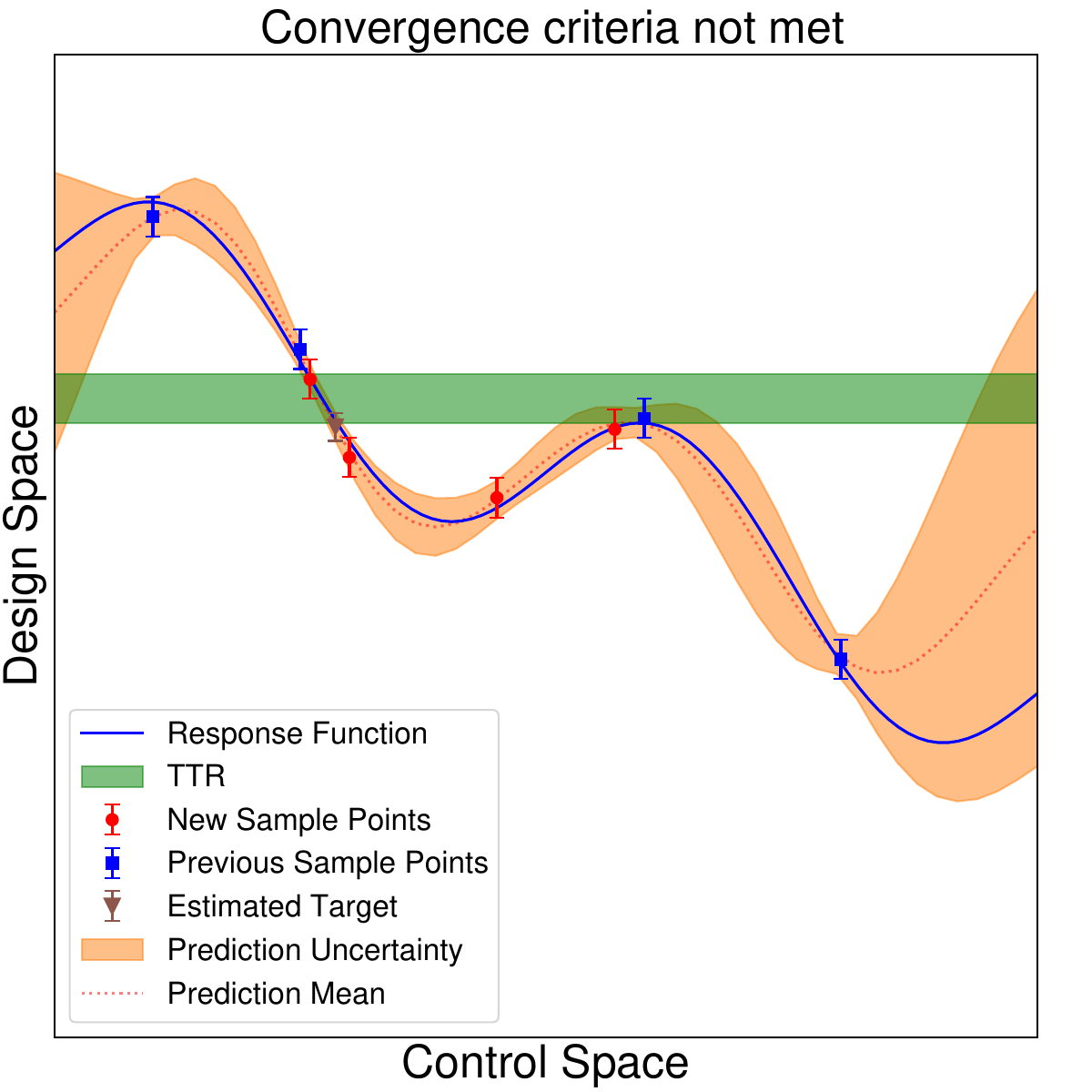}
\includegraphics[width=0.325\textwidth]{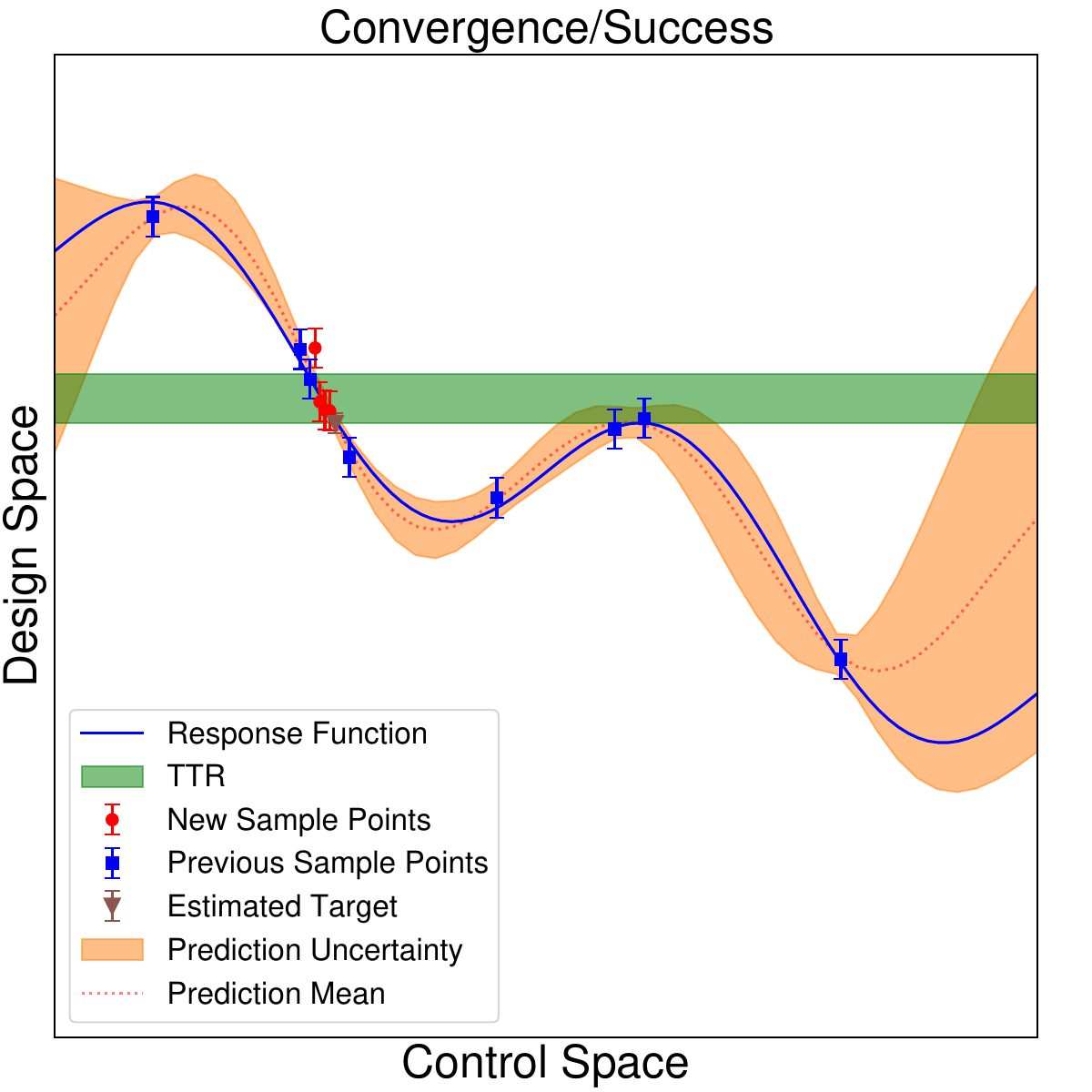}\\
\includegraphics[width=0.325\textwidth]{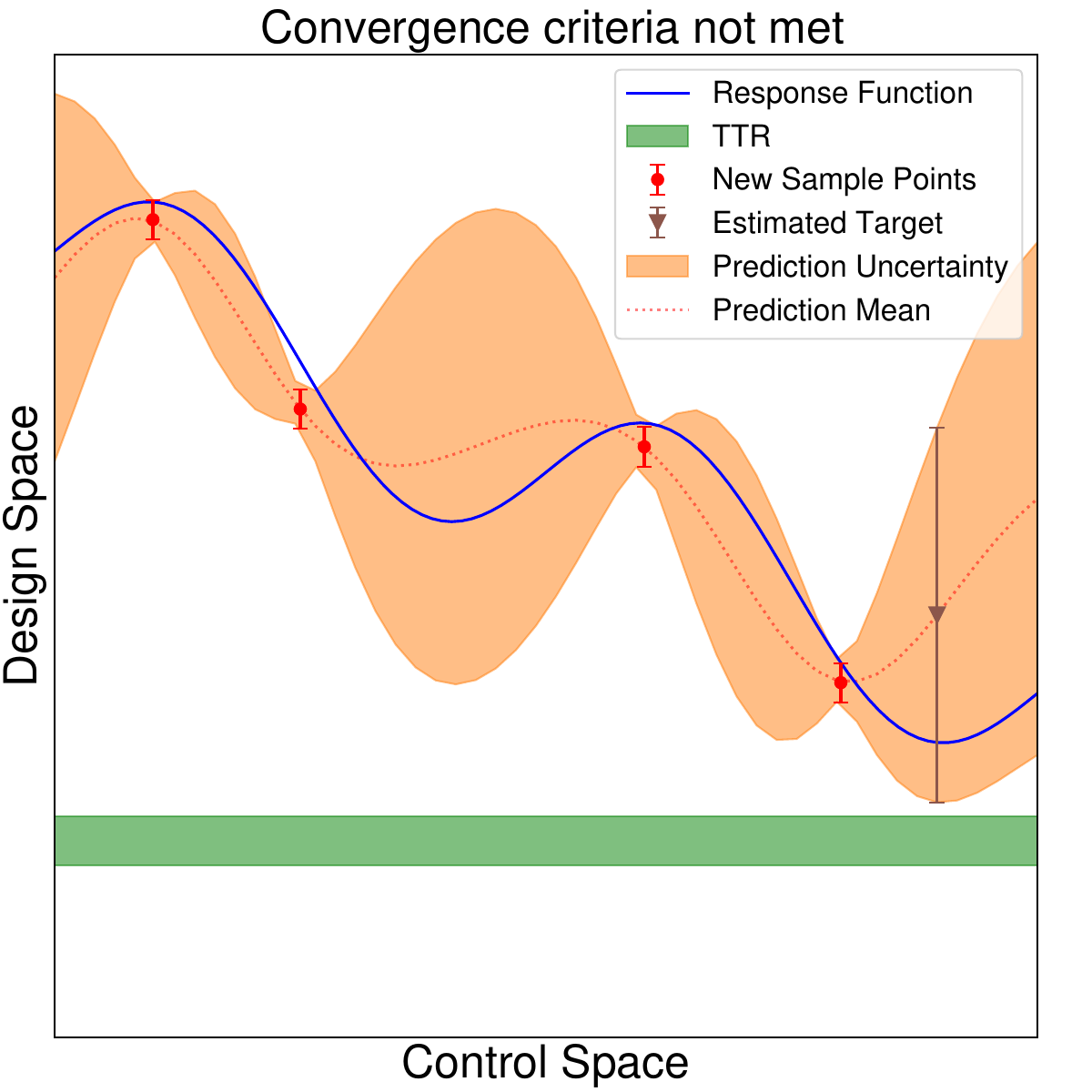}
\includegraphics[width=0.325\textwidth]{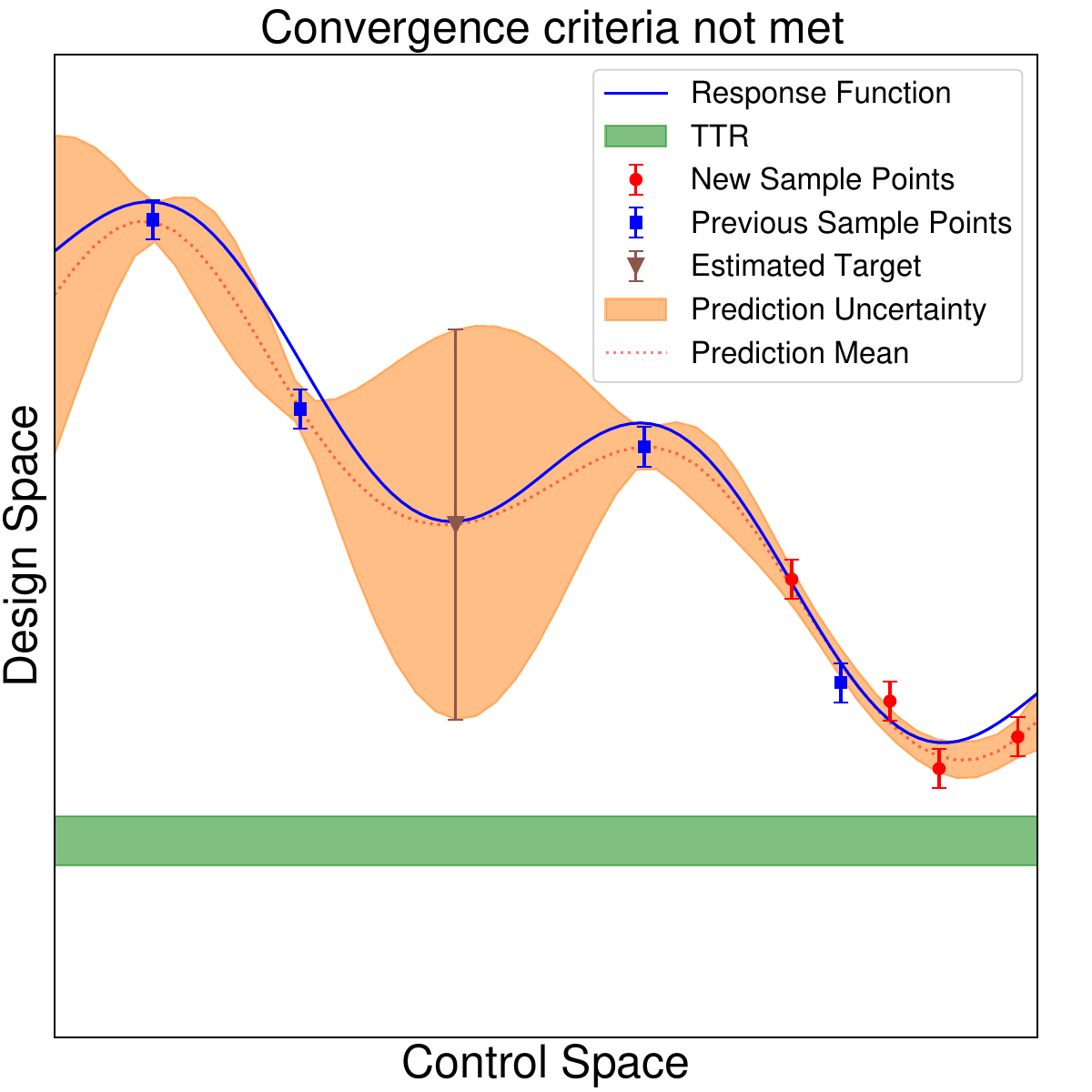}
\includegraphics[width=0.325\textwidth]{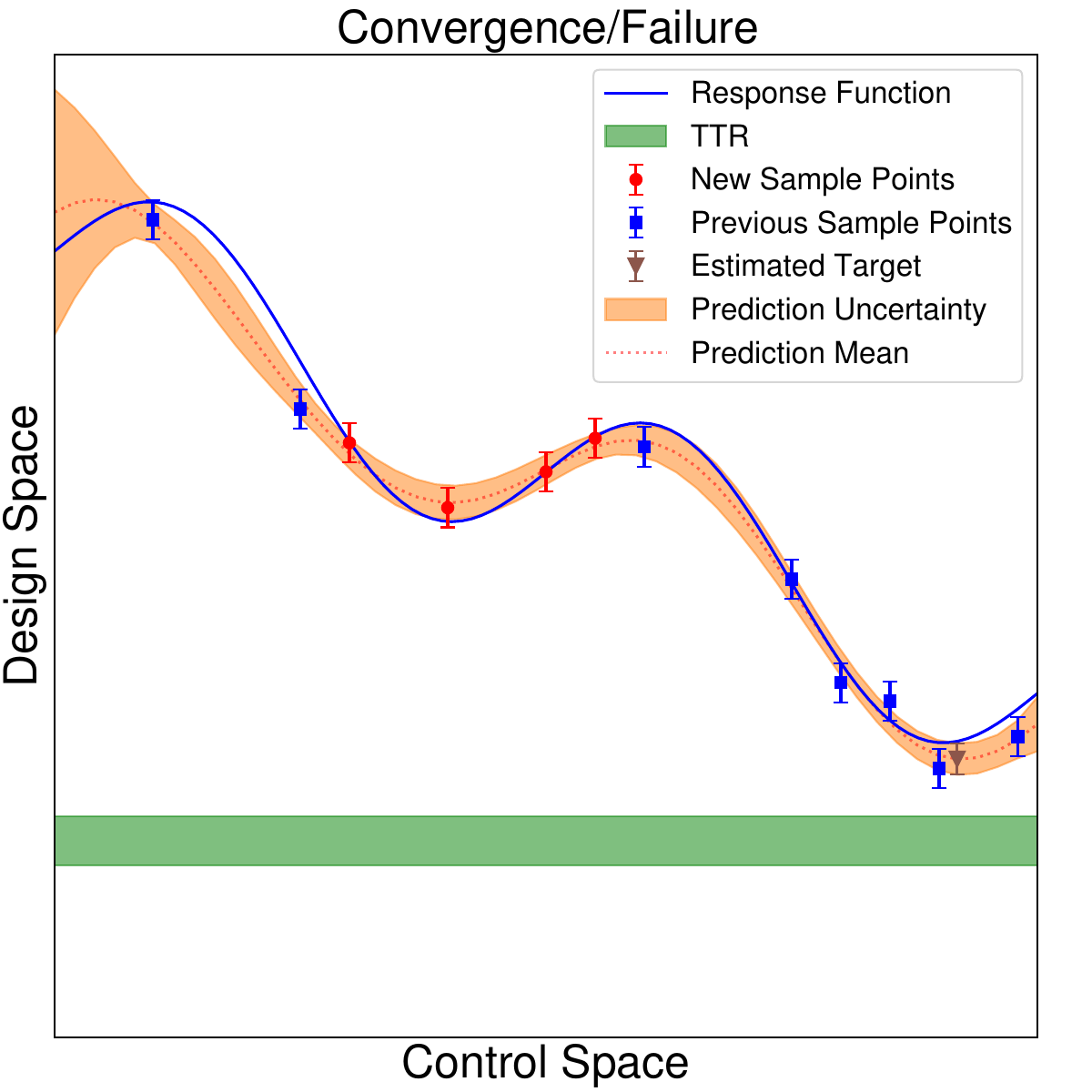}\\
\caption{Schematic illustration of the TAD algorithm. The panels display an idealized problem in which the control space is collapsed onto the $x$-axis and the design space is collapsed onto the $y$-axis. The algorithm samples the control space to discover the shape of the response surface while searching the space for the target design. The top panels illustrate the case where a control setting exists such that the target design exists, within tolerance. The bottom panels illustrate the case where no such design is to be found. \label{figure:Illustration}}
\end{figure}

The TAD algorithm in effect defines the TTR--- a rectangular parallelepiped centered on the target design $f_T$---and iterates,  at each iteration maximizing $\mathcal{L}_{TAD}(x,\bx_1,\bg_1,\bx_2)$ with respect to new sample points $\bx_2$ and a target point $x$, either until  the predicted mean design and its uncertainty (as computed by using Equations \cref{eq:gpred_update_t} and \cref{eq:Qthm}) fit inside the TTR or until it is possible to conclude that no point satisfying this condition exists in the control space $X$.


The process is illustrated in idealized form in Figure \cref{figure:Illustration}, in which the control space $X$ and the design space $Y$ are collapsed to one dimension and displayed as the $x$-axis and $y$-axis of the figure illustrations, respectively. The solid blue curve shows the experimental response, while the TTR is illustrated by the horizontal green band stretching across each plot. New ($\bx_2,\bg_2$) noisy sample points are shown as red circle markers, while old ($\bx_1,\bg_1$) sample points are shown as blue square markers. The Gaussian process model mean is shown by the dotted line, while the 1-$\sigma$ variance model uncertainty is shown by the shaded area. At each stage the control point closest in predictive probability to the target, together with its uncertainty, is shown by the brown triangle marker and error bar.

The top panels illustrate the case where a solution exists in the search space and show how the algorithm is intended to preferentially shrink the uncertainty in the response surface in regions where the predictive probability of the target is highest until the solution is located.  The lower panels illustrate the other possible case: a solution does not exist, possibly because the TTR has been specified too stringently for some design feature.  In this case, TAD shrinks the uncertainty in the function to the point where it can conclude (in a manner more systematic than merely glancing at a graphical illustration!) that no such solution exists.
As this discussion illustrates, optimization of the TAD acquisition function $\mathcal{L}_{TAD}()$ in the algorithm sketched above (and to be described in detail below) provides a tool for balancing the competing objectives of exploration and exploitation, the classic tension between discovering the global shape of the response function and discovering the optimal response point. 

\subsection{Tolerances and Stopping Rules}

Clearly the question of convergence has to do with whether the uncertainty associated with the estimated best target control point fits inside the TTR.
For making this comparison, let us define the \textit{uncertainty box} (UB) of the unknown function in a manner similar to the TTR, which is to say simply a rectangular parallelepiped at each control point $x\in X$ centered on the current GP model mean vector at $x$, $p^{{(f(x)|1)}}$, with sides along each coordinate axis in $Y$ of length equal to twice the corresponding root variance (diagonal element of $Q^{(f(x)|1+2)}$). The reason for defining the UB this way, instead of by an appeal to the principal ellipsoids of the covariance $Q^{(f(x)|1+2)}$), is for ease of comparison with the TTR and because  no real benefit accrues from the algebraic complexity that arises from a too-principled geometric approach here.

Initially the surrogate function is assumed known vaguely, and the (prior) UB at each point is wide enough to contain the TTR. At each iteration TAD chooses sample points to ``squeeze'' the UB selectively in regions of the control space where target design seems most likely, according to the current model. This variance reduction may improve the chances of locating the target in the region of interest, in which case more attention is focused on this region and more sample points are invoked to exploit its solution-bearing potential. Alternatively, the variance reduction may make the target less likely in this region, or even put the UBs of points in this region definitively outside of the TTR, in which case sample points are directed to  more promising regions. 

If a region is discovered to hold a point where a UB fits entirely within the TTR, the algorithm declares convergence and a successful solution. This condition is called ``convergence/success.''

Alternatively, the algorithm can signal that it is giving up on the search by transitioning from the ``Hopeful Path'' to the ``Hopeless Path'' of optimization, a transition that is clearly detectable in a decline in expected information $\mathcal{I}()$ to very small values. By thresholding on $\mathcal{I}()$, TAD declares convergence with no solution. This condition is called ``convergence/failure.''

These considerations are admittedly heuristic. We do not have a proof that $\mathcal{I}()$ must eventually decline and remain below some threshold, nor have we derived the rate at which decline should occur on average.  Such results might be obtainable through results on average and/or optimal \textit{learning curves} describing the rate of decline of GP generalization error with training data size \cite{sollich1998learning,opper1998general,williams2000upper}. Here we confine ourselves to noting that in our numerical experiments such as the ones described in $\S$\ref{sec:results} we observe $\mathcal{I}()$ to decline noisily but steadily enough to provide a useful gauge of solution failure.

Note the considerable benefit that has followed from the information-theory approach: while ``convergence/success'' is a local criterion, ``convergence/failure'' is a \textit{global} statement about the absence of a solution in $X$ satisfying the required tolerances. Ironically, it is much harder to fail properly than it is to succeed! A naive approach to establishing a ``convergence/failure'' criterion---say, an explicit search of $X$, partitioned a priori into search volumes---could be  expensive if the dimension $D$ of $X$ is even moderately large.  As matters stand, the TAD acquisition function $\mathcal{L}()$ provides a built-in efficient strategy to search $X$, and the expected information $\mathcal{I}()$ furnishes a gauge of when that search transitions from hopeful to hopeless, providing exactly the required ``convergence/failure'' criterion.

\subsection{Model Validation}\label{subsec:validation2}

\begin{algorithm}[t]
\caption{Targeted Adaptive Design}\label{alg:TAD}
\begin{algorithmic}[1]
\Require TTR $\etattr$ with component intervals $\eta_{i}\gets\left[f_{Ti}-\tau_{i},f_{Ti}+\tau_{i}\right]$, $i=1,\ldots,E$, where $\tau_{i}$ is the $i$-th component of the tolerance vector $\tau$, and $f_{Ti}$ is the $i$th component of the target vector $f_{T}$
\Require $\mathcal{I}_0$, expected information threshold, and $N_\mathcal{I}$ number of consecutive iterations in which $\mathcal{I}<\mathcal{I}_0$ required to declare convergence/failure
\Require $\bx_1$, a set of initial design points of size $N_1$, and the corresponding samples $\bg_1$
\Require $x$, a target control point initialization
\Require $\bx_2\gets \left\{x_{2,l}\sim\mathcal{N}\left[x,\sigma^2\bm{I}_D\right]:l=1,\ldots,N_2\right\}$, initial cluster of proposed sample settings
\Require  $Converged\gets\textbf{False}$, $Check\_Model\gets\textbf{False}$, $Perturb\gets\textbf{True}$
\Require  $iter\gets 0$, $n\_check\gets 0$, $P\gets 2$, $n_\mathcal{I}\gets 0$
\While{$\textbf{not}\; Converged$}
    \If {$\textbf{not}\; Check\_Model$}
       \State $iter\gets iter + 1$
       \State Solve the GP optimization problem \cref{eq:gpopt} for hyperparameters $\mu$, $\theta$.
    \EndIf
    \State $x^{(0)}\gets x$, $\bx_2^{(0)}\gets\bx_2$
    \State $x,\bx_2\gets \mathrm{TAD\_OPT}\left(x^{(0)}, \bx_2^{(0)}, Perturb\right)$ \Comment{See \cref{alg:TAD_optfn}}
    \State Compute $Q$ using \cref{eq:validation_Q}
    \If{$Q>0.01$}
       \State $Check\_Model\gets\textbf{False}$, $n\_check\gets 0$
       \State Call Procedure \textsc{Check\_Convergence} \Comment{See \cref{alg:TAD_conv}}
        \State Acquire data $\bg_2$ at $\bx_2$ and $g$ at $x$
        \State $\bx_1\gets\bx_1\cup\bx_2\cup x$, $\bg_1\gets\bg_1\cup\bg_2\cup g$; $Perturb\gets\textbf{True}$
    \Else
    \State $n\_check\gets n\_check+1$
        \If{$n\_check == 2$} \Comment{Model validation failure}
            \State $P\gets P+1$ \Comment{Increase number of Kronecker products}
            \State Acquire data $\bg_2$ at $\bx_2$  
            \Statex\Comment{Guard against failing validation due to data inadequacy}
            \State $\bx_1\gets\bx_1\cup\bx_2$, $\bg_1\gets\bg_1\cup\bg_2$
            \State $\bx_2\gets\bx_2^{(0)}$, $x\gets x^{(0)}$, $Perturb\gets\textbf{False}$ \Comment{Restart from previous state}
            \State $n\_check\gets 0$, $Check\_Model\gets\textbf{False}$
        \Else \Comment{Model validation alert, but not yet alarm}
            \State Acquire data $\bg_2$ at $\bx_2$ and $g$ at $x$
            \State $\bx_1\gets\bx_1\cup\bx_2\cup x$,$\bg_1\gets\bg_1\cup\bg_2\cup g$
            \State $Perturb\gets\textbf{True}$, $Check\_Model\gets\textbf{True}$
        \EndIf
    \EndIf
\EndWhile
    \end{algorithmic}
\end{algorithm}
%
\begin{algorithm}[ht]
\caption{TAD Optimization Function}\label{alg:TAD_optfn}
\begin{algorithmic}[1]
\Function{TAD\_OPT}{$x^{(0)}$, $\bx_2^{(0)}$, \textit{Perturb}}
\Statex \Comment{$x^{(0)}$ and $\bx_2^{(0)}$ are initialization values, \textit{Perturb} is Boolean}
\If{\textit{Perturb}}
  \State $\bS_2\gets <(\bx_2^{(0)}-x^{(0)})(\bx_2^{(0)}-x^{(0)})^T>$
  \State Initialize $x\gets x_\epsilon\sim\mathcal{N}\left[x^{(0)},\epsilon\right]$;
  \State Initialize $\bx_2$: $x_{2,1}\gets x_\epsilon\sim\mathcal{N}\left[x^{(0)},\epsilon\right]$; $x_{2,l}\gets z_2\sim\mathcal{N}\left[x_0,\bS_2\right]$, $l=2,\ldots,N_2$
  \Statex\Comment{$\epsilon$ is a ``small'' perturbing covariance}
\Else
  \State Initialize $x\gets x^{(0)}$
  \State Initialize $\bx_2 \gets \bx_2^{(0)}$
\EndIf
\State Solve the TAD optimization problem \cref{eq:TAD_opt} for $x$ and $\bx_2$.
\State \textbf{return} $x$, $\bx_2$
\EndFunction
\end{algorithmic}
\end{algorithm}
\begin{algorithm}[h!]
\caption{TAD Convergence Testing Procedure}\label{alg:TAD_conv}
\begin{algorithmic}[1]
\Procedure{Check\_Convergence}{}
\State Calculate UB $u$:\\ $u_i\gets
\left\{
p^{(f(x)|1)}-\left[\left(Q^{(f(x)|1+2)}\right)_{ii}\right]^{1/2},p^{(f(x)|1)}+\left[\left(Q^{(f(x)|1+2)}\right)_{ii}\right]^{1/2}\right\}$, $i=1,\ldots,E$
\State $Convergence/Success\gets \forall (i\in\mathbb{N};i> 0;i\le E)\,u_i\subset \eta_i$
\State Calculate $\mathcal{I}(x,\bx_1,\bg_1,\bx_2)$ according to Equation \cref{eq:EI_2}
\If {$\textbf{not}\; Convergence/Success\;\textbf{and}\;\mathcal{I}(x,\bx_1,\bg_1,\bx_2)<\mathcal{I}_0$}
\State $n_\mathcal{I}\gets n_\mathcal{I}+1$
\EndIf
\State  $Convergence/Failure\gets n_\mathcal{I} > N_\mathcal{I}$
\State $Converged \gets Convergence/Success\;\textbf{or}\; Convergence/Failure$
\EndProcedure
\end{algorithmic}
\end{algorithm}

Our experience with TAD has led us to be  cautious about trusting conclusions drawn from an unvalidated GP model. Indeed,  more than once our intuition about the expected behavior of the algorithm, which we based on what appeared to be sound mathematics, foundered on problems due to model inadequacy. As we noted in $\S$\ref{subsec:validation}, one should expect a progressive loss of model adequacy as successive iterations of the algorithm introduce more sample data and better knowledge of the structure of the response function. We have therefore built into the algorithm an anticipation of loss of model adequacy with iteration number, as well as a mechanism for dynamic increase of model complexity where required.

Model complexity could be increased in a few different ways, by choosing different types of basis kernel functions, for example. In this work we choose the simplest option, which is to parameterize model complexity by the number $P$ of Kronecker components in Equation \cref{eq:kronecker}. When we need to complexify the model, we add a Kronecker component, thus increasing $P$ by 1.

We monitor model adequacy using the quantity $Q$ defined in Equation \cref{eq:validation_Q}. We set a smallness threshold for this quantity,  0.01, for example, and complexify the model if $Q$ falls below this value. In order to guard against unnecessary complexification, each time a validation test at the end of a TAD optimization produces a value  of $Q$ indicating model inadequacy, we restart the optimization from a different random initialization of the $\bx_2$ sample point locations. If the reinitialized optimization still indicates model inadequacy, then model complexity is increased. Otherwise the event is considered a false alarm.

In addition, prior to each such restart, we acquire the data $\bg_2$ at the $\bx_2$ sample point locations from the TAD optimization solution that ``failed'' validation. The reason is the observation from \cref{subsec:validation} that in early iterations, when only a few $(\bx_1,\bg_1)$ samples covering small regions of $X$ are available, the GP model parameter optimization usually produces surrogate models that are not representative of the behavior of the experimental response function over the whole of $X$. In such cases, it is perfectly possible for a $Q$ value to flunk a validation test, not because of lack of model complexity, but because of data inadequacy.

\subsection{Outline of the TAD Algorithm}

We now discursively outline the TAD algorithm, so that we can discuss various elements below. The  formal definition of TAD is given as \cref{alg:TAD}.

We select an initial design $\bx_1$ of size $N_1$ and acquire the corresponding points $\bg_1$. We have found that if one wishes to allow plenty of opportunity for expected improvement at the outset of the algorithm then
it is useful to cluster the initial points in a region of $X$. However, the initial design is in principle arbitrary, and may span $X$, or be chosen as a Latin Hypercube, for example.


We fix $N_2$, the number of $(\bx_2,\bg_2)$ samples to be acquired at each new iteration (in principle $N_2$ could change adaptively with iteration number to create a more flexible algorithm in the same family).

At the outset (iteration $I=0$), the TTR is constructed, and a convergence/failure threshold $\mathcal{I}_0$ for expected information is set.  Since $\mathcal{I}$ can fluctuate during the optimization and hence can rise back up above $\mathcal{I}_0$ after first crossing the threshold, we also set a parameter $N_\mathcal{I}$, representing the number of consecutive iterations during which it must remain the case that $\mathcal{I}<\mathcal{I}_0$ in order to declare convergence/failure. The initial model complexity $P$ is chosen. The target location $x$ is initialized randomly. The sample point locations $\bx_2$ are selected randomly in some small region near $x$ using a multivariate normal with mean $x$ and a unit-proportional covariance (for simplicity). The corresponding sample values $\bg_2$ are acquired. The initial $\bx_1$-sample set is chosen to be some modest-sized design.

At iteration $I$ the algorithm performs the following steps:
\begin{enumerate}
\item Optimize the GP log-likelihood $\mathcal{L}_{GP}(\mu,\theta)$ with respect to the GP model hyperparameters $\mu$ and $\theta$ (Equation \cref{eq:gpopt}), verifying
that the quadratic form $S$ in Equation \cref{eq:verification_S} takes plausible values, as described in \S{subsec:validation}.
\item\label{item:empcov} Compute the empirical covariance $\bX_2\equiv<(\bx_2-x)(\bx_2-x)^T>$ 
of the previous $\bx_2$-sample locations about the target location (see note below), and sample a set of $N_2-1$ initialization values $\bx_{2,0}\sim\mathcal{N}\left\{x,\bX_2\right\}$. The remaining initialization value is chosen at a small perturbation of $x$, the perturbation size selected to be large enough not to induce numerical instability in computations of covariances with nearly redundant rows.
\item\label{item:tadopt} Optimize the TAD acquisition function $\mathcal{L}_{TAD}(x,\bx_1,\bg_1,\bx_2)$ with respect to the sample locations $\bx_2$ and target location $x$, initializing the optimization target location $x$ from the value of $x$ found at the previous iteration and initializing the $\bx_2$ by $\bx_2=\bx_{2,0}$ (Equation \cref{eq:TAD_opt}).
\item Perform model validation, by computing the quadratic form $Q$ of Equation \cref{eq:validation_Q} and checking whether  it takes plausible values, as described in \cref{subsec:validation2}. If it does not, restart the previous step; and if the failure persists, complexify the model by adding a Kronecker component. Otherwise, proceed.
\item Test for convergence/success condition: construct the UB at $x$, and test whether all of its components fit inside the TTR. If so, declare convergence/success; report the solution control setting $x$, the corresponding design $p^{(f(x)|1)}$, and predictive covariance $Q^{(f(x)|1+2)}$; and stop. Otherwise continue.
\item Test for convergence/failure condition: compute $\mathcal{I}(x,\bx_1,\bg_1,\bx_2)$ using Equation \cref{eq:EI_3}, and test whether $\mathcal{I}(x,\bx_1,\bg_1,\bx_2)<\mathcal{I}_0$. If so and if this has occurred for more than $N_\mathcal{I}$ consecutive iterations, declare convergence/failure, and stop; otherwise continue.

\item Acquire sample of data $\bg_2$ at newly determined sample points $\bx_2$, and target data $g$ from latest estimated target point $x$.
\item Proceed to iteration $I+1$.
\end{enumerate}

At \cref{item:empcov} we have found empirically that it is in general better to compute the scatter of the $\bx_2$ points about the target point $x$ than about the mean $\bx_2$ point. The reason is that in an exploitation phase, there is not much difference, but in exploration the $\bx_2$ points often need to be initialized with greater variance in order to sense  the optimal direction of $\mathcal{L}_{TAD}()$. The target point usually ``leads'' the cloud of $\bx_2$ points during exploration, which elongates the principal axes of the covariance $\bX_2$ that extend in the direction from the $\bx_2$ points to the target point. We have found that this initialization has a beneficial impact on the TAD optimization problem \cref{eq:TAD_opt}.

At \cref{item:tadopt} it is important to initialize the target point $x$ at or very near the solution from the previous iteration. Doing so allows the acquisition function to increase monotonically while exploiting a possible solution. It is also useful to initialize the $\bx_2$ sample locations near $x$ and at remixed values of $\bx_2$, in order to begin adding information to what is already available from the previous iteration.

\subsection{Implementation}

Our implementation of the TAD algorithm is hosted at a GitHub repository \cite{TAD_github} and implemented in Python. The GP portion of the implementation is built on top of GPyTorch 
\cite{GPyTorch,gardner2018gpytorch}. 
The reason for this choice is that the GPyTorch package implements fast kernel methods such as the SKI inducing point method described in \cite{wilson2015kernel}, while exploiting the GPU acceleration facilities built into the PyTorch deep learning framework. These features hold great promise for large GP kernel matrix performance and appear to have broken through the ``curse of dimensionality'' problem of GP modeling and prediction, allowing computation with covariance matrices so large that they must be computed on the fly, because they cannot be affordably stored \cite{gardner2018gpytorch}.

These are choices that we have made with an eye to the future, in view of the potential for  large-scale optimization problems that may arise in the exploration of advanced manufacturing control spaces. For example, Paulson et al.~\cite{PAULSON2020108972} describe exploration of  flame spray pyrolysis experimental controls with $D=6$, in which the design objective is to control features of the particle formation size spectrum---three moments of the spectrum in the case described, although there is no reason a more complex description might not be required.  Even granting that TAD can do much better than a naive sampler, with a fixed grid on each axis, one can easily see how the classic computational costs of GP modeling, which grow as $N^3$ (where $N$ is the dimensionality of the covariances in question), might make such modeling unaffordable in the absence of the new methods.

\begin{figure}[t]
    \centering
    \includegraphics[width=1.0\textwidth]{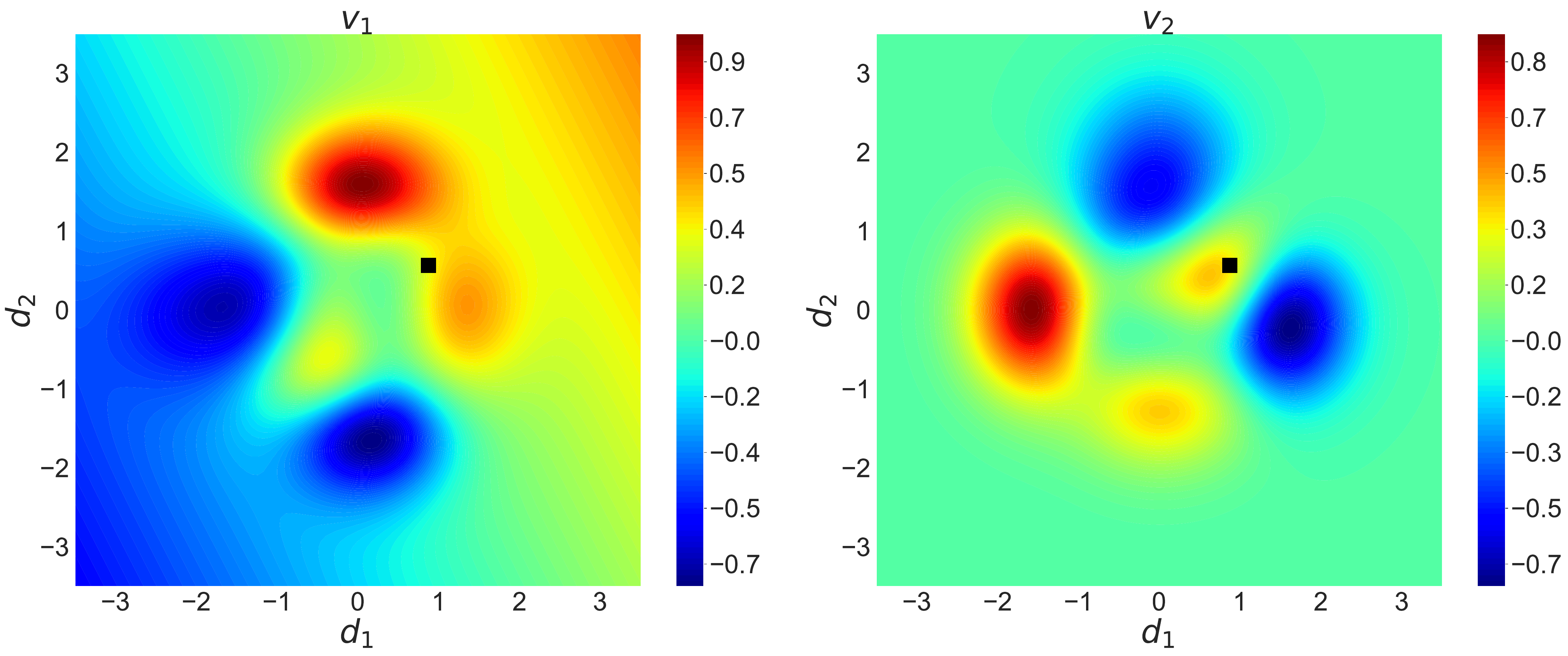}
    \caption{Test function for numerical experiments. The black square is the location of the target value for the ``convergence/success'' test.}
    \label{fig:tgtfun}
  \end{figure}
  

\begin{figure}[h]
\includegraphics[width=0.5\textwidth]{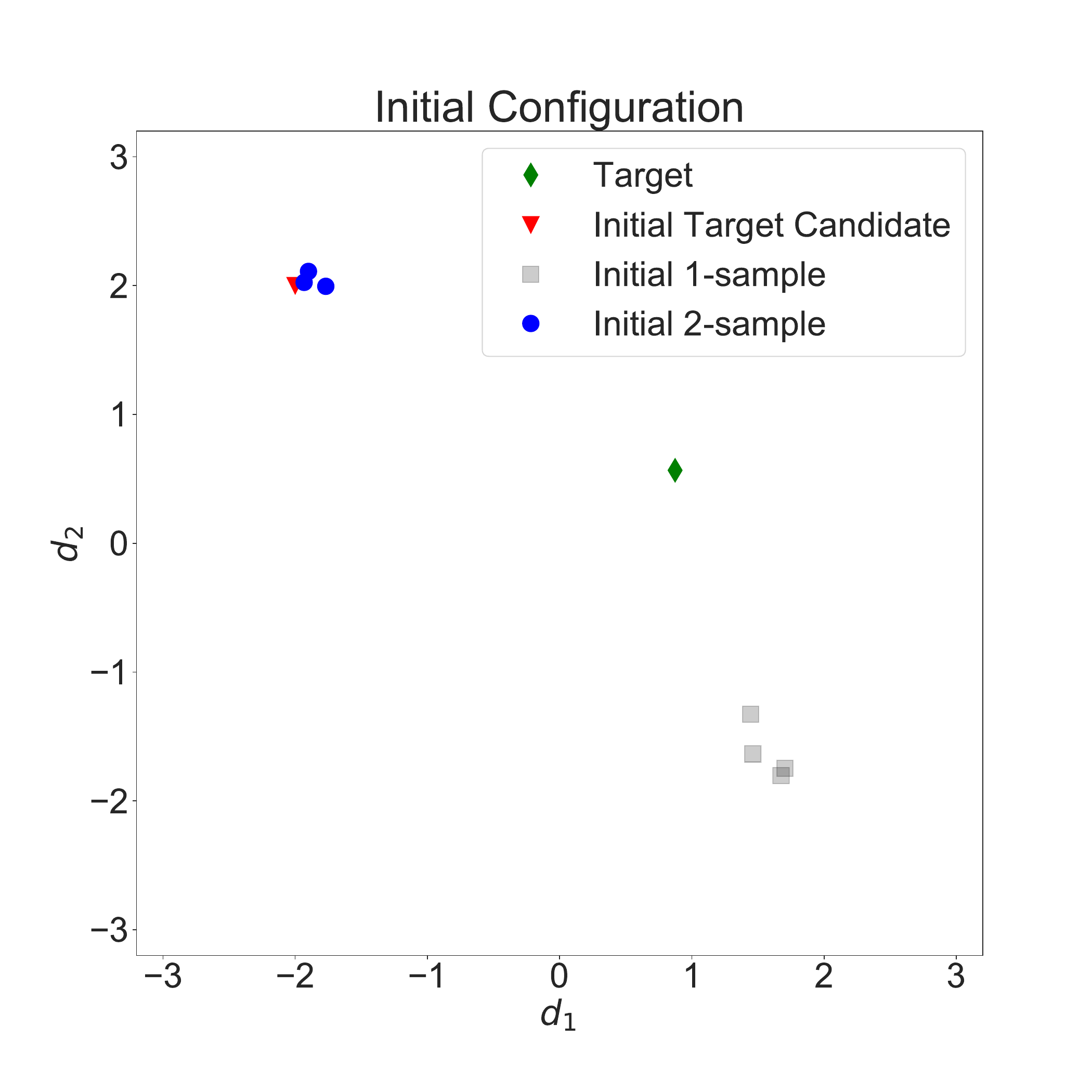}
\includegraphics[width=0.5\textwidth]{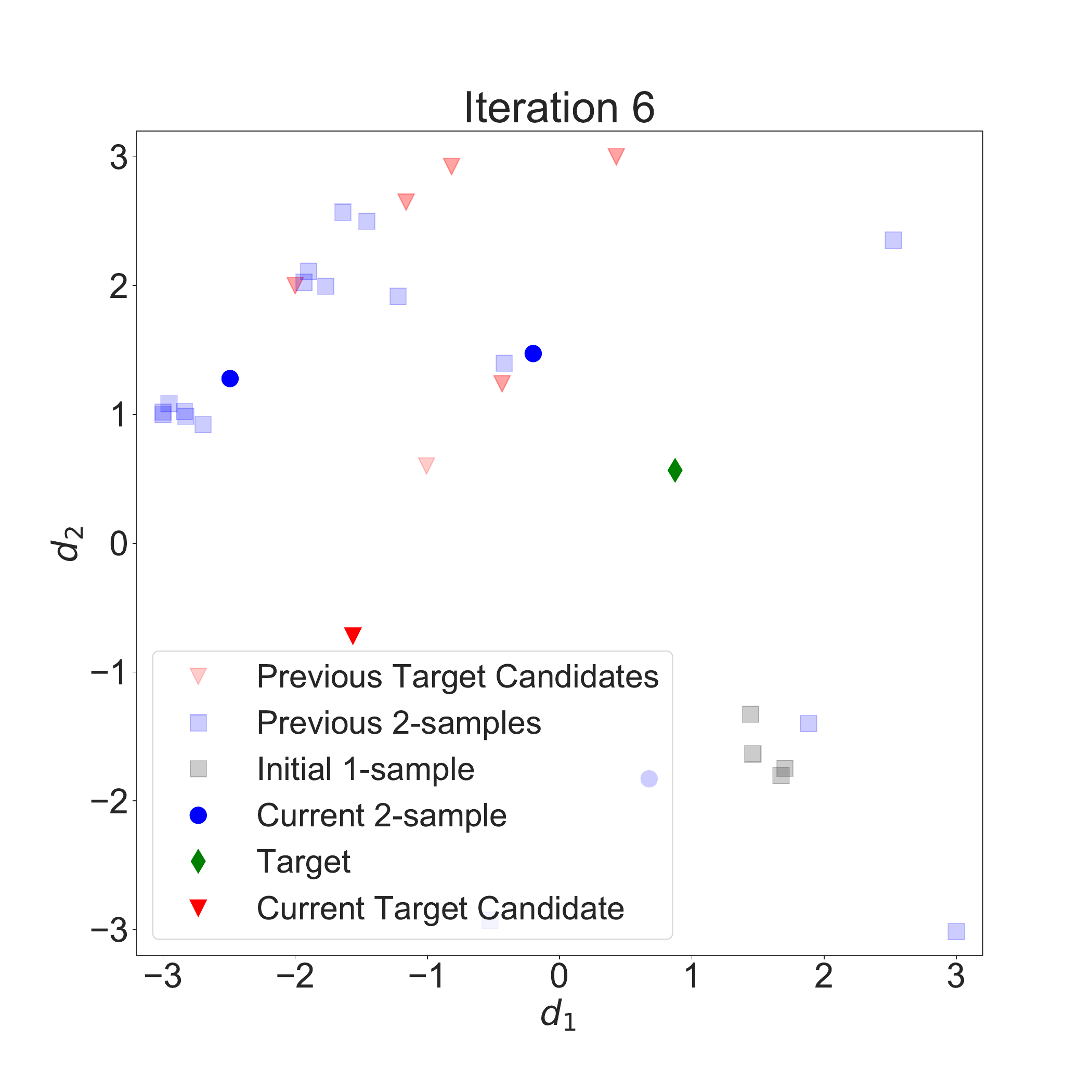}\\
\includegraphics[width=0.5\textwidth]{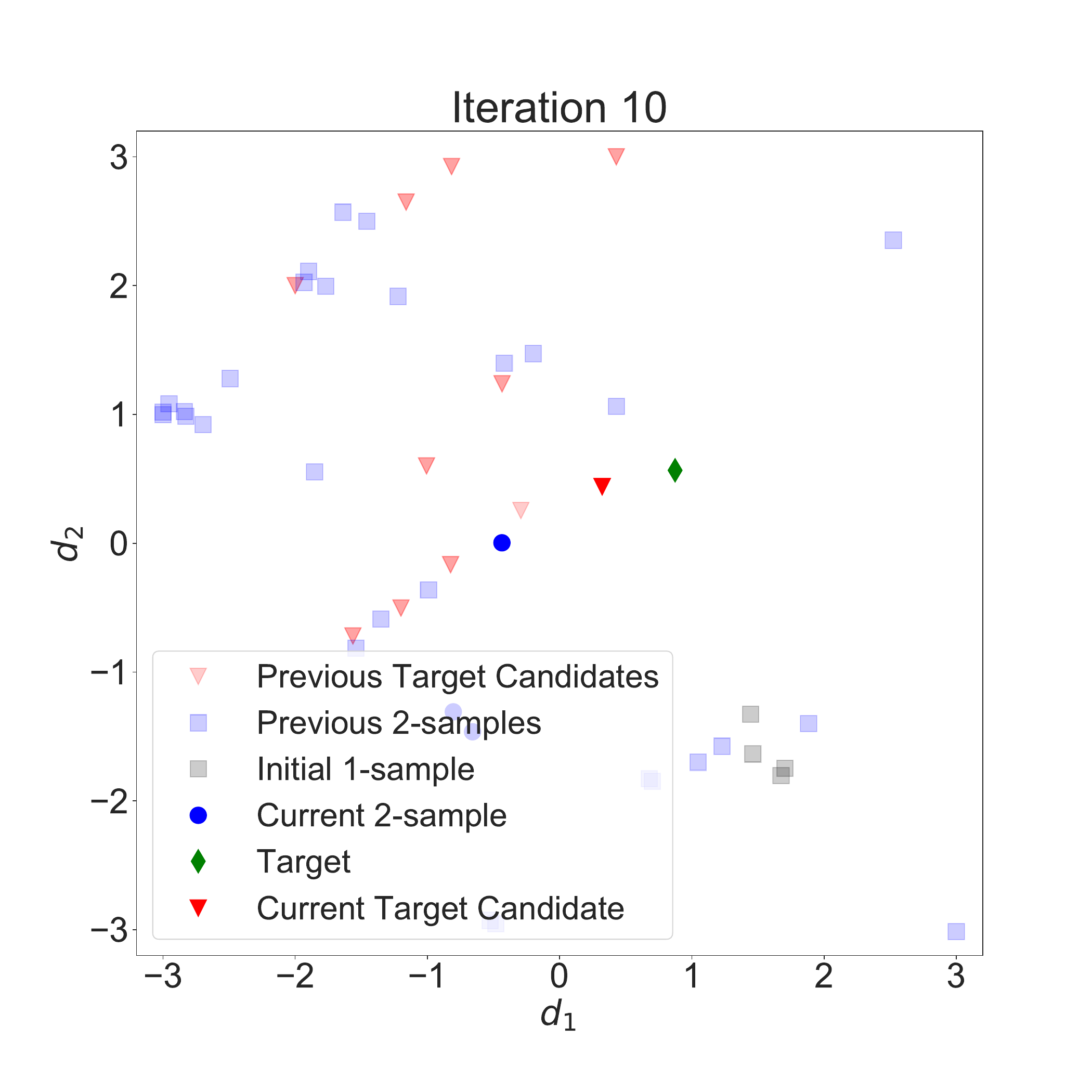}
\includegraphics[width=0.5\textwidth]{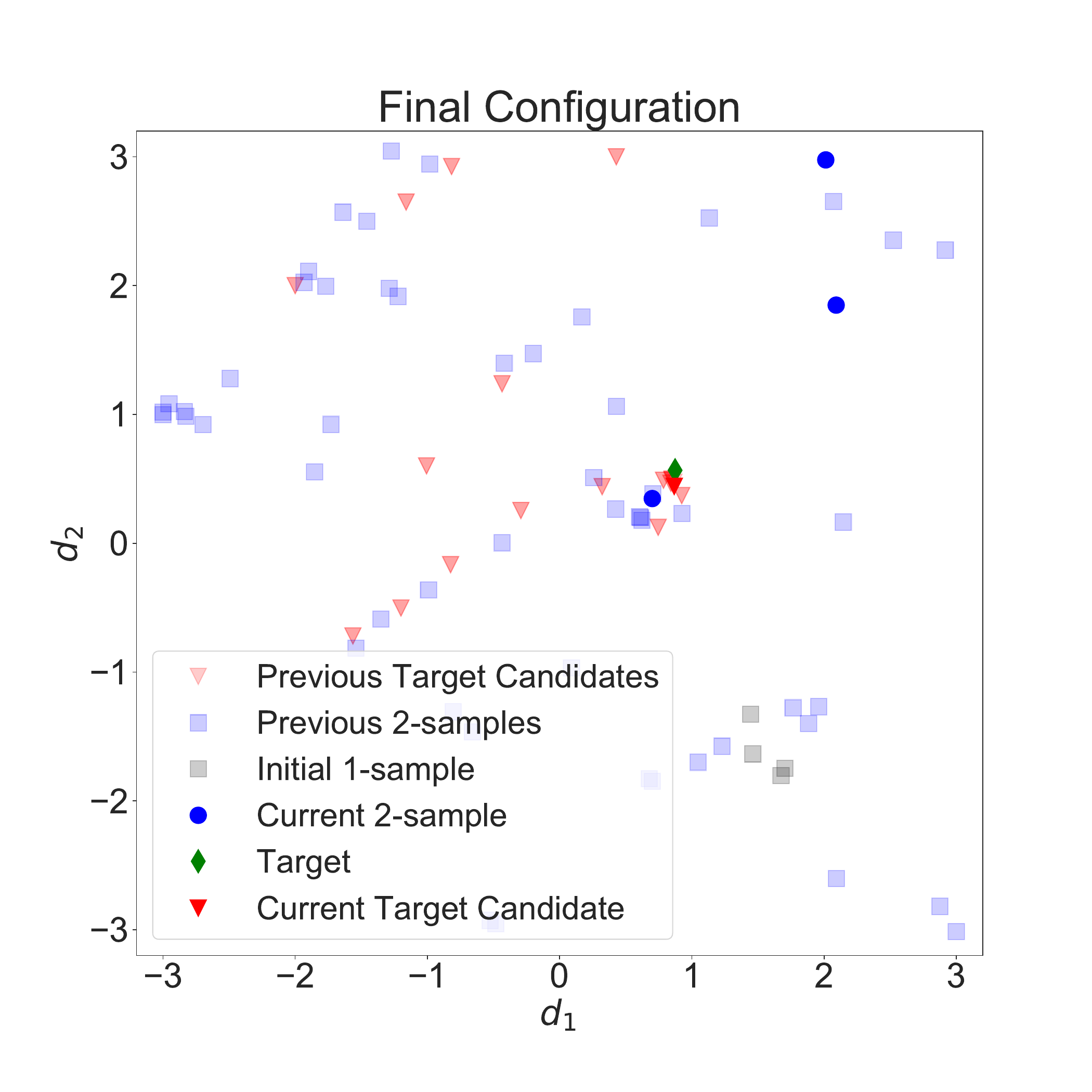}
\caption{Evolution of TAD ``2'' samples and solution.\label{fig:evolsol}}
\end{figure}

\section{Numerical Results}\label{sec:results}
\subsection{Twin-peak problem}\label{subsec:TP}
The TAD class has  no existing problems. For testing purposes we use a problem taken from \cite{OCW} with a $D=2$-dimensional control space and an $E=2$-dimensional design space.  The two-dimensional vector function $(v_1,v_2)$ has the form
\begin{multline}
v_1(d_1, d_2) = 3(1 - d_1)^2 e^{-d_1^2 - (d_2 +1)^2} - 10 (d_1/5 - d_1 ^3 - d_2^5) e^{-d_1^2 - d_2 ^2}\\
- 3 e^{- (d_1 + 2) ^2 - d_2^2} + 0.5(2d_1 + d_2);
\label{eq:vfield1}
\end{multline}
\begin{multline}
v_2(d_1, d_2) = 3(1 +d_2)^2 e^{-d_2^2 - (d_1 +1)^2} - 10 (-d_2/5 + d_2 ^3 + d_1^5) e^{-d_1^2 - d_2 ^2}\\
 - 3 e^{- ( 2- d_2) ^2 - d_1^2} + 0.5(2d_1 + d_2),
\label{eq:vfield2}
\end{multline}
where $(d_1, d_2) \in [-3, 3]^2$. The behavior of the function is displayed in Figure \cref{fig:tgtfun}.

\subsubsection{Convergence/Success}

We tested the convergence/success path of the algorithm by searching for the target $f_T=[0.3380,3502]$ with tolerance $\tau=0.01$. This value occurs in the domain at $[d_1,d_2]=[0.8731, 0.5664]$, shown by the black square in Figure \cref{fig:tgtfun}.  We adopted convergence/failure criteria of $\mathcal{I}_0=10^{-3}$, $N_\mathcal{I}=50$.  The initial design was a set of $N_1=4$ points in the vicinity of $[d_1,d_2]=[1.5,-1.5]$. We used $N_2=3$ and started the test at $x=[-2.0, 2.0]$.

Convergence occurred in 18 iterations, which included 4 restarts for model
update checks, requiring a total of $4+(18+4)\times 3=70$ sample points. The
initial and final sample point configurations and two intermediate snapshots are
displayed in \cref{fig:evolsol}. In each case, previously acquired samples are
displayed with partial transparency, while the current $\bx_2$ samples and
target value are displayed solidly, as is the true target solution. These
figures clearly show that the algorithm is engaging in some balance between
feeling out the response function and hunting for the target value.

\begin{figure}[htbp]
  \centering
  \includegraphics[width=1.0\textwidth]{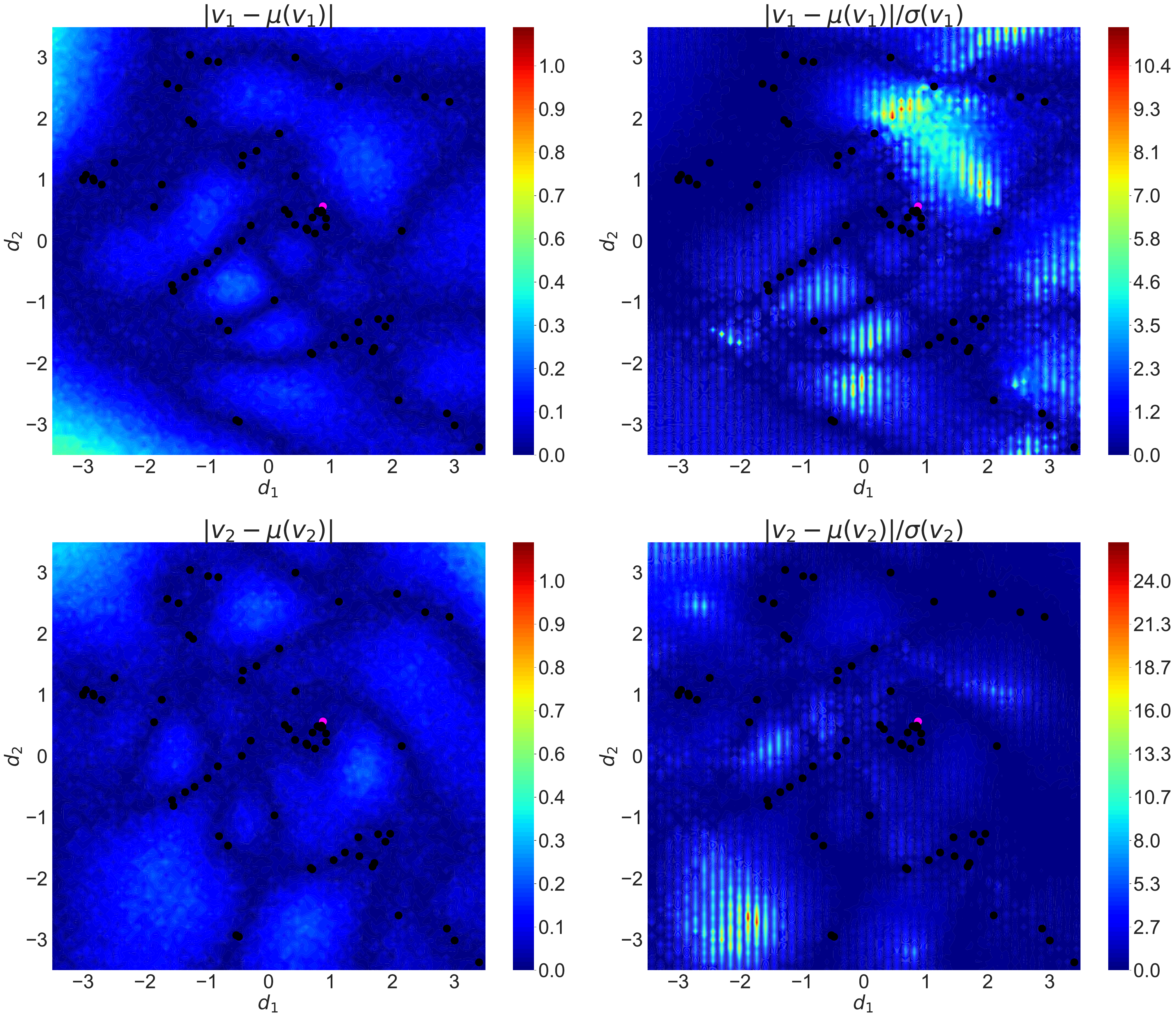}
  \caption{Pointwise residuals. Left column: Absolute residuals between predictive mean and target function. Right column: Residuals normalized to root predictive variance.\label{fig:predmean_2}}
\end{figure}

The pointwise residuals of the GP surrogate at the final iteration are displayed in \cref{fig:predmean_2}. The left panels show unnormalized residuals, whose magnitudes are not as straightforward to interpret as those of the residuals normalized to root predictive variance, shown in the right panels. These are expected to be of order unity in regions where the GP surrogate furnishes a reasonable model of the experimental response. From the figure, it is clear that this is, unsurprisingly, true only in the vicinity of the sample points.

\begin{figure}[htbp]
\centering
\includegraphics[width=0.7\textwidth]{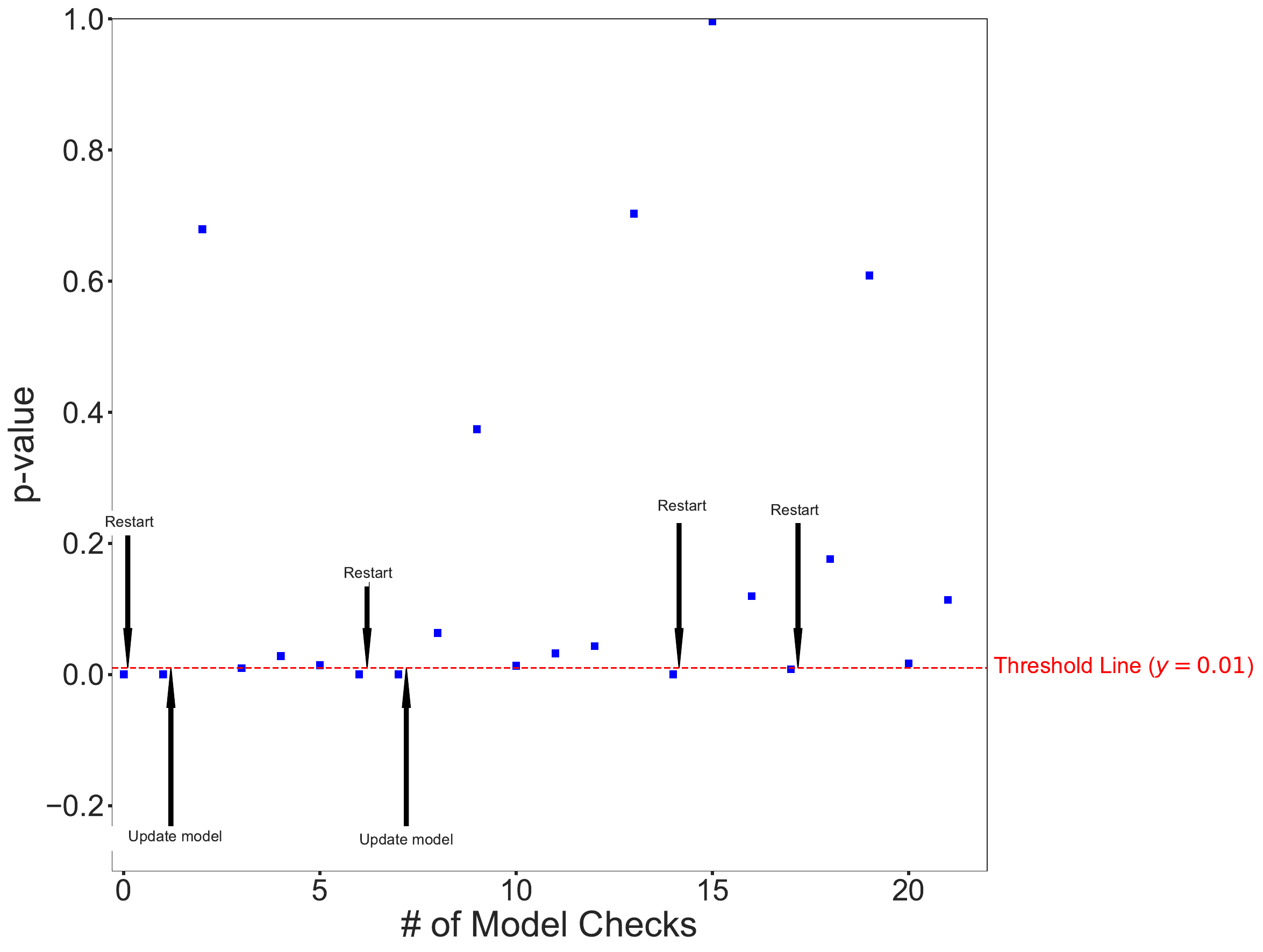}
\caption{Model validation, convergence/success case.\label{fig:qcheck}}
\end{figure}

The model validation and dynamic model complexification mechanism is visualized in \cref{fig:qcheck}. The figure shows the $P$-value threshold as a horizontal dashed line and the $P$-values ascertained at each iteration as square markers. The downward-pointing arrows show model validation failure alerts, where restarts occur. The subsequent upward-pointing arrows show confirmations of validation failures, where the GP model is updated by increasing the number of Kronecker components. The plot shows that after an initial couple of rounds of model validation failures, the surrogate model gained enough complexity to represent the portion of the response surface required to exploit the solution.

The run of expected information gain and of the log-Gaussian term is displayed in \cref{fig:expinfo_success}. The initial drop in expected information is associated with the refinement of the initially crude model. In this study the algorithm found an exploitation mode very early. As can be seen from \cref{fig:evolsol}, by iteration 7 the target point $x$ was already in the neighborhood of the solution; and in the remaining iterations the drop in log-Gaussian represents the effort to obtain convergence by squeezing the UB into the TTR. 

\begin{figure}[t]
\centering
\includegraphics[width=1.\textwidth]{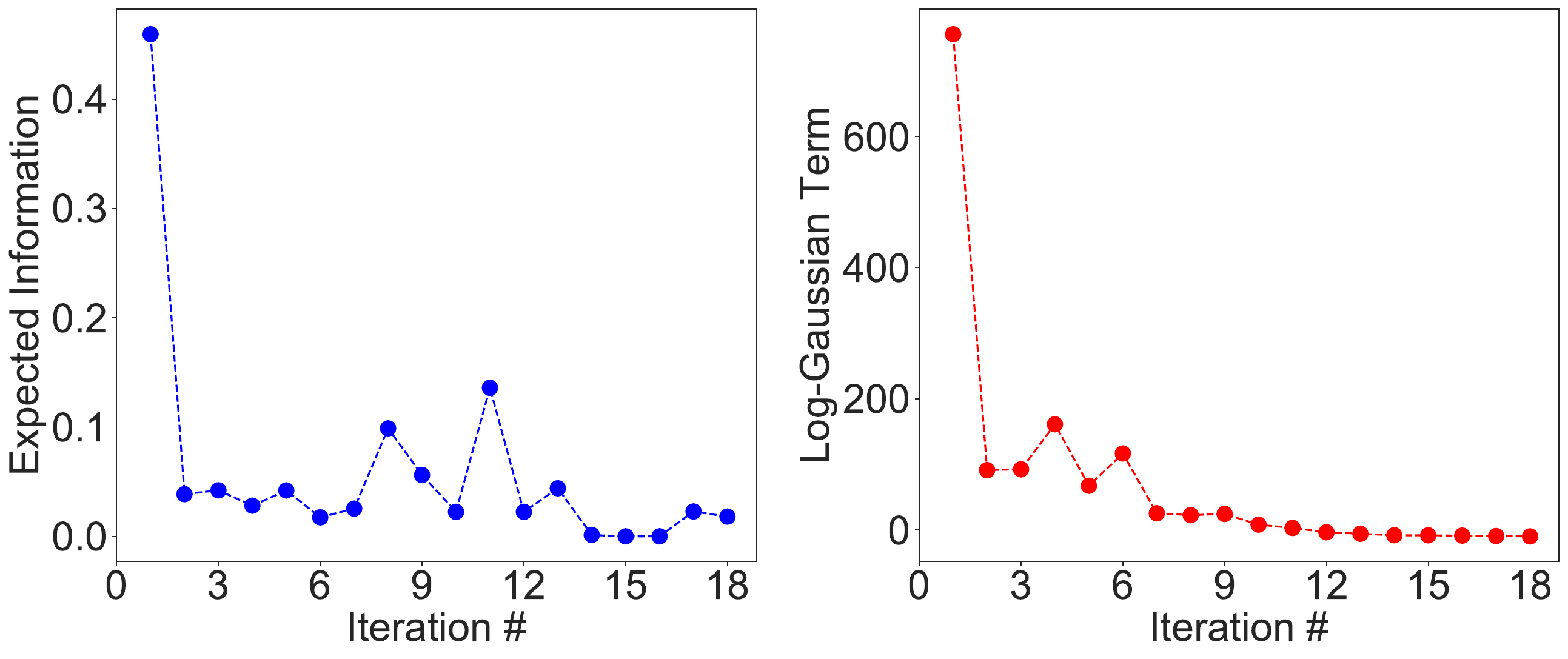}
\caption{Expected information gain and log-Gaussian term vs number of iterations, convergence/success case.\label{fig:expinfo_success}}
\end{figure}

\subsubsection{Convergence/Failure}

We tested the convergence/failure path of the algorithm by searching for the target $f_T=[-1,-1]$, which does not occur in the domain, with tolerance $\tau=0.01$.   We adopted convergence/failure criteria of $\mathcal{I}_0=10^{-3}$, $N_\mathcal{I}=50$. The initial design was the same set of $N_1=4$ points as for the convergence/success test.  We used $N_2=3$ and started the test at $x=[2.0, 2.0]$.

\begin{figure}[t] 
  \centering
 \includegraphics[width=0.7\textwidth]{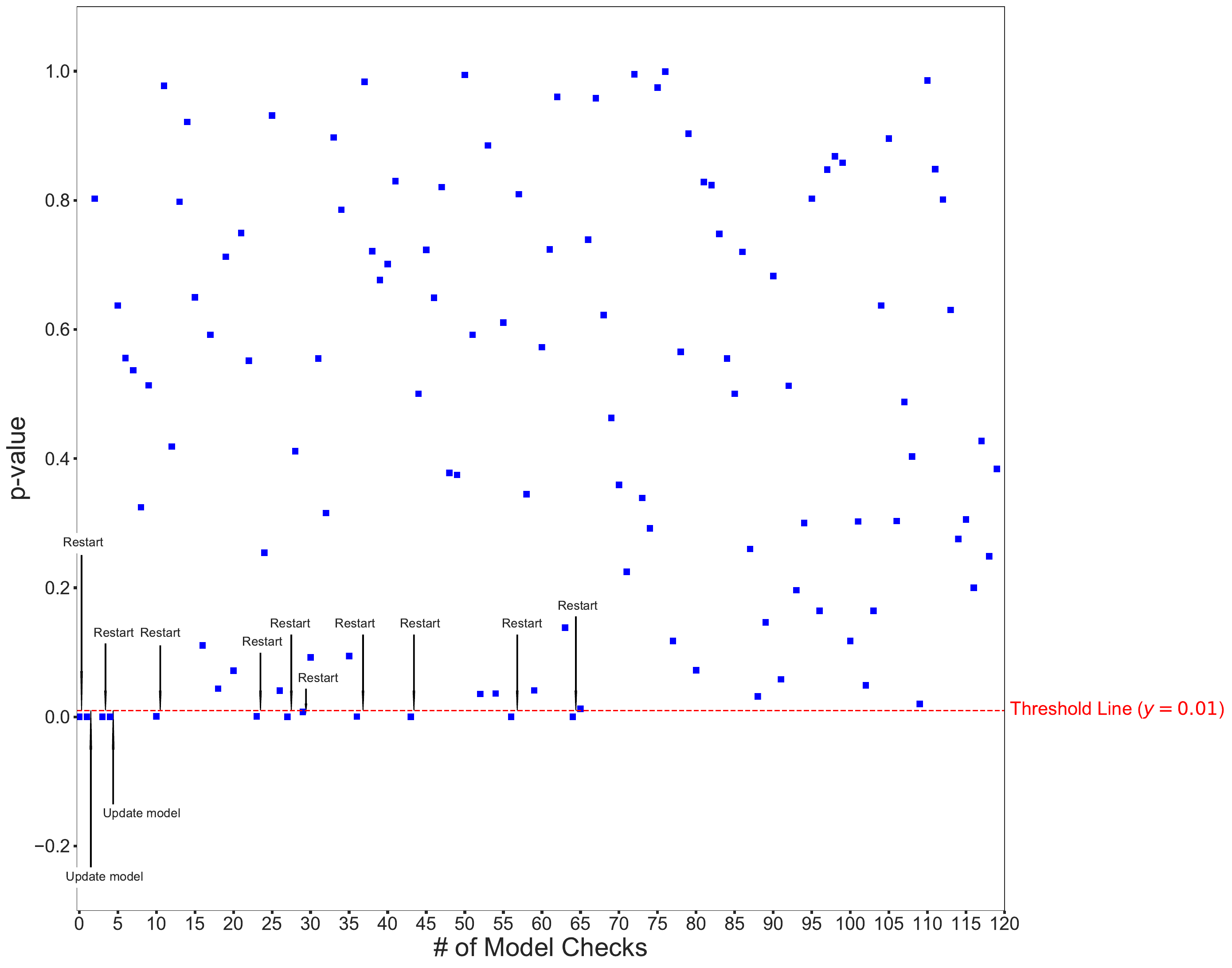}
  \caption{Model validation, convergence/failure case.\label{fig:qcheck_fail}}
\end{figure}

The algorithm declared convergence/failure after 111 iterations, requiring 442 sample points in all. \cref{fig:qcheck_fail} shows the model validation and dynamic model complexification mechanism. Once again we see an initial period of model inadequacy. There follows a cycle of exploration punctuated by occasional model validation alerts, all but two of which are false positives.

\begin{figure}[h]
  \centering
  \includegraphics[width=0.6\textwidth]{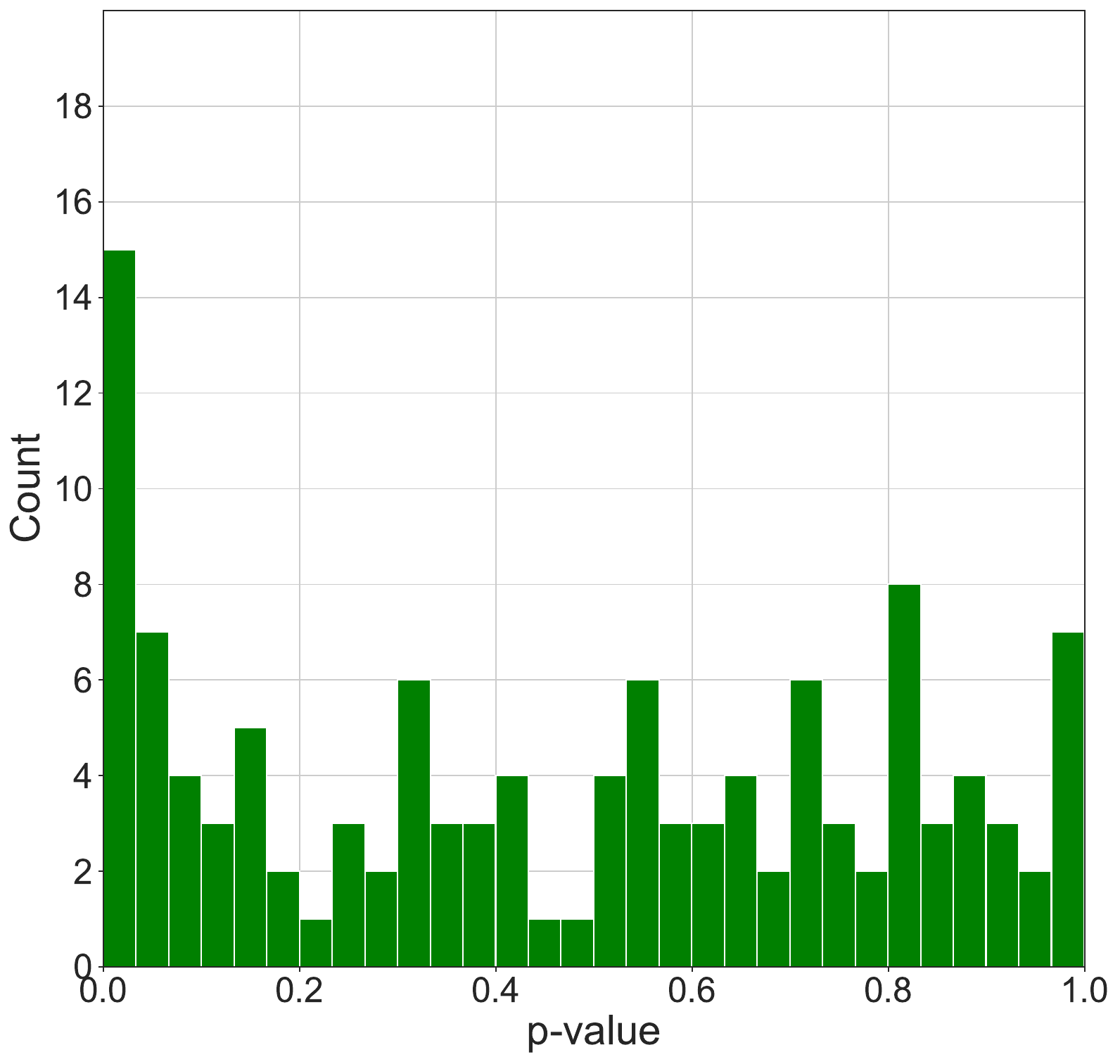}
  \caption{Histogram of $P$-values from \cref{fig:qcheck_fail}.\label{fig:pval_hist_failure}}
\end{figure}

The distribution of $P$-values appearing in \cref{fig:qcheck_fail} is displayed in \cref{fig:pval_hist_failure}. Other than the spike of small $P$-values caused by validation failures that trigger model complexifications, the histogram appears reasonably uniform, as it should.

\begin{figure}[htbp]
  \centering
 \includegraphics[width=1.\textwidth]{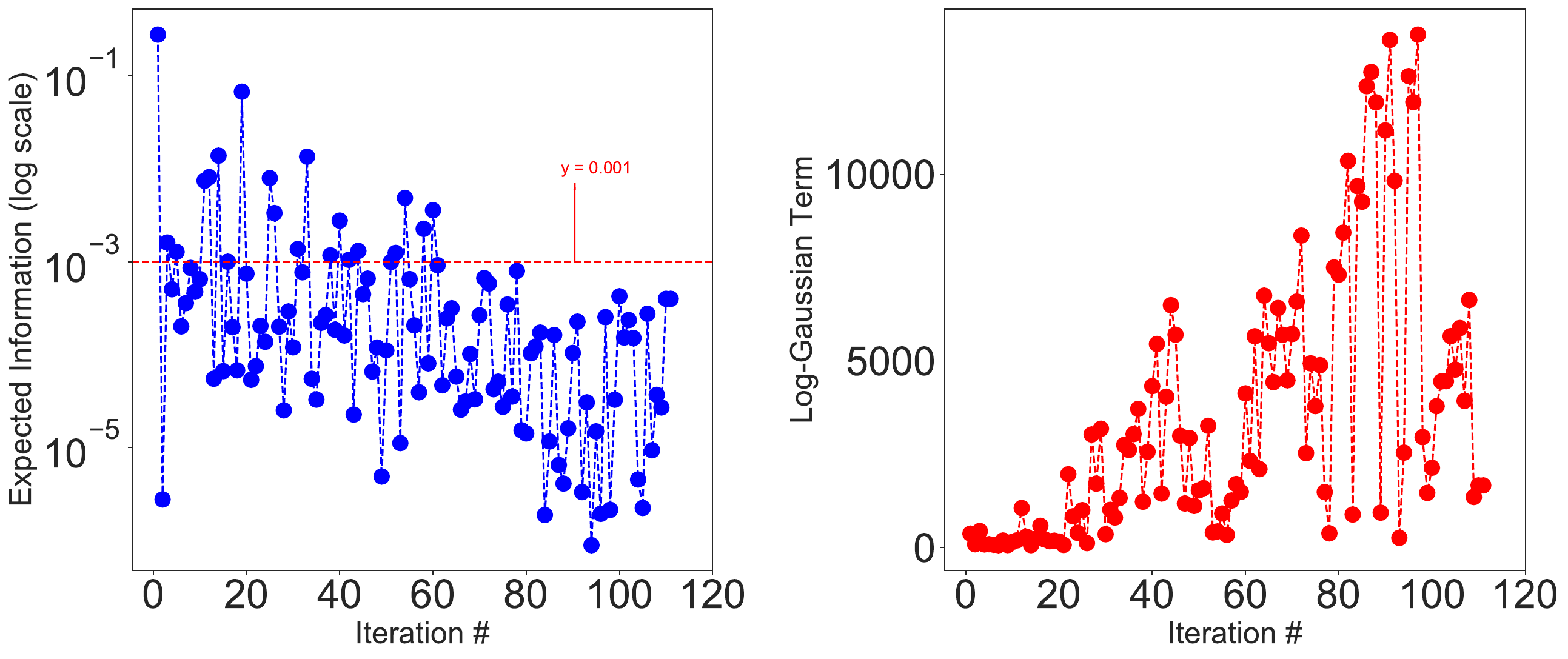}
  \caption{Expected information gain and log-Gaussian term, convergence/failure case.}\label{fig:expinfo_fail}
\end{figure}

\cref{fig:expinfo_fail} shows the run of expected information gain and of the log-Gaussian term.  Here we see a logarithmic decline in expected information gain as the uncertainty is squeezed out of the surrogate response. Monitoring expected information gain for this kind of drop results in an effective convergence/failure criterion.

\begin{figure}[htbp]
  \centering
  \includegraphics[width=1.0\textwidth]{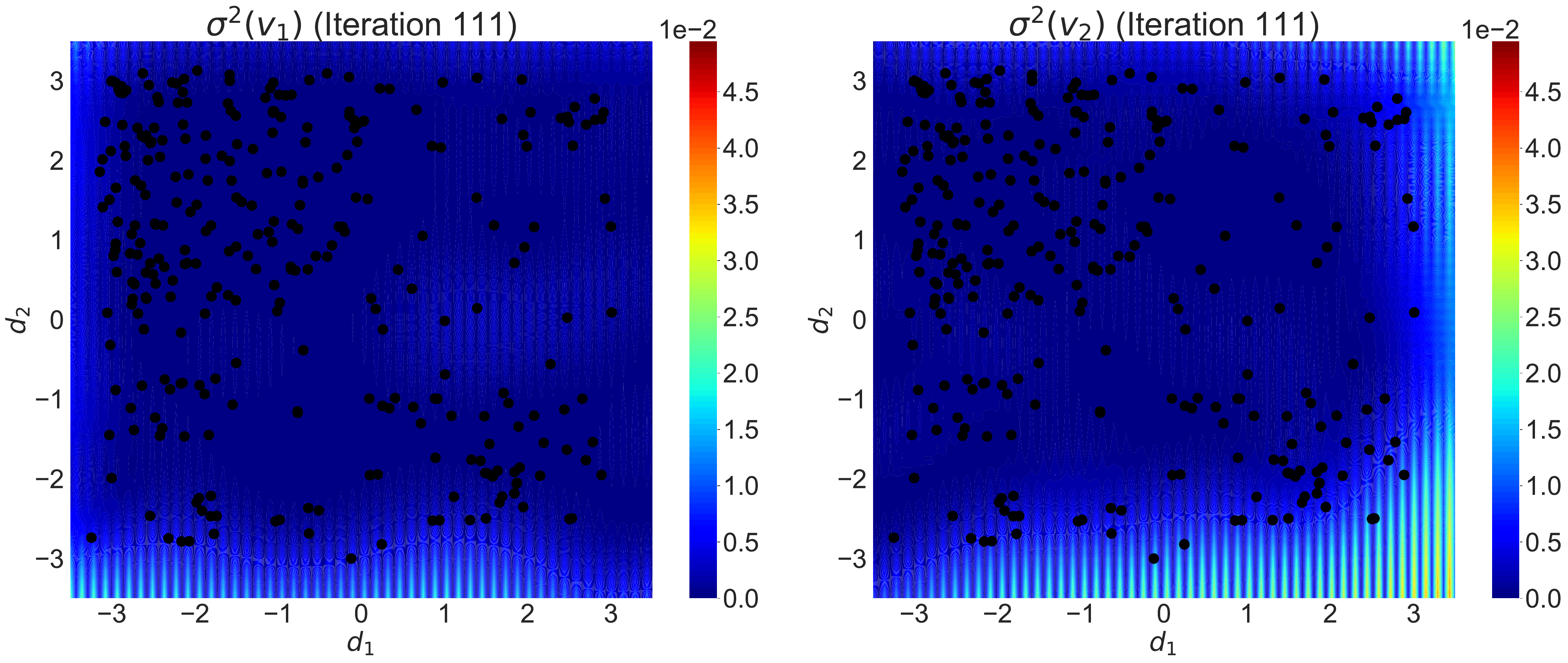}
  \caption{Variance at convergence/failure.\label{fig:var_failure}}
  
\end{figure}

\cref{fig:var_failure} displays the surrogate model variance at the final iteration when convergence/failure is declared. At this stage the algorithm has sampled the space fairly uniformly and has generated a relatively smooth model of the variance for both components of the vector field.

\begin{figure}[t]
\centering
\includegraphics[width=1.0\textwidth]{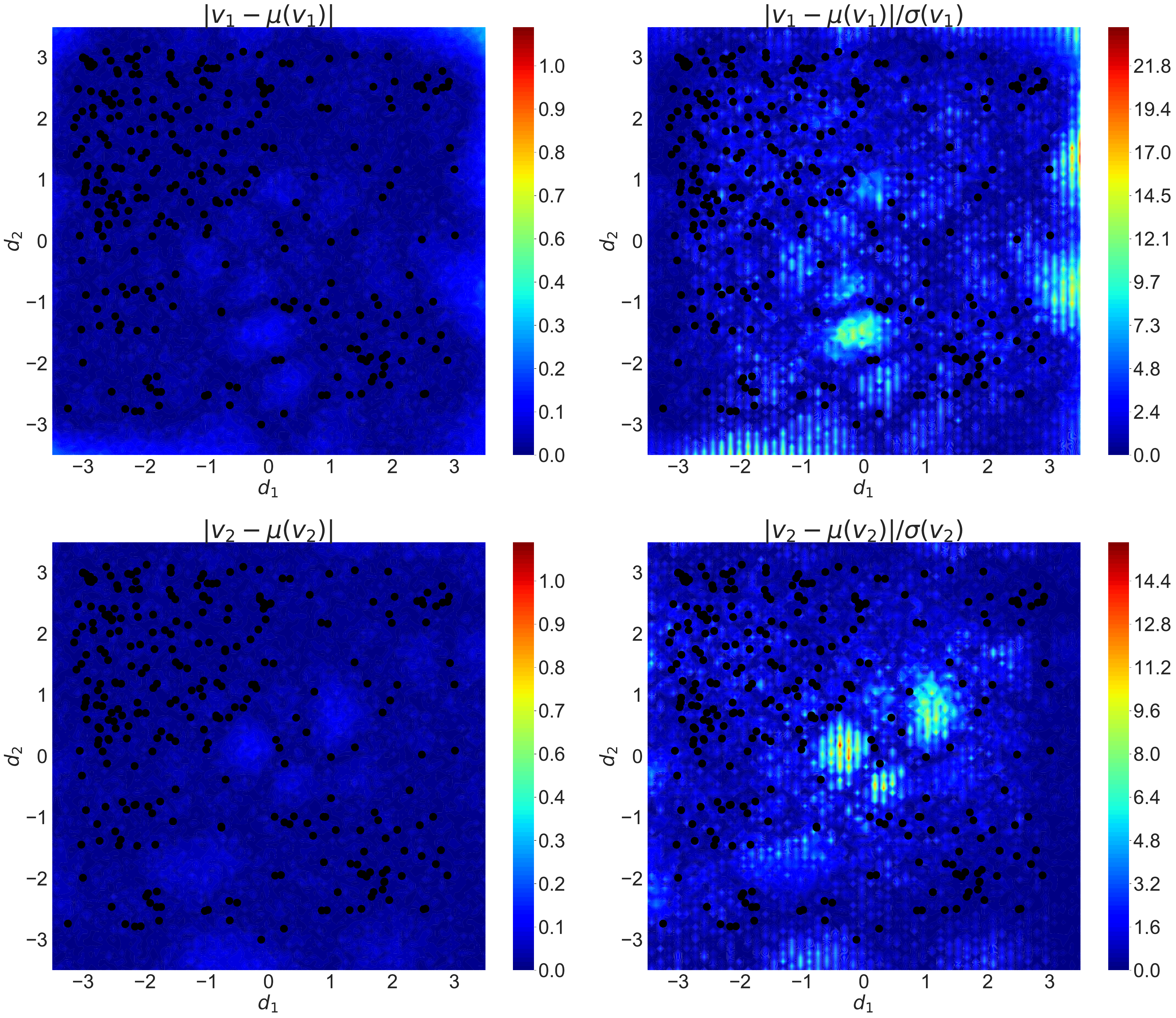}
\caption{Pointwise residuals. Left column: Absolute residuals between predictive mean and target function. Right column: Residuals normalized to root predictive variance.\label{fig:predmean_failure}}
\end{figure}

\cref{fig:predmean_failure} displays the residuals between the surrogate model and the target function. The right-hand panels display the normalized residual, expected to be of order unity, with some statistical fluctuations, when the model is valid. We see that this is largely the case, with some residual model inadequacy near the domain edge and in some patches of less-dense sampling.

\subsubsection{Impact of Tolerance Selection}

\begin{figure}[t]
\centering
\includegraphics[width=1.0\textwidth]{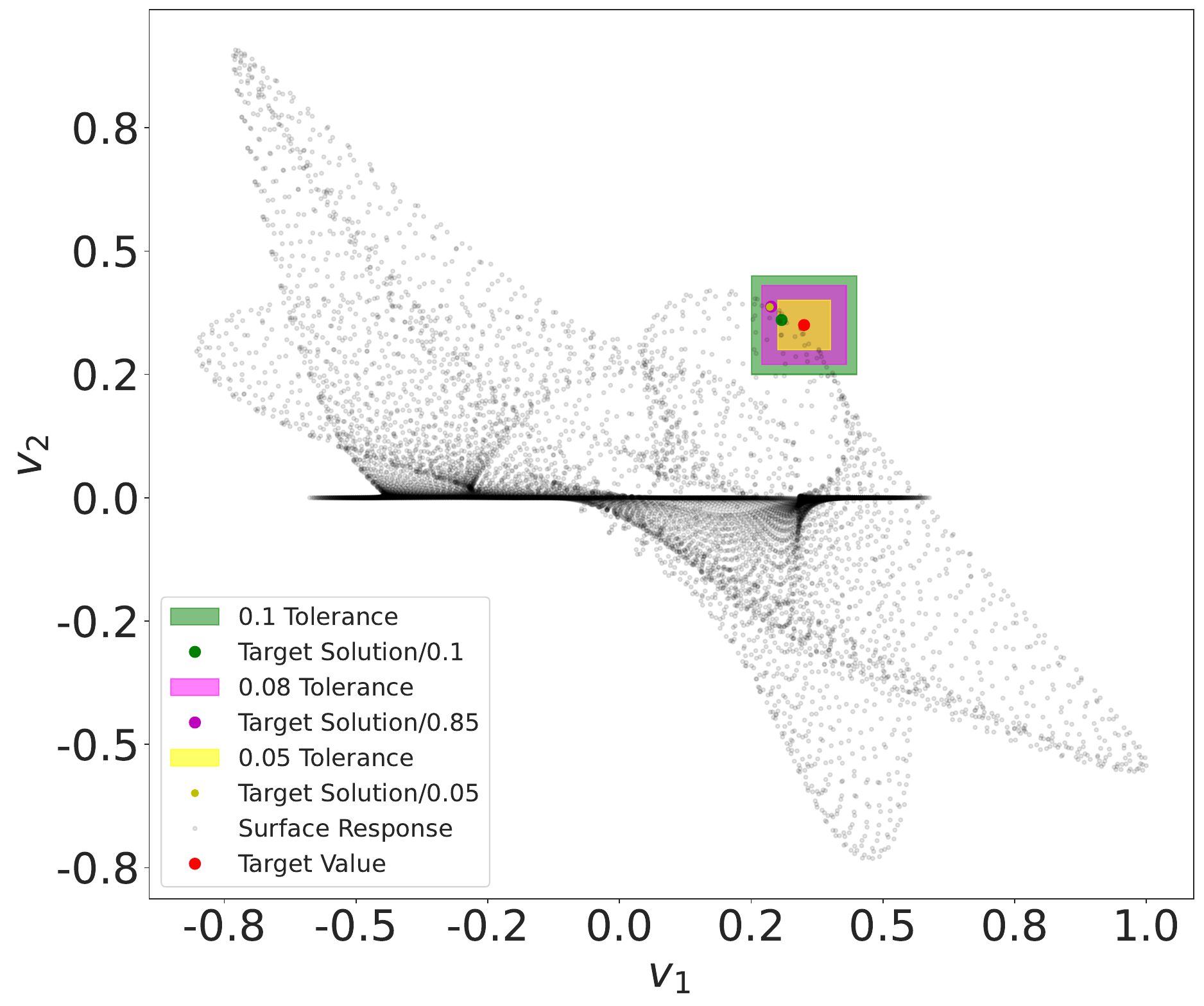}
\caption{Study of the effect of tolerance selection on the TAD solution.  The gray dots represent the values of the output pairs $(v_1,v_2)$ sampled from a $100\times 100$ uniform grid of points in the input space $(u_1,u_2)$. The red dot is the target value, which is slightly offset from the range of feasible values of the function.  The colored boxes correspond to different TTRs, whereas the correspondingly-colored dots represent the corresponding TAD solutions.  The innermost (yellow) TTR triggered a Convergence/Fail termination
after 200 iterations. \label{fig:tolerance_selection}}
\end{figure}

Our goal here is to illustrate the interplay between the convergence rules, the
tolerance and the uncertainty around TAD output.
\Cref{fig:tolerance_selection} shows the change in TAD solutions as the tolerance
threshold changes.  The figure represents as gray points all the values of the
output pairs $(v_1,v_2)$ sampled from a $100\times 100$ uniform grid of points
in the input space $(u_1,u_2)$. The target value (red circle) is not in the
range of the surface function (represented by the gray dots) but is slightly
offset from that range, so that a solution only exists if the size of the TTR
exceeds some minimum value. 

In each case, the algorithm starts at $x_0 = [-2,2]$. The tolerance threshold
takes values $0.1$ (green TTR box), $0.08$ (magenta box) and $0.05$ (yellow
box). The $0.05$ tolerance is too small to permit a solution. 

The colored dots of represent the solution arrived at by TAD for the correspondingly-colored TTR. As expected, TAD converges successfully within the prescribed tolerances for $0.1$ and $0.08$ and terminates with ``Convergence/Fail'' for the $0.05$ threshold. With the $0.05$ threshold, the final target point estimated by the algorithm before exiting coincides with the target point
found for the $0.08$ threshold.

\subsection{DTLZ4 problem}

The DTLZ problem test suite comprises benchmark problems for multiobjective optimization. These problems were developed by Deb, Thiele, Laumanns and Zitzler \cite{Deb2005} and are particularly for scalable multiobjective optimization as the number of variables, as well as the number of objective functions, can be tuned. In this work, we use the dtlz4 problem 
\begin{eqnarray*}
 f_1(x) &=& (1 + g(x_M)) \cos(x_1^{\alpha}\pi/2)\cdots\cos(x_{M-2}^{\alpha}\pi/2)\cos(x_{M-1}^{\alpha}\pi/2) \\
  f_2(x) &=& (1 + g(x_M)) \cos(x_1^{\alpha}\pi/2)\cdots\cos(x_{M-2}^{\alpha}\pi/2)\sin(x_{M-1}^{\alpha}\pi/2)\\
  \vdots \\
 f_M(x) &=& (1 + g(x_M)) \sin(x_1^{\alpha}\pi/2) \\
  \text{with } g(x_M) &=& \sum_{x_i \in x_M} (x_i - 0.5)^2,\; \alpha = 100 \;\text{and} \;\;0\leq x_i\leq 1\;\; \text{for } i = 1,2,\cdots, n.
\end{eqnarray*}
The last $k = n - M + 1$ variables are represented in $x_M$. See 
\cite{pymoo} for more details.

\subsubsection{Convergence/Success}

We set the number of control variables to $n = 4$ control variables and the number of tasks to $M = 3$. The target value is $f_T = [1.0607, 0.0109, 0.0040]$. We set the noise value at $0.0001$ and the tolerance at $1\%$. TAD found a design value of $f_{TAD} = [1.0474e, 0.025, 0.00025]$ after 6 iterations with 4 model updates i.e. 5 total Kronecker products needed. The behavior of the expected information gain and the log-Gaussian term is displayed in \cref{fig:expinfo_success_dtlz4}.

\begin{figure}[htbp]
  \centering
 \includegraphics[width=1.\textwidth]{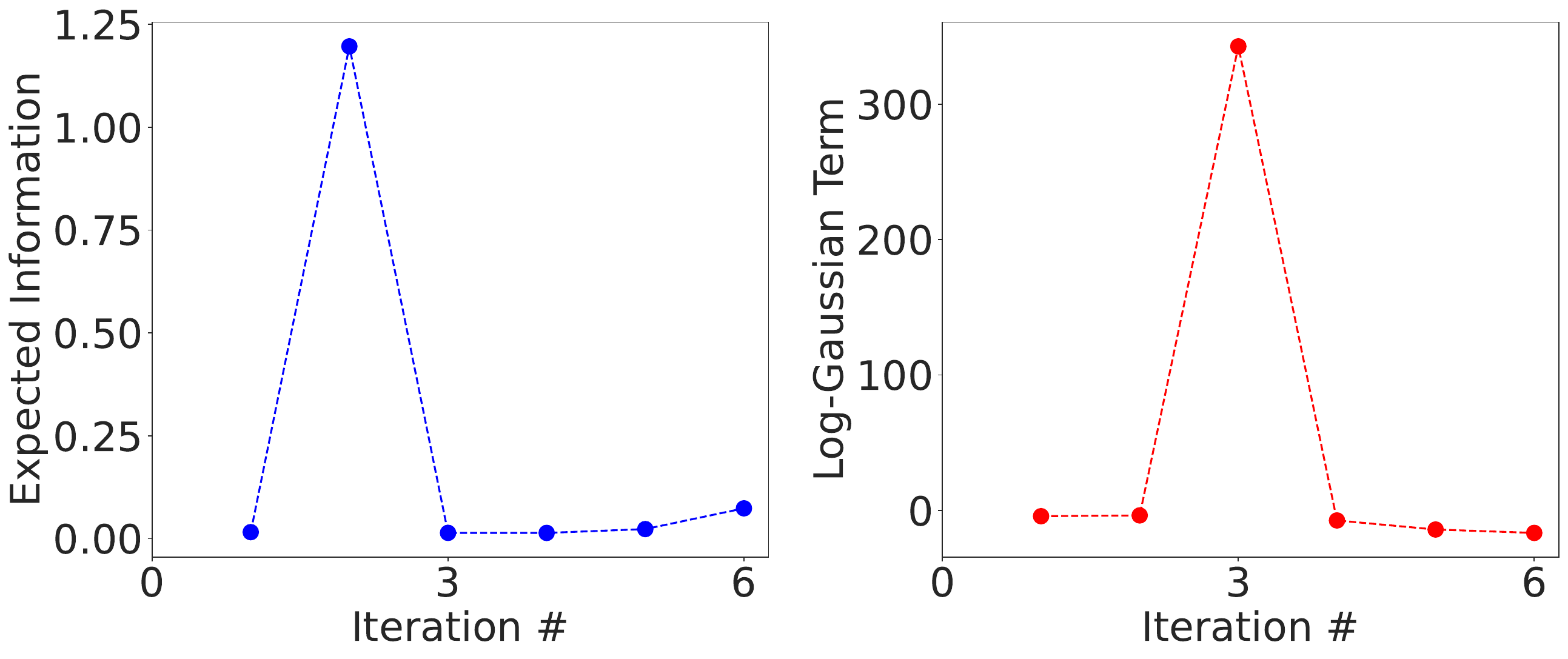}
  \caption{Expected info and log gaussian term for the dtlz4 problem, convergence/success.}\label{fig:expinfo_success_dtlz4}
\end{figure}

\subsubsection{Convergence/Failure}
We set the number of control variables to $n = 4$ control variables and the number of tasks to $M = 3$. The target value is $f_T = [2., 2., 2.]$ which is not feasible. As for the convergence/success case, the noise value is $0.0001$ and the tolerance at $1\%$. \cref{fig:expinfo_fail_dtlz4} displays behavior of the expected information gain and the log-Gaussian term. TAD needed three Kronecker products and failed in $57$ iterations as.
\begin{figure}[htbp]
  \centering
 \includegraphics[width=1.\textwidth]{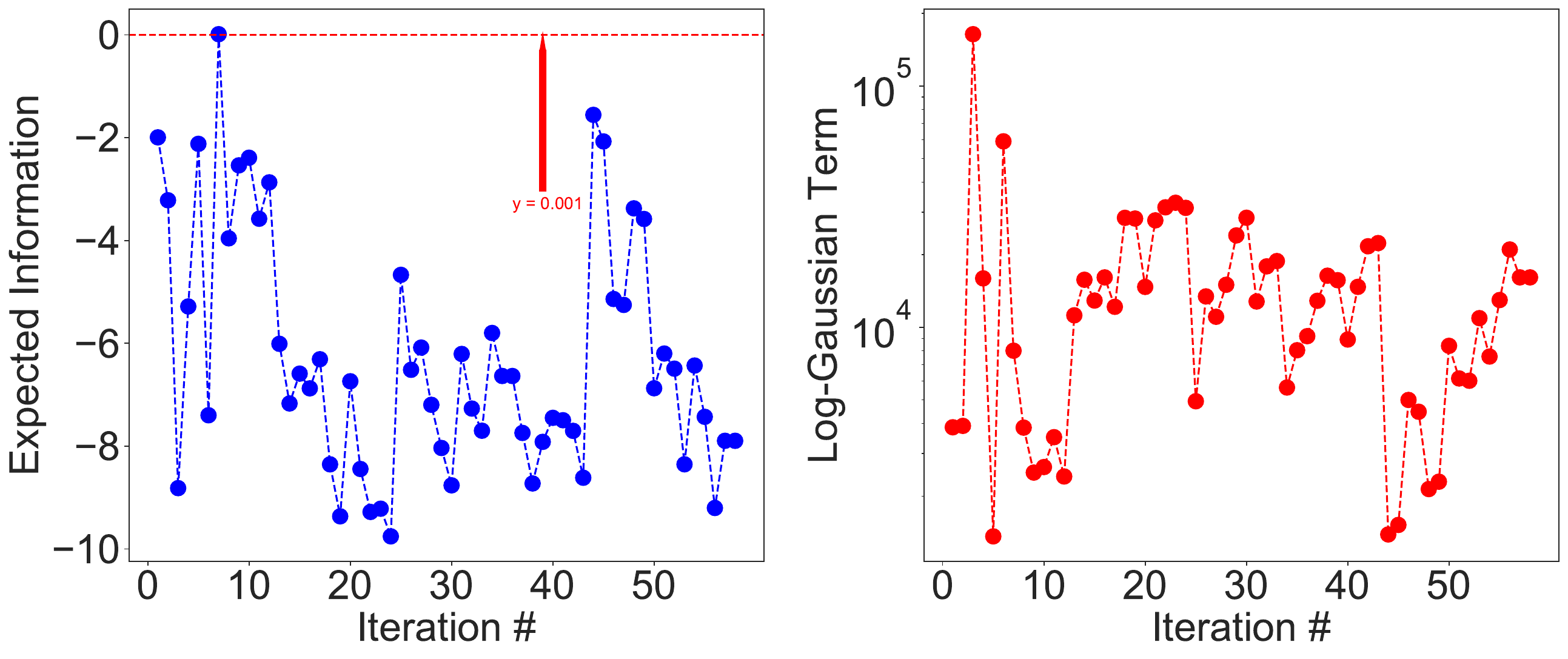}
  \caption{Expected information gain and log-Gaussian term for the dtlz4 problem in log scale, convergence/failure case.}\label{fig:expinfo_fail_dtlz4}
\end{figure}

\subsection{Comparison With an $L_2$-Based Method}

As discussed in $\S$\ref{sec:Intro}, the specific problem addressed by the TAD
method, multi-target black-box optimization with solution uncertainties within
target tolerances, does not appear to have been previously addressed in the
literature, so that it is not possible to perform precise comparisons of the
algorithm to alternative approaches.  One of the closest methods that can
be adapted and pressed into service is described in \cite{zhang2019sequential},
wherein parametrized dynamical simulations of time-series are calibrated to
reproduce an observed series by considering an acquisition function based on the
$L_2$ discrepancy between model and target time-series.  If we analogize
time-series elements to design features, the $L_2$ discrepancy in \cite{zhang2019sequential} becomes
\begin{equation}
\delta(\bx)=\sum_{l=1}^E(f_l(\bx)-f_{T,l})^2=(f(\bx)-f_T)^\top(f(\bx)-f_T),
\label{eq:L2_disc}
\end{equation}
where $f(\bx)$ is the multi-output GP surrogate.

In \cite{zhang2019sequential}, the Expected Improvement (EI) of $\delta(\bx)$ is
adopted as a sampling acquisition function, although, as the authors note, there
is no closed-form expression for this function, so that it must be approximated.
Furthermore, the characteristic shape of the EI function tends to be multimodal
with large ``flat'' regions of minimal change, and it is not continuously
differentiable \cite{jones1998efficient}, so that it is not ideal for gradient
optimization.  On the other hand, \cite{zhang2019sequential} also introduce the
Expected Squared $L_2$ Discrepancy (ESL2D) for the purpose of optimizing the
simulator parameters once a new sample has been acquired.  In present context,
the ESL2D $\mathcal{E}(\bx)$ is given by
\begin{align}
\mathcal{E}(\bx)&=E_{f(\bx)|1}\left[(f(\bx)-f_T)^\top(f(\bx)-f_T)\right]\nonumber\\
&=\left(p^{(f|1)}-f_T\right)^\top\left(p^{(f|1)}-f_T\right)
+\mathrm{Trace}\left(Q^{(f|1)}\right),
\label{eq:ESL2D}
\end{align}
which corresponds to the expression given for the ESL2D in \cite{zhang2019sequential}.

It follows that a simplified version of sequential method of
\cite{zhang2019sequential} can be readily adapted to our context if we combine
the sampling and optimization stages of that method, and use the ESL2D
$\mathcal{E}(\bx)$ as the sample acquisition function, instead of the EI.  The
result is a method in which the aquisition function is simple to compute in
closed form, has no optimization-defying regions of discontinuous gradient, and
simply seeks out sample points that minimize the expected value of the $L_2$
discrepancy $\delta(\bx)$.  In this method, there are no $\bx_2$ points, because they would have no effect, as it is straightforward to demonstrate that $E_{\bg_2|\bg_1} E_{f(\bx)|\bg_1\bg_2} [\delta(\bx)] = E_{f(\bx)|\bg_1}[\delta(\bx)]$. The method is therefore purely sequential, and has no batch aspect. We
will refer to this method as the ``sequential $L_2$ method.''

The sequential $L_2$ method differs from TAD in that while its aquisition
function, Equation (\ref{eq:ESL2D}), is sensitive to the sum of the principal
variances (the trace term in Equation \ref{eq:ESL2D}), it does not use them to
discriminate between directions in design space that reduce the distance in
probability to the target, in the manner of the TAD acquisition function,
Equation (\ref{eq:TAD_of}).  The value of a comparison between TAD and the
sequential $L_2$ method, then, beyond simply providing a baseline comparison
with an alternative method, is to illustrate the benefit of incorporating such
directional information.

\begin{figure}[t]
  \centering
 \includegraphics[width=0.45\textwidth]{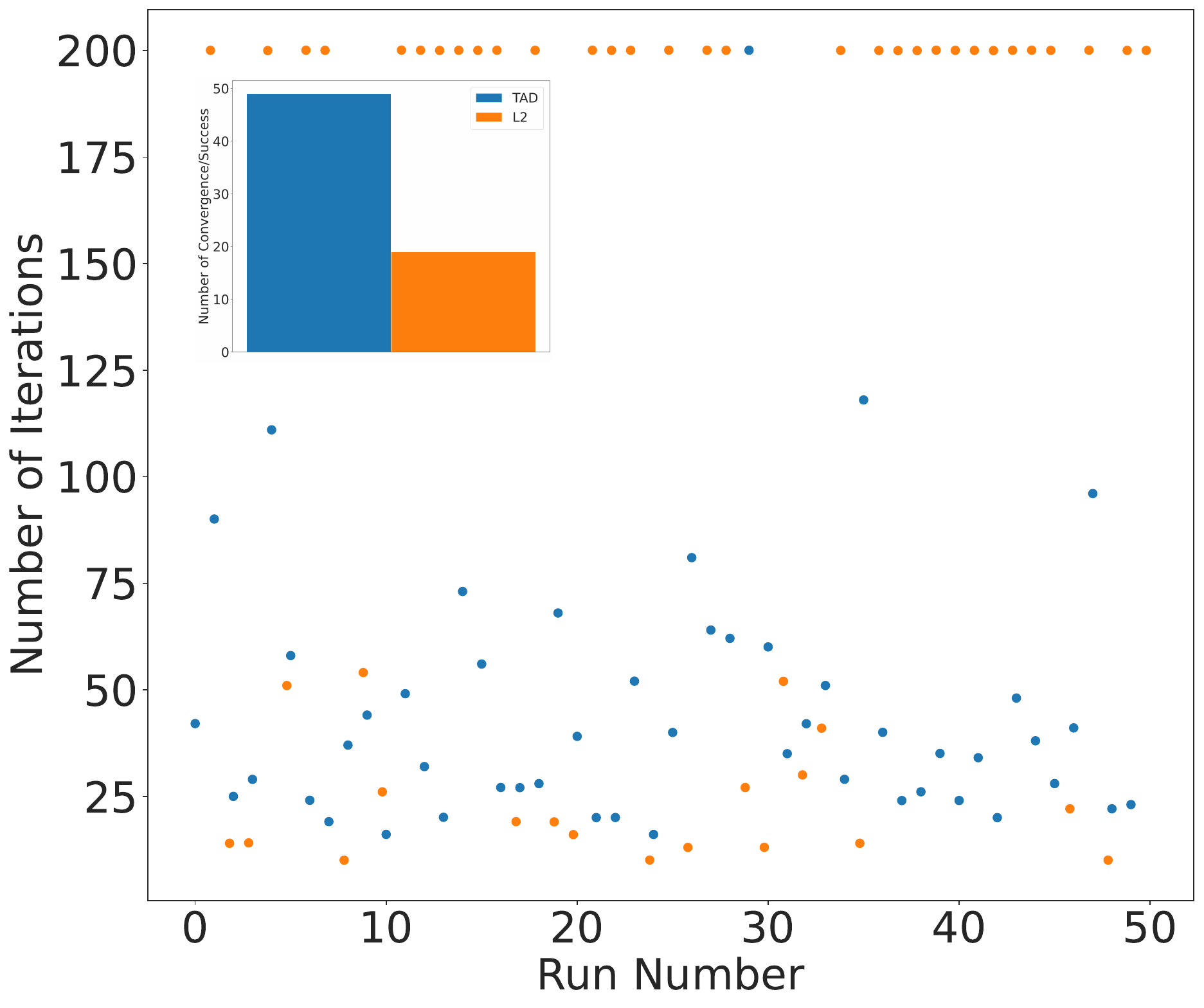}
\includegraphics[width=0.45\textwidth]{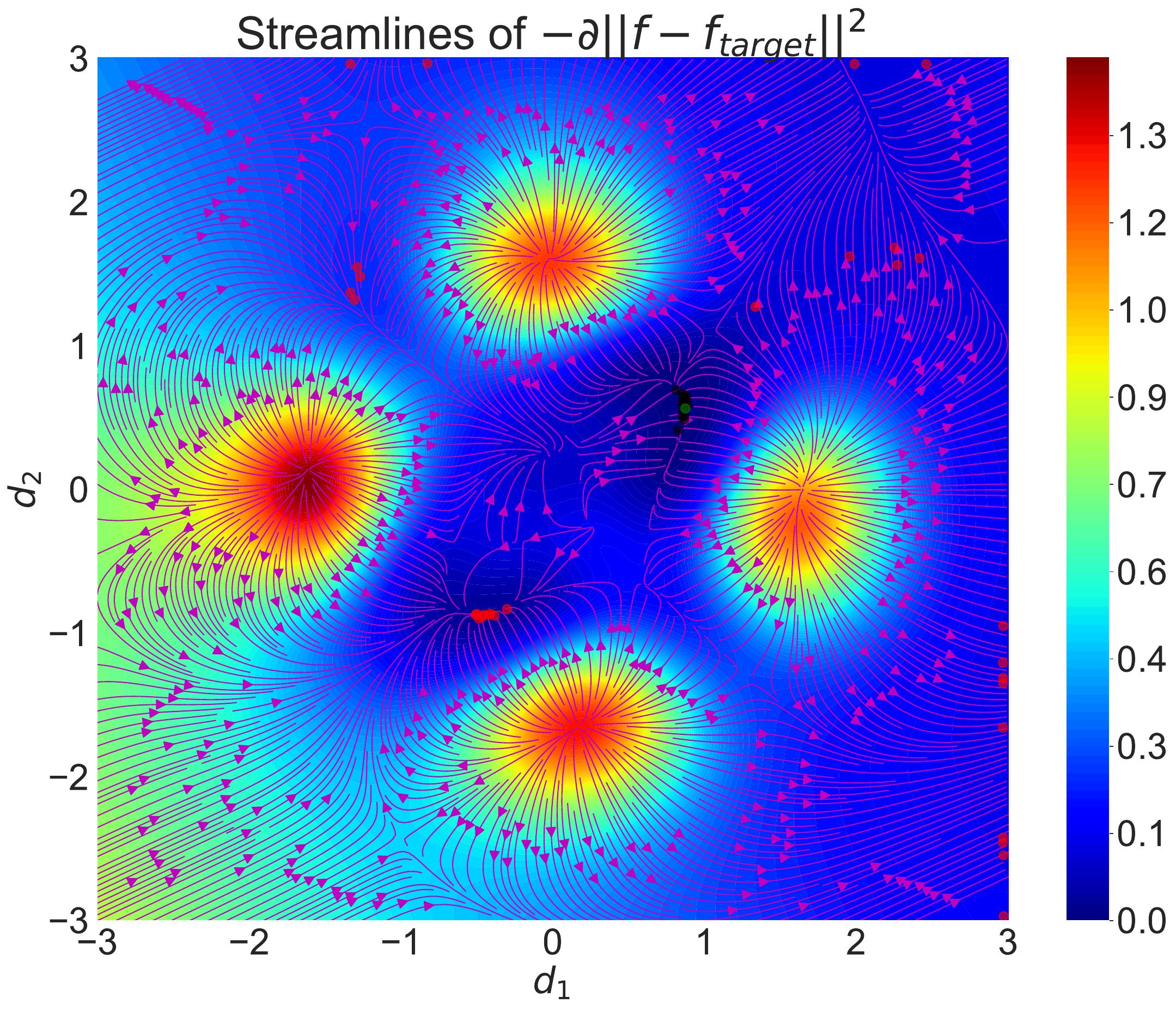}
  \caption{Left panel: Comparison between TAD and sequential $L_2$ method on the
Twin Peak problem of $\S$\ref{subsec:TP}. The inset displays the relative number
of convergence/success cases for the two methods. Right panel: Final states of
sequential $L_2$ searches for the Twin Peaks problem, superposed on colormap of
$l(\bx)\equiv||f(\bx)-f_T||_2^2$, and on streamlines of $\nabla l(\bx)$. The green
dot is the target. The black dots that neighbor the target are the solutions
found in convergence/success cases. The red dots are the convergence/failure
final states. }\label{fig:TADvsL2}
\end{figure}

The left panel of \cref{fig:TADvsL2} illustrates the different performance of the two algorithms.
We run each algorithm $50$ times on the Twin Peaks problem of
$\S$\ref{subsec:TP}, with initialization location chosen randomly for each run,
to assess their convergence/success capabilities. The tolerance is $0.01$ and
the target value $f_T = [0.3380, 0.3502]$ (as in $\S$\ref{subsec:TP}) for all
the tests. We stopped the algorithms and declared failure whenever the number of
iterations reached $200$. Note that TAD can potentially declare
convergence/failure before reaching $200$ iterations. We added the same penalty
term at the boundaries to the $L_2$ acquisition function $\mathcal{E}$ as we
did to the TAD acquisition function. 

As shown by the inset, TAD reached "convergence/success" $49$ times out of the
$50$--the remaining case was halted at 200 iterations, rather than according to
TAD's native convergence/failure criterion. The sequential $L_2$ method
converged $22/50$ times. Thus TAD consistently found a solution within the
prescribed tolerance, whereas the sequential $L_2$ method struggled to converge
within the set number of iterations. TAD required, on average, 43 iterations to
converge, corresponding to about 130-150 function evaluations (three per
iteration, including restarts). It is worth noting that, in the cases where they
reached  "convergence/success", the sequential $L^2$ method did so on average faster than TAD---24 iterations on average, as opposed to 43.

The nature of the convergence failures of the sequential $L_2$ method are
interesting. The right panel of \cref{fig:TADvsL2} displays the final states of
all the runs of the method, superposed on a colormap of $l(\bx)\equiv\left\Vert
f(\bx_{opt})-f_T\right\Vert_2^2$ and on streamlines of $-\nabla l(\bx)$. It
appears that the method located a few local minima of $\mathcal{E}(\bx)$,
represented in the figure by the clusters near stationary points indicated by
the behavior of the streamlines, and by the concentrations at the top and left
boundary, where the boundary penalty term created a spurious minimum of
$\mathcal{E}(\bx)$. Such points $\bx_{opt}$ are not solutions, because
$f(\bx_{opt})$ is outside the TTR. Adding samples in the neighborhood of such an
``optimum'' point $\bx_{opt}$ causes the predictive distribution to
progressively narrow, shrinking the trace term in Equation (\ref{eq:ESL2D})
while the predictive mean $p^{(f|1)}$ converges to $f(\bx_{opt})$.  The function
$\mathcal{E}(\bx)$ thus eventually stabilizes at the optimal value $\left\Vert
f(\bx_{opt})-f_T\right\Vert_2^2$, and the optimization stalls, being unable to
fit the uncertainty box at $\bx_{opt}$ into the TTR.  Had we adopted an EI-based
approach instead of the ESL2D approach chosen here, this behavior near
non-solution minima of $\delta(\bx)$ would manifest itself as flat regions where
the sampling is unable to increase the value of the $\delta(\bx)$ from previous
iterations, and the optimization would again stall. In order for such an
algorithm to produce confidence that a convergence failure likely corresponds to
the nonexistence of a solution, it would be necessary to restart it multiple
times in different parts of the domain, or to begin with a dense design of
initial sample points.

By contrast, the TAD acquisition function naturally leads to rejection of an
initially-promising non-solution point, because as the predictive distribution
narrows, the log-probability density of $f_T$, initially high, begins to shrink
rapidly. TAD thus avoids stalling at such points, once their inadequacy is
revealed, and is therefore capable of exploring different regions of the input space starting from a small initial design  $(\bx_1, \bg_1)$, and without requiring an extensive campaign of restarts across the domain.

It thus appears that the incorporation of changing relative
uncertainty in the surrogate function along different search directions does
indeed impact the efficiency of the search. 

\section{Discussion}

We have given some theory and some empirical experience to justify the design of TAD, a new algorithm for data-driven sequential discovery of optimal control settings for advanced manufacturing equipment. The algorithm is designed to track uncertainty in multiple design feature dimensions, so as to be capable of satisfying simultaneous specified tolerance criteria within specified uncertainty. 

Some noteworthy features of TAD  bear comment. TAD works unproblematically with gradient-based optimization methods, unlike methods that rely on expected improvement-type criteria. EI is known not to be well suited to gradient optimization, because it tends to be multimodal with large ``flat'' regions of minimal change and because it is not continuously differentiable \cite{jones1998efficient}. In addition, the TAD acquisition function is inexpensive to compute, unlike, for example, the knowledge gradient \cite{NIPS2016_18d10dc6}.   Such methods are thereby limited to relatively low-dimensional problems. 
The advanced manufacturing problem space, on the other hand, can easily lead to relatively large control and design spaces. Having a gradient method available provides some confidence that the algorithm can be scaled to larger problems.

A second reason to feel confidence in success at larger scales is the recent breakthroughs that have occurred in GP modeling with vastly more points than was possible in the past, using fast kernel approximations and GPU acceleration \cite{wilson2015kernel,gardner2018gpytorch}. Given that high-dimensional control and design spaces might require many iterations and numerous samples to reach convergence and that the associated GP and TAD optimization problems can therefore technically require many potentially expensive linear solves, the availability of fast  methods seems likely to be  important.

The importance of validation, principally by means of LOO validation, has been emphasized by, for example,  \cite{jones1998efficient,bingham2014design}.  It is worth driving home again. A Gaussian process surrogate whose validity is not assiduously monitored has every chance of making poor predictions, and no algorithm founded on such a model is secure no matter how sound its mathematical foundation.  Oddly we have not found any references in the literature to the seemingly obvious idea of using the $\chi^2$ distribution with the predictive distribution of a batch of samples, in the manner that we described in \cref{subsec:validation}. This is a straightforward validation test that offers sensitivity gains over straight LOO and that comes practically  for free since the linear solve of a GP prediction is an inner product away from the quadratic form that yields $\chi^2$.

In real-world manufacturing applications the highly simplified model represented by the mapping $X\rightarrow Y$ from control space to design space can be clouded. For example, a product's desired design properties might require so-called ex situ measurements at facilities external to the manufacturing facilities \cite{wolfman2022importance}, and this requirement can introduce long time latencies in the measurement of the experimental response function that vitiate the practical application of a direct TAD-based methodology. In such cases one can still envision a place for the TAD algorithm in the design workflow. For example, in additive manufacturing, it is pointed  out in \cite{10.1115/DETC2019-98415}  that  complex but discernible relationships  can exist between process (i.e., promptly measured) manufacture parameters and final structure, property, and performance parameters of a product.  Even when those parameters must be validated externally, a TAD workflow targeted at a known-good region of a process-parameter space from a high-dimensional control space could cut down considerably on the number of test pieces that must be sent to the external test facility for validation. 

\section*{Acknowledgments}
This material was based upon work supported by the U.S. Department of Energy, Office of Science, Office of Advanced Scientific Computing Research, Scientific Discovery through Advanced Computing (SciDAC) program through the FASTMath Institute under Contract DE-AC02-06CH11357 at Argonne National Laboratory.

\appendix

\section{Gaussian Prediction Update Formula}\label{sec:appendix_gpuf}

Suppose a vector-valued random variable of dimension $1+N_1+N_2$, written in partitioned form $[f,\bg_1,\bg_2]^T$, has a multivariate normal distribution, 
\begin{equation}
\begin{bmatrix}
f\\
\bg_1\\
\bg_2
\end{bmatrix}
\sim\mathcal{N}\left\{ \begin{bmatrix}
m\\
\bmm_{1}\\
\bmm_{2}
\end{bmatrix},
\begin{bmatrix}
C_{ff} & \bC_{f1} & \bC_{f2}\\
\bC_{1f} & \bC_{11} & \bC_{12}\\
\bC_{2f} & \bC_{21} & \bC_{22}
\end{bmatrix}\right\},\label{eq:jnorm_1}
\end{equation}
where $f$ and $m$ are 1-dimensional, $\bg_1$ and $\bmm_1$ are $N_1$-dimensional, $\bg_2$ and $\bmm_2$ are $N_2$-dimensional, and the various covariance matrix partitions have the obvious necessary dimensions.

The usual normal theory conditioning formulae allow us to write the distributions for $\bg_2|\bg_1$, $f|\bg_1$ and $f|(\bg_1,\bg_2)$:
\begin{align}
\bg_{2}|\bg_{1} &\sim \mathcal{N}\left(\bp^{(2|1)},\bQ^{(2|1)}\right)\label{eq:datapred}\\
\bp^{(2|1)} &= \bmm_{2}+\bC_{21}\bC_{11}^{-1}(\bg_{1}-\bmm_{1})\label{eq:datapred_mean}\\
\bQ^{(2|1)} &= \bC_{22}-\bC_{21}\bC_{11}^{-1}\bC_{12}.\label{eq:datapred_cov}
\end{align}
\begin{align}
f|\bg_{1} &\sim \mathcal{N}(p^{(f|1)},Q^{(f|1)})\label{eq:pred_1}\\
p^{(f|1)} &= m+\bC_{f1}\bC_{11}^{-1}(\bg_{1}-\bmm_1)\label{eq:pred_1_mean}\\
Q^{(f|1)} &= C_{ff}-\bC_{f1}\bC_{11}^{-1}\bC_{1f}.,\label{eq:pred_1_cov}
\end{align}
and
\begin{align}
f|\bg_{1},\bg_{2} &\sim \mathcal{N}\left\{ p^{(f|1+2)},Q^{(f|1+2)}\right\} \label{eq:pred_1+2}\\
p^{(f|1+2)} &= m+\begin{bmatrix}
\bC_{f1} & \bC_{f2}\end{bmatrix}
\begin{bmatrix}
\bC_{11} & \bC_{12}\\
\bC_{21} & \bC_{22}
\end{bmatrix}^{-1}
\begin{bmatrix}
\bg_{1}-\bmm_{1}\\
\bg_{2}-\bmm_{2}
\end{bmatrix}\label{eq:pred_1+2_mean}\\
Q^{(f|1+2)} &= C_{ff}-\begin{bmatrix}
\bC_{f1} & \bC_{f2}\end{bmatrix}
\begin{bmatrix}
\bC_{11} & \bC_{12}\\
\bC_{21} & \bC_{22}
\end{bmatrix}^{-1}
\begin{bmatrix}
\bC_{1f}\\
\bC_{2f}
\end{bmatrix}.\label{eq:pred_1+2_cov}
\end{align}

In the main text, the interpretation of $\bg_1$ is ``data in the can,'' while that of $\bg_2$ is ``data that could be acquired at the next iteration,'' and $f$ is a latent value to be predicted. It is therefore useful to re-express the predictive distribution $f|(\bg_1,\bg_2)$ represented by Equations (\ref{eq:pred_1+2}--\ref{eq:pred_1+2_cov}) as an update to the predictive distribution $f|\bg_1$ represented by Equations (\ref{eq:pred_1}--\ref{eq:pred_1_cov}) consequent on the updating of the data $\{\bg_1\}\rightarrow\{\bg_1,\bg_2\}$. 

\begin{lemma}[Gaussian Prediction Update]\label{thm:gpred_update} Update formulae from $f|\bg_1$ to $f|(\bg_1,\bg_2)$ may be written as follows:

\begin{enumerate}
\item[(i)] The normal means $p^{(f|1+2)}$ and $p^{(f|1)}$ are related by
\begin{equation}
p^{(f|1+2)} = p^{(f|1)} + \left(\bC_{f2}-\bC_{f1}\bC_{11}^{-1}\bC_{12}\right)\left(\bQ^{(2|1)}\right)^{-1}\left(\bg_{2}-\bp^{(2|1)}\right).\label{eq:mean_update}
\end{equation}

\item[(ii)] The normal covariances $Q^{(f|1+2)}$ and $Q^{(f|1)}$ are related by
\begin{equation}
Q^{(f|1+2)} = Q^{(f|1)} -\left(\bC_{f2}-\bC_{f1}\bC_{11}^{-1}\bC_{12}\right)\left(\bQ^{(2|1)}\right)^{-1}\left(\bC_{2f}-\bC_{21}\bC_{11}^{-1}\bC_{1f}\right).\label{eq:Covariance_update}
\end{equation}
\end{enumerate}
\end{lemma}
\begin{proof}

We begin by expanding the inversion of the partitioned matrix in Equations (\ref{eq:pred_1+2_mean}) and (\ref{eq:pred_1+2_cov}) using the matrix inversion lemma (see \S 2.7 of \cite{press2007numerical}):

\begin{equation}
\begin{bmatrix}
C_{11} & C_{12}\\
C_{21} & C_{22}
\end{bmatrix}^{-1}=
\begin{bmatrix}
I_{11} & I_{12}\\
I_{21} & I_{22}
\end{bmatrix},\label{eq:MIL_1}
\end{equation}
with
\begin{align}
\bI_{22} &\equiv \left(\bC_{22}-\bC_{21}\bC_{11}^{-1}\bC_{12}\right)^{-1}\label{eq:MIL_2}\\
 &= \left(Q^{(2|1)}\right)^{-1}\label{eq:MIL_2p}\\
\bI_{12} &\equiv -\bC_{11}^{-1}\bC_{12}\bI_{22}\label{eq:MIL_3}\\
\bI_{21} &\equiv -\bI_{22}\bC_{21}\bC_{11}^{-1}\label{eq:MIL_4}\\
\bI_{11} &\equiv \bC_{11}^{-1}+\bC_{11}^{-1}\bC_{12}\bI_{22}\bC_{21}\bC_{11}^{-1}.\label{eq:MIL_5}
\end{align}
The second line notes the equality with the inverse of Equation (\ref{eq:datapred_cov}).

To establish part (i), we combine Equations (\ref{eq:pred_1+2_mean}) and
(\ref{eq:MIL_1}-\ref{eq:MIL_5}), obtaining an expression for the mean update:

\begin{align}
p^{(f|1+2)} &= m+\begin{pmatrix}
\bC_{f1}I_{11}+\bC_{f2}I_{21} & \bC_{f1}I_{12}+\bC_{f2}I_{22}\end{pmatrix}
\begin{pmatrix}
\bg_{1}-\bmm_{1}\\
\bg_{2}-\bmm_{2}
 \end{pmatrix}\nonumber \\
 &= m+\bC_{f1}\left(\bC_{11}^{-1}+\bC_{11}^{-1}\bC_{12}I_{22}\bC_{21}\bC_{11}^{-1}\right)\left(\bg_{1}-\bmm_{1}\right)\nonumber \\
 & +\bC_{f2}\left(-I_{22}\bC_{21}\bC_{11}^{-1}\right)\left(\bg_{1}-\bmm_{1}\right)\nonumber \\
 & +\left[\bC_{f1}\left(-\bC_{11}^{-1}\bC_{12}I_{22}\right)+\bC_{f2}I_{22}\right]\left(\bg_{2}-\bmm_{2}\right)\nonumber \\
 &= p^{(f|1)}+\left(\bC_{f1}\bC_{11}^{-1}\bC_{12}-\bC_{f2}\right)I_{22}\bC_{21}\bC_{11}^{-1}\left(\bg_{1}-\bmm_{1}\right)\nonumber \\
 &  -\left(\bC_{f1}\bC_{11}^{-1}\bC_{12}-\bC_{f2}\right)I_{22}\left(\bg_{2}-\bmm_{2}\right)\nonumber \\
 & = p^{(f|1)}+\left(\bC_{f2}-\bC_{f1}\bC_{11}^{-1}\bC_{12}\right)I_{22}\left[\bg_{2}-\bmm_{2}-\bC_{21}\bC_{11}^{-1}\left(\bg_{1}-\bmm_{1}\right)\right]\nonumber \\
 & = 
p^{(f|1)}+\left(\bC_{f2}-\bC_{f1}\bC_{11}^{-1}\bC_{12}\right)
\left(\bQ^{(2|1)}\right)^{-1}
\left(\bg_{2}-p^{(2|1)}\right),\nonumber
\label{eq:mean_update_2}
\end{align}
where in the last line we have used Equations (\ref{eq:datapred_mean})
and (\ref{eq:MIL_2p}).

To establish part (ii), we combine Equations (\ref{eq:pred_1+2_cov}) and (\ref{eq:MIL_1}-\ref{eq:MIL_5}) to obtain an expression for the covariance update:
\begin{align}
Q^{(f|1+2)} &= \bC_{ff}-\bC_{f1}\bI_{11}\bC_{1f}-\bC_{f2}\bI_{21}\bC_{1f}-\bC_{f1}\bI_{12}\bC_{2f}-\bC_{f2}\bI_{22}\bC_{2f}\nonumber \\
 &= \bC_{ff}-\bC_{f1}\bC_{11}^{-1}\bC_{1f}-\bC_{f1}\bC_{11}^{-1}\bC_{12}\bI_{22}\bC_{21}\bC_{11}^{-1}\bC_{1f}\nonumber \\
 & -\bC_{f2}\left(-\bI_{22}\bC_{21}\bC_{11}^{-1}\right)\bC_{1f}-\bC_{f1}\left(-\bC_{11}^{-1}\bC_{12}\bI_{22}\right)\bC_{2f}-\bC_{f2}\bI_{22}\bC_{2f}\nonumber \\
 &= Q^{(f|1)}-\left(\bC_{f1}\bC_{11}^{-1}\bC_{12}-\bC_{f2}\right)\bI_{22}\bC_{21}\bC_{11}^{-1}\bC_{1f}-\left(\bC_{f2}-\bC_{f1}\bC_{11}^{-1}\bC_{12}\right)\bI_{22}\bC_{2f}\nonumber \\
 &= Q^{(f|1)}-\left(\bC_{f2}-\bC_{f1}\bC_{11}^{-1}\bC_{12}\right)\bI_{22}\left(\bC_{2f}-\bC_{21}\bC_{11}^{-1}\bC_{1f}\right)\nonumber \\
 &= Q^{(f|1)}-\left(\bC_{f2}-\bC_{f1}\bC_{11}^{-1}\bC_{12}\right)\left(\bQ^{(2|1)}\right)^{-1}\left(\bC_{2f}-\bC_{21}\bC_{11}^{-1}\bC_{1f}\right).\nonumber
\label{eq:Covariance_update_2}
\end{align}
\end{proof}

\begin{remk}
We note---although  hardly unexpected---that $p^{(f|1+2)}$ is linearly dependent on $\bg_2$, while $Q^{(f|1+2)}$ is independent of $\bg_2$.
\end{remk}

\section{Redundant Information in The Zero-Noise Limit}

The behavior of the TAD algorithm is elucidated in the text by considering the case of redundant points in the limit of zero measurement noise. In this appendix we supply the technical results to support that discussion.

Suppose that the points $\bx_1$ and $\bx_2$ have a subset of points $\br$ in common, so that $\bx_1=\by_1\cup \br$, $\bx_2=\by_2\cup \br$,
$\by_1\cap\by_2=\emptyset$. We partition the submatrices of the covariance in Equation \cref{eq:joint1_cov} accordingly:
\begin{equation}
\Gamma\equiv
\begin{bmatrix}
K_{xx} & \bK_{x1} & \bK_{x2}\\
\bK_{1x} & \bK_{11} + \bSigma_1  & \bK_{12}\\
\bK_{2x} & \bK_{21} & \bK_{22} + \bSigma_2
\end{bmatrix},
\label{eq:joint1_cov2}
\end{equation}
where
\begin{equation}
\bK_{11}=
\begin{bmatrix}
\bJ_{11}&\bJ_{1r}\\
\bJ_{r1}&\bJ_{rr}
\end{bmatrix}
\quad;\quad
\bK_{22}=
\begin{bmatrix}
\bJ_{22}&\bJ_{2r}\\
\bJ_{r2}&\bJ_{rr}
\end{bmatrix}
\quad;\quad
\bK_{12}=
\begin{bmatrix}
\bJ_{12}&\bJ_{1r}\\
\bJ_{r2}&\bJ_{rr}
\end{bmatrix},
\label{eq:partitions_1}
\end{equation}
\begin{equation}
\bK_{x1}=
\begin{bmatrix}
\bJ_{x1}&\bJ_{xr}
\end{bmatrix}
\quad;\quad
\bK_{x2}=
\begin{bmatrix}
\bJ_{x2}&\bJ_{xr}
\end{bmatrix},
\label{eq:partitions_2}
\end{equation}
with $\bJ_{11}\equiv C(\by_1,\by_1)$, $\bJ_{1r}=\bJ_{r1}^T\equiv C(\by_1,\br)$, $\bJ_{rr}\equiv C(\br,\br)$, $\bJ_{22}\equiv C(\by_2,\by_2)$, $\bJ_{2r}=\bJ_{r2}^T\equiv C(\by_r,\br)$, $\bJ_{12}\equiv C(\by_1,\by_2)$, $\bJ_{x1}\equiv C(x,\by_1)$, $\bJ_{x2}\equiv C(x,\by_2)$, and $\bJ_{xr}\equiv C(x,\br)$.

If we omit the redundant $\br$ points from $\bx_2$, we obtain a nonredundant set $\bx_2^\prime=\by_2$, in terms of which the TAD problem gives rise to the covariance $\Gamma^\prime$ given by
\begin{equation}
\Gamma^\prime\equiv
\begin{bmatrix}
K_{xx} & \bK_{x1} & \bJ_{x2}\\
\bK_{1x} & \bK_{11} + \bSigma_1  & \bL_{12}\\
\bJ_{2x} & \bL_{21} & \bJ_{22} + \bsigma_2
\end{bmatrix},
\label{eq:joint1_covprime}
\end{equation}
where we have defined the block matrices $\bL_{12}\equiv\bL_{21}^T$, with
\begin{equation}
\bL_{12}\equiv
\begin{bmatrix}
\bJ_{12}\\
\bJ_{r2}
\end{bmatrix}
\label{eq:Lmat_def}
\end{equation}
and where $\bsigma_2$ is the restriction of the noise $\bSigma_2$ to the nonredundant space of $\by_2$. We will refer to the dimensionality of the various partitioned spaces as $D_1^\prime$ for the nonredundant ``$1^\prime$''  space, $D_2^\prime$ for the nonredundant ``$2^\prime$'' space, $D_r$ for the redundant space, and $D_x$ for the prediction space. So, for example, the matrix $\bL_{12}$ is $(D_1^\prime+D_r)\times D_2^\prime$-dimensional, the matrix $\bJ_{2r}$ is $D_2^\prime\times D_r$-dimensional, $\bJ_{x1}$ is $D_x\times D_1^\prime$-dimensional, and so on.

The predictive covariance that follows from the covariance structure in Equation \cref{eq:joint1_cov2} is
\begin{equation}
Q^{(f(x)|1+2)}=K_{xx}-\begin{bmatrix}
\bK_{x1} & \bK_{x2}
\end{bmatrix}
\begin{bmatrix}
\bK_{11}+\bSigma_1 & \bK_{12}\\
\bK_{21} & \bK_{22}+\bSigma_2
\end{bmatrix}^{-1}
\begin{bmatrix}
\bK_{1x}\\
\bK_{2x}
\end{bmatrix},\label{eq:pred_1+2_cov2}
\end{equation}
while the predictive covariance that follows from the nonredundant data and the covariance in Equation \cref{eq:joint1_covprime} is
\begin{equation}
Q^{(f(x)|1+2^\prime)}=K_{xx}-\begin{bmatrix}
\bK_{x1} & \bJ_{x2}
\end{bmatrix}
\begin{bmatrix}
\bK_{11}+\bSigma_1 & \bJ_{12}\\
\bJ_{21} & \bJ_{22}+\bsigma_2
\end{bmatrix}^{-1}
\begin{bmatrix}
\bK_{1x}\\
\bJ_{2x}
\end{bmatrix}.\label{eq:pred_1+2_covprime}
\end{equation}

Suppose  that we set $\bSigma_1=0$, and we introduce  a multiplicative parameter $\epsilon$ regulating the size of $\bSigma_2$, that  is,  $\bSigma_2=\epsilon\overline{\bSigma}_2$ with $\overline{\bSigma}_2$ a finite noise covariance matrix, so that we can take a noise-free limit by taking the limit $\epsilon\rightarrow 0$.

As a matter of information theory, it is intuitively clear that redundant information affects predictions only up to noise and that in the noise-free limit the predictive distributions characterized by the covariances in Equations \cref{eq:pred_1+2_cov2,eq:pred_1+2_covprime} should coincide. On the other hand, as an algebraic matter it is not at all clear that the covariance in Equation \cref{eq:pred_1+2_cov2} is even well defined, since the matrix inverse on the right-hand side of Equation  \cref{eq:pred_1+2_cov2} certainly does not exist in the limit. This is troubling because we would like to be sure of the behavior of the TAD acquisition function, Equation \cref{eq:TAD_of}, when ``2'' points approach ``1'' points. Numerically, we are likely to encounter instabilities when this occurs.  It is helpful to at least know what one ought to expect in this case if one had infinite numerical precision at one's disposal.

The following theorem is therefore helpful.

\begin{theorem}[Noise-Free Limit]\label{thm:noise-free} The covariance 
$Q^{(f(x)|1+2)}$ has a finite limit as $\epsilon\rightarrow 0$ and, in fact,
$\lim_{\epsilon\rightarrow 0}Q^{(f(x)|1+2)}= Q^{(f(x)|1+2^\prime)}$.
\end{theorem}

\begin{proof}
Consider from the update formula, Equation \cref{eq:Qthm},
\begin{equation}
Q^{(f(x)|1+2)}=Q^{(f(x)|1)}-T,
\label{eq:Q_pred_formula}
\end{equation}
where  $T$ is the correction term for for the ``2'' information, given by
\begin{equation}
T=\left[
\bK_{x2}-\bK_{x1}\bK_{11}^{-1}\bK_{12}
\right]
\left(\bQ^{(2|1)}\right)^{-1}
\left[
\bK_{2x}-\bK_{21}\bK_{11}^{-1}\bK_{1x}
\right]
\label{eq:correction_term}
\end{equation}
and where $\bQ^{(2|1)}$ is given by Equation \cref{eq:datapred_cov_t},
\begin{equation}
\bQ^{(2|1)} = 
\bK_{22}+\bSigma_2
- \bK_{21}\bK_{11}^{-1}\bK_{12}.
\label{eq:Q21_2}
\end{equation}

We will have a proof of the theorem if we can show that $\lim_{\epsilon\rightarrow 0}T=T^\prime$, where $T^\prime$ is the correction term for the nonredundant ``$2^\prime$'' information. $T^\prime$ can be written down by adapting the expression for $T$ in Equation \cref{eq:correction_term} in light of a simple comparison of Equations \cref{eq:joint1_cov2,eq:joint1_covprime}:
\begin{equation}
T^\prime=
\left[\bJ_{x2}-\bK_{x1}\bK_{11}^{-1}\bL_{12}\right]
\left(\bQ^{(2^\prime|1)}\right)^{-1}
\left[\bJ_{2x}-\bL_{21}\bK_{11}^{-1}\bK_{1x}\right],
\label{eq:Tprime}
\end{equation}
where
\begin{equation}
\bQ^{(2^\prime|1)}=\bJ_{22}+\bsigma_2-\bL_{21}\bK_{11}^{-1}\bL_{12}.
\label{eq:Qprime}
\end{equation}

The first task is to provide an asymptotic form for the term $\left(\bQ^{(2|1)}\right)^{-1}$ in the limit $\epsilon\rightarrow 0$.  The term in Equation \cref{eq:Q21_2} proportional to $\epsilon^0$ is
$\bK_{22}-\bK_{21}\bK_{11}^{-1}\bK_{12}$. Note that according to Equations
\cref{eq:partitions_1}, the matrices $\bK_{11}$ and $\bK_{12}$ share their
second block column. We may exploit this fact, together with the definition of
the inverse matrix and the definitions in Equations \cref{eq:partitions_1}, to
evaluate $\bK_{21}\bK_{11}^{-1}\bK_{12}$ and thereby show that
\begin{equation}
\bQ^{(2|1)}=
\begin{bmatrix}
\bz&0\\
0&0
\end{bmatrix}
+\bSigma_2,
\label{eq:Q21_3}
\end{equation}
where
\begin{equation}
\bz\equiv
\bJ_{22}-
\bL_{21}\bK_{11}^{-1}\bL_{12}.
\end{equation}

We now simplify the notation, defining  the $D_2^\prime\times(D_2^\prime+D_r)$-dimensional projection operator $\bP_{2^\prime}$ from the redundant 2-space to the nonredundant $2^\prime$-space
\begin{equation}
\bP_{2^\prime}\equiv
\begin{bmatrix}
\bm{I}_{D_2^\prime}&\bm{0}
\end{bmatrix}.
\label{eq:Projection}
\end{equation}
We may write
\begin{equation}
\bQ^{(2|1)}=\bSigma_2 + \bP_{2^\prime}^T\bz\bP_{2^\prime}.
\label{eq:Q21_4}
\end{equation}
Equation \cref{eq:Q21_4} is a low-rank update to the first term on its right-hand side, so we may use the Sherman--Morrison--Woodbury lemma (see \S 2.7 of \cite{press2007numerical}) to obtain an expression for the inverse of $\bQ^{(2|1)}$ that we require in Equation \cref{eq:correction_term}:
\begin{align}
\left(\bQ^{(2|1)}\right)^{-1}&={\bSigma_2}^{-1}-{\bSigma_2}^{-1}\bP_{2^\prime}^T
\left(\bz^{-1}+\bP_{2^\prime}{\bSigma_2}^{-1}\bP_{2^\prime}^T\right)^{-1}\bP_{2^\prime}{\bSigma_2}^{-1}\nonumber\\
&={\bSigma_2}^{-1}-{\bSigma_2}^{-1}\bP_{2^\prime}^T\bigg[
\left(\bP_{2^\prime}{\bSigma_2}^{-1}\bP_{2^\prime}^T\right)^{-1}\nonumber\\
&\hspace{2cm}
-\left(\bP_{2^\prime}{\bSigma_2}^{-1}\bP_{2^\prime}^T\right)^{-1}
\bz^{-1}
\left(\bP_{2^\prime}{\bSigma_2}^{-1}\bP_{2^\prime}^T\right)^{-1}+\mathcal{O}(\epsilon^3)
\bigg]\bP_{2^\prime}\bSigma_2^{-1},
\label{eq:smw}
\end{align}
where the asymptotic notation $\mathcal{O}[\epsilon^n]$ has the conventional meaning $x=z+\mathcal{O}[\epsilon^n]\Rightarrow\left|\lim_{\epsilon\rightarrow 0}(x-z)/\epsilon^n\right|<\infty$ and is interpreted in matrix equations as applying to all the matrix elements in the equation.

We may also compute the term $\bK_{x2}-\bK_{x1}\bK_{11}^{-1}\bK_{12}$ in Equation \cref{eq:correction_term}, again appealing to the common block column structure of $\bK_{11}$ and $\bK_{12}$ to evaluate $\bK_{11}^{-1}\bK_{12}$:
\begin{align}
\bK_{x2}-\bK_{x1}\bK_{11}^{-1}\bK_{12}&=
\begin{bmatrix}
\bJ_{x2}&\bJ_{xr}
\end{bmatrix}
-
\begin{bmatrix}
\bJ_{x1}&\bJ_{xr}
\end{bmatrix}
\bK_{11}^{-1}\bK_{12}\nonumber\\
&=
\begin{bmatrix}
\bq&\bm{0}
\end{bmatrix}\nonumber\\
&=\bq\bP_{2^\prime},
\label{eq:otherterm}
\end{align}
where we have defined
\begin{equation}
\bq\equiv \bJ_{x2}-
\bK_{x1}\bK_{11}^{-1}\bL_{12}.
\label{eq:qdef}
\end{equation}

Inserting Equations \cref{eq:smw,eq:qdef} into Equation \cref{eq:correction_term}, we find

\begin{align}
T&=\bq\left(\bP_{2^\prime}\bSigma_2^{-1}\bP_{2^\prime}^T\right)\bq^T
-\bq\left(\bP_{2^\prime}\bSigma_2^{-1}\bP_{2^\prime}^T\right)
\left(\bP_{2^\prime}\bSigma_2^{-1}\bP_{2^\prime}^T\right)^{-1}
\left(\bP_{2^\prime}\bSigma_2^{-1}\bP_{2^\prime}^T\right)\bq^T\nonumber\\
&+\bq\left(\bP_{2^\prime}\bSigma_2^{-1}\bP_{2^\prime}^T\right)
\left(\bP_{2^\prime}\bSigma_2^{-1}\bP_{2^\prime}^T\right)^{-1}
\bz^{-1}
\left(\bP_{2^\prime}\bSigma_2^{-1}\bP_{2^\prime}^T\right)^{-1}
\left(\bP_{2^\prime}\bSigma_2^{-1}\bP_{2^\prime}^T\right)
\bq^T\nonumber\\
&+\mathcal{O}(\epsilon)\nonumber\\
&=\bq\bz^{-1}\bq^T+\mathcal{O}(\epsilon)\nonumber\\
&=T^\prime+\mathcal{O}(\epsilon).
\label{eq:T_and_Tprime}
\end{align}

It follows that $\lim_{\epsilon\rightarrow 0}T = T^\prime$, so the theorem is proven.
\end{proof}

\section*{Government License}
The submitted manuscript has been created by UChicago Argonne, LLC,
Operator of Argonne National Laboratory (``Argonne''). Argonne, a U.S.
Department of Energy Office of Science laboratory, is operated under
Contract No. DE-AC02-06CH11357. The U.S. Government retains for itself,
and others acting on its behalf, a paid-up nonexclusive, irrevocable
worldwide license in said article to reproduce, prepare derivative
works, distribute copies to the public, and perform publicly and display
publicly, by or on behalf of the Government. The Department of Energy
will provide public access to these results of federally sponsored
research in accordance with the DOE Public Access Plan \url{https://www.energy.gov/downloads/doe-public-access-plan}.

\bibliographystyle{siamplain}

\begin{thebibliography}{10}

\bibitem{OCW}
{\sc A.~Alfaris}, {\em Multiobjective optimization}.
\newblock
  \url{https://ocw.mit.edu/courses/ids-338j-multidisciplinary-system-design-optimization-spring-2010/f091204b62ab5bb46825829fdedc66cd_MITESD_77S10_lec14.pdf},
  2010.

\bibitem{alvarez2012kernels}
{\sc M.~A. {\'A}lvarez, L.~Rosasco, and N.~D. Lawrence}, {\em Kernels for
  vector-valued functions: A review}, Foundations and Trends{\textregistered}
  in Machine Learning, 4 (2012), pp.~195--266,
  \url{https://arxiv.org/abs/1106.6251}.

\bibitem{bingham2014design}
{\sc D.~Bingham, P.~Ranjan, and W.~J. Welch}, {\em {Design Of Computer
  Experiments for Optimization, Estimation Of Function Contours, and Related
  Objectives}}, Statistics in Action: A Canadian Outlook, 109 (2014),
  \url{https://arxiv.org/abs/1601.05887}.

\bibitem{pymoo}
{\sc J.~{Blank} and K.~{Deb}}, {\em pymoo: Multi-objective optimization in
  python}, IEEE Access, 8 (2020), pp.~89497--89509.

\bibitem{10.2307/2246015}
{\sc K.~Chaloner and I.~Verdinelli}, {\em Bayesian experimental design: A
  review}, Statistical Science, 10 (1995), pp.~273--304,
  \url{https://www.jstor.org/stable/pdf/2246015.pdf}.

\bibitem{cheng2021recent}
{\sc H.~Cheng, J.~G. Shapter, Y.~Li, and G.~Gao}, {\em Recent progress of
  advanced anode materials of lithium-ion batteries}, Journal of Energy
  Chemistry, 57 (2021), pp.~451--468.

\bibitem{contal2013parallel}
{\sc E.~Contal, D.~Buffoni, A.~Robicquet, and N.~Vayatis}, {\em Parallel
  {Gaussian} process optimization with upper confidence bound and pure
  exploration}, in Joint European Conference on Machine Learning and Knowledge
  Discovery in Databases, Springer, 2013, pp.~225--240.

\bibitem{david2004order}
{\sc H.~David and H.~Nagaraja}, {\em Order Statistics}, Wiley Series in
  Probability and Statistics, Wiley, 2004,
  \url{https://books.google.com/books?id=bdhzFXg6xFkC}.

\bibitem{Deb2005}
{\sc K.~Deb, L.~Thiele, M.~Laumanns, and E.~Zitzler}, {\em Scalable Test
  Problems for Evolutionary Multiobjective Optimization}, Springer London,
  London, 2005, pp.~105--145, \url{https://doi.org/10.1007/1-84628-137-7_6},
  \url{https://doi.org/10.1007/1-84628-137-7_6}.

\bibitem{DelshadKhatibi2017}
{\sc P.~Delshad~Khatibi, H.~Henein, and U.~Fritsching}, {\em In-Situ, Real Time
  Diagnostics in the Spray Forming Process}, Springer International Publishing,
  Cham, 2017, pp.~221--263, \url{https://doi.org/10.1007/978-3-319-52689-8_6},
  \url{https://doi.org/10.1007/978-3-319-52689-8_6}.

\bibitem{Eadie:100342}
{\sc W.~T. Eadie, D.~Drijard, F.~E. James, M.~Roos, and B.~Sadoulet}, {\em
  {Statistical methods in experimental physics; 2nd ed.}}, World Scientific,
  Singapore, 2006.

\bibitem{gardner2018gpytorch}
{\sc J.~Gardner, G.~Pleiss, K.~Q. Weinberger, D.~Bindel, and A.~G. Wilson},
  {\em {GPyTorch}: Blackbox matrix-matrix {Gaussian} process inference with
  {GPU} acceleration}, in Advances in Neural Information Processing Systems,
  S.~Bengio, H.~Wallach, H.~Larochelle, K.~Grauman, N.~Cesa-Bianchi, and
  R.~Garnett, eds., vol.~31, Curran Associates, Inc., 2018,
  \url{https://proceedings.neurips.cc/paper/2018/file/27e8e17134dd7083b050476733207ea1-Paper.pdf}.

\bibitem{gonzalez2016batch}
{\sc J.~Gonz{\'a}lez, Z.~Dai, P.~Hennig, and N.~Lawrence}, {\em Batch
  {Bayesian} optimization via local penalization}, in Artificial intelligence
  and statistics, PMLR, 2016, pp.~648--657.

\bibitem{higdon2008computer}
{\sc D.~Higdon, J.~Gattiker, B.~Williams, and M.~Rightley}, {\em Computer model
  calibration using high-dimensional output}, Journal of the American
  Statistical Association, 103 (2008), pp.~570--583.

\bibitem{huang2006sequential}
{\sc D.~Huang, T.~T. Allen, W.~I. Notz, and R.~A. Miller}, {\em Sequential
  kriging optimization using multiple-fidelity evaluations}, Structural and
  Multidisciplinary Optimization, 32 (2006), pp.~369--382.

\bibitem{jenks2020basic}
{\sc C.~Jenks, H.~N. Lee, J.~Lewis, C.~R. Kagan, P.~Nealey, P.~Braun, J.~E.
  Holladay, Y.~Gao, D.~S. Sholl, B.~A. Helms, et~al.}, {\em Basic research
  needs for transformative manufacturing}, tech. report, USDOE Office of
  Science (SC), 2020, \url{https://doi.org/10.2172/1618267}.

\bibitem{jones1998efficient}
{\sc D.~R. Jones, M.~Schonlau, and W.~J. Welch}, {\em {Efficient Global
  Optimization of Expensive Black-Box Functions}}, Journal of Global
  optimization, 13 (1998), pp.~455--492,
  \url{https://doi.org/10.1023/A:1008306431147}.

\bibitem{kennedy2001bayesian}
{\sc M.~C. Kennedy and A.~O'Hagan}, {\em Bayesian calibration of computer
  models}, Journal of the Royal Statistical Society: Series B (Statistical
  Methodology), 63 (2001), pp.~425--464.

\bibitem{kullback1951information}
{\sc S.~Kullback and R.~A. Leibler}, {\em On information and sufficiency}, The
  Annals of Mathematical Statistics, 22 (1951), pp.~79--86.

\bibitem{lindley1956measure}
{\sc D.~V. Lindley}, {\em {On a Measure of the Information Provided by an
  Experiment}}, The Annals of Mathematical Statistics,  (1956), pp.~986--1005,
  \url{https://www.jstor.org/stable/pdf/2237191.pdf}.

\bibitem{609493}
{\sc K.~Muske}, {\em Steady-state target optimization in linear model
  predictive control}, in Proceedings of the 1997 American Control Conference
  (Cat. No.97CH36041), vol.~6, 1997, pp.~3597--3601 vol.6,
  \url{https://doi.org/10.1109/ACC.1997.609493}.

\bibitem{TAD_github}
{\sc M.~Ngom}, {\em Targeted adaptive design}.
\newblock \url{https://github.com/mngom2/Targeted_Adaptive_Design}, 2021.

\bibitem{opper1998general}
{\sc M.~Opper and F.~Vivarelli}, {\em General bounds on bayes errors for
  regression with gaussian processes}, Advances in Neural Information
  Processing Systems, 11 (1998).

\bibitem{PAULSON2020108972}
{\sc N.~H. Paulson, J.~A. Libera, and M.~Stan}, {\em Flame spray pyrolysis
  optimization via statistics and machine learning}, Materials \& Design, 196
  (2020), p.~108972,
  \url{https://doi.org/https://doi.org/10.1016/j.matdes.2020.108972}.

\bibitem{Peremezhney_et_al_2014}
{\sc N.~Peremezhney, E.~Hines, A.~Lapkin, and C.~Connaughton}, {\em Combining
  gaussian processes, mutual information and a genetic algorithm for
  multi-target optimization of expensive-to-evaluate functions}, Engineering
  Optimization, 46 (2014), pp.~1593--1607,
  \url{https://doi.org/10.1080/0305215X.2014.881997}.

\bibitem{GPyTorch}
{\sc G.~Pleiss, J.~R. Gardner, K.~Q. Weinberger, A.~G. Wilson, and
  M.~Balandat}, {\em Gpytorch}.
\newblock \url{https://github.com/cornellius-gp/gpytorch}, 2022,
  \url{https://gpytorch.ai/}.

\bibitem{press2007numerical}
{\sc W.~Press, S.~Teukolsky, W.~Vetterling, and B.~Flannery}, {\em Numerical
  Recipes 3rd Edition: The Art of Scientific Computing}, Cambridge University
  Press, 2007.

\bibitem{ranjan2016inverse}
{\sc P.~Ranjan, M.~Thomas, H.~Teismann, and S.~Mukhoti}, {\em Inverse problem
  for time-series valued computer model via scalarization}, Open Journal of
  Statistics,  (2016), pp.~528--544,
  \url{https://arxiv.org/abs/arXiv:1605.09503}.

\bibitem{RandW}
{\sc C.~Rasmussen and C.~Williams}, {\em Gaussian Processes for Machine
  Learning}, Adaptive Computation and Machine Learning, MIT Press, Cambridge,
  MA, USA, Jan. 2006.

\bibitem{10.1115/DETC2019-98415}
{\sc S.~S. Razvi, S.~Feng, A.~Narayanan, Y.-T.~T. Lee, and P.~Witherell}, {\em
  {A Review of Machine Learning Applications in Additive Manufacturing}},
  vol.~Volume 1: 39th Computers and Information in Engineering Conference of
  International Design Engineering Technical Conferences and Computers and
  Information in Engineering Conference, 08 2019, p.~V001T02A040,
  \url{https://doi.org/10.1115/DETC2019-98415}.

\bibitem{ryan2003estimating}
{\sc K.~J. Ryan}, {\em Estimating expected information gains for experimental
  designs with application to the random fatigue-limit model}, Journal of
  Computational and Graphical Statistics, 12 (2003), pp.~585--603.

\bibitem{sacks1989design}
{\sc J.~Sacks, W.~J. Welch, T.~J. Mitchell, and H.~P. Wynn}, {\em Design and
  analysis of computer experiments}, Statistical science, 4 (1989),
  pp.~409--423.

\bibitem{sacks1984some}
{\sc J.~Sacks and D.~Ylvisaker}, {\em Some model robust designs in regression},
  The Annals of Statistics,  (1984), pp.~1324--1348.

\bibitem{7352306}
{\sc B.~Shahriari, K.~Swersky, Z.~Wang, R.~P. Adams, and N.~de~Freitas}, {\em
  Taking the human out of the loop: A review of bayesian optimization},
  Proceedings of the IEEE, 104 (2016), pp.~148--175,
  \url{https://doi.org/10.1109/JPROC.2015.2494218}.

\bibitem{sinsbeck2017sequential}
{\sc M.~Sinsbeck and W.~Nowak}, {\em Sequential design of computer experiments
  for the solution of {Bayesian} inverse problems}, SIAM/ASA Journal on
  Uncertainty Quantification, 5 (2017), pp.~640--664.

\bibitem{sollich1998learning}
{\sc P.~Sollich}, {\em Learning curves for gaussian processes}, Advances in
  neural information processing systems, 11 (1998).

\bibitem{Vernon_etal_2010}
{\sc I.~Vernon, M.~Goldstein, and R.~G. Bower}, {\em {Galaxy formation: a
  Bayesian uncertainty analysis}}, Bayesian Analysis, 5 (2010), pp.~619 -- 669,
  \url{https://doi.org/10.1214/10-BA524},
  \url{https://doi.org/10.1214/10-BA524}.

\bibitem{williams2007multi}
{\sc C.~Williams, E.~V. Bonilla, and K.~M. Chai}, {\em {Multi-Task {G}aussian
  Process Prediction}}, {Advances In Neural Information Processing Systems},
  (2007), pp.~153--160,
  \url{http://videolectures.net/site/normal_dl/tag=28445/bark08_williams_mtlwgp_01.pdf}.

\bibitem{williams2000upper}
{\sc C.~K. Williams and F.~Vivarelli}, {\em Upper and lower bounds on the
  learning curve for gaussian processes}, Machine Learning, 40 (2000),
  pp.~77--102.

\bibitem{wilson2015kernel}
{\sc A.~Wilson and H.~Nickisch}, {\em Kernel interpolation for scalable
  structured {Gaussian} processes ({KISS-GP})}, in International Conference on
  Machine Learning, PMLR, 2015, pp.~1775--1784,
  \url{http://proceedings.mlr.press/v37/wilson15.pdf}.

\bibitem{wolfman2022importance}
{\sc M.~Wolfman, X.~Wang, J.~C. Garcia, P.~Barai, J.~E. Stubbs, P.~J. Eng,
  O.~Kahvecioglu, T.~L. Kinnibrugh, K.~E. Madsen, H.~Iddir, et~al.}, {\em The
  importance of surface oxygen for lithiation and morphology evolution during
  calcination of high-nickel {NMC} cathodes}, Advanced Energy Materials,
  (2022), p.~2102951.

\bibitem{NIPS2016_18d10dc6}
{\sc J.~Wu and P.~Frazier}, {\em The parallel knowledge gradient method for
  batch {Bayesian} optimization}, in Advances in Neural Information Processing
  Systems, D.~Lee, M.~Sugiyama, U.~Luxburg, I.~Guyon, and R.~Garnett, eds.,
  vol.~29, Curran Associates, Inc., 2016,
  \url{https://proceedings.neurips.cc/paper/2016/file/18d10dc6e666eab6de9215ae5b3d54df-Paper.pdf}.

\bibitem{zhang2016adaptive}
{\sc J.~Zhang, W.~Li, L.~Zeng, and L.~Wu}, {\em An adaptive {G}aussian
  process-based method for efficient {B}ayesian experimental design in
  groundwater contaminant source identification problems}, Water Resources
  Research, 52 (2016), pp.~5971--5984,
  \url{https://doi.org/10.1002/2016WR018598}.

\bibitem{zhang2019sequential}
{\sc R.~Zhang, C.~D. Lin, and P.~Ranjan}, {\em A sequential design approach for
  calibrating dynamic computer simulators}, SIAM/ASA Journal on Uncertainty
  Quantification, 7 (2019), pp.~1245--1274,
  \url{https://doi.org/10.1137/18M1224544}.

\end{thebibliography}

\end{document}